\documentclass[11pt]{article}
\usepackage{dsfont}
\usepackage{environ}
\usepackage[margin=1.15in]{geometry}
\usepackage[utf8x]{inputenc}
\usepackage[usenames,dvipsnames]{xcolor}
\definecolor{niceRed}{RGB}{153,0,0}
\definecolor{niceRed}{RGB}{190,38,38}
\definecolor{blueGrotto}{HTML}{059DC0}
\definecolor{royalBlue}{HTML}{057DCD}
\definecolor{navyBlueP}{HTML}{0B579C}
\definecolor{Azure}{rgb}{0.0, 0.5, 1.0}
\definecolor{ceruleanblue}{rgb}{0.16, 0.32, 0.75}

\usepackage{cmap}
\usepackage{bm}

\usepackage{amsmath}
\usepackage{amsfonts}
\usepackage{amssymb}
\usepackage{amsbsy}
\usepackage{amsthm}

\numberwithin{equation}{section}

\usepackage{algorithm, caption}
\usepackage[noend]{algpseudocode}
\usepackage{listings}

\usepackage{enumitem}

\usepackage[pagebackref]{hyperref}
\hypersetup{
	colorlinks = true,
	urlcolor = {ForestGreen},
	linkcolor = {magenta},
	citecolor = {ceruleanblue}
}
\usepackage[nameinlink]{cleveref}

\usepackage{multirow}
\usepackage{array}

\usepackage{chngcntr}

\counterwithin*{equation}{section}
\usepackage{chngcntr}
\usepackage{soul}
\usepackage{dsfont}
\usepackage{nicefrac}

\def\compactify{\itemsep=0pt \topsep=0pt \partopsep=0pt \parsep=0pt}
\let\latexusecounter=\usecounter

\definecolor{myC}{rgb}{0, 255, 255}
\definecolor{myY}{rgb}{204, 204, 0}
\definecolor{myM}{rgb}{255, 0, 255}
\definecolor{secinhead}{RGB}{249,196,95}
\definecolor{lgray}{gray}{0.8}

\usepackage{soul}

\usepackage{algorithm}
\usepackage{algpseudocode}

\newtheorem{theorem}{Theorem}  
\newtheorem{proposition}[theorem]{Proposition}
\newtheorem{corollary}[theorem]{Corollary}
\newtheorem{lemma}[theorem]{Lemma}
\newtheorem{definition}[theorem]{Definition}
\newtheorem{inftheorem}{Informal Theorem}
\newtheorem{fact}{Fact}

\newtheorem{claim}{Claim}
\newtheorem{remark}{Remark}
\newtheorem{question}{Question}

\newcommand{\reals}{\mathbb{R}}

\newcommand{\D}{\mathcal{D}}
\newcommand{\Q}{\mathcal{Q}}

\newcommand{\supp}{\mathrm{supp}}
\newcommand{\poly}{\mathrm{poly}}

\def\l{\ell}
\def\<{\langle}
\def\>{\rangle}

\def\wt{\widetilde}
\def\poly{\mathrm{poly}}

\def\vec{\bm}

\makeatletter
\renewenvironment{abstract}{%
	\if@twocolumn
	\section*{\abstractname}%
	\else 
	\begin{center}%
		{\bfseries \large\abstractname\vspace{\z@}}
	\end{center}%
	\quotation
	\fi}
{\if@twocolumn\else\endquotation\fi}
\makeatother

\usepackage{pgfplots}
\usepgfplotslibrary{fillbetween}

\usepackage{filecontents}
\usepackage{caption}
\usepackage{subcaption}
\usepackage{graphicx}

\usepackage{xinttools} 
\usepackage{tikz}
\usetikzlibrary{calc}

\def\biglen{20cm} 
\tikzset{
  half plane/.style={ to path={
       ($(\tikztostart)!.5!(\tikztotarget)!#1!(\tikztotarget)!\biglen!90:(\tikztotarget)$)
    -- ($(\tikztostart)!.5!(\tikztotarget)!#1!(\tikztotarget)!\biglen!-90:(\tikztotarget)$)
    -- ([turn]0,2*\biglen) -- ([turn]0,2*\biglen) -- cycle}},
  half plane/.default={1pt}
}

\def\maxxy{4} 

\def\reals{\mathbb{R}}
\def\S{\mathcal{S}}
\def\l{\ell}

\let\Pr\relax
\DeclareMathOperator*{\Pr}{\mathbf{Pr}}
\DeclareMathOperator*{\E}{\mathbf{E}}
\DeclareMathOperator*{\Var}{\mathbf{V}}

\def\D{\mathcal{D}}
\def\L{\mathcal{L}}

\def\Q{\mathcal{Q}}
\def\N{\mathcal{N}}

\def\dim{d}

\def\maxcut{\textsc{Max-Cut }}
\def\P{\textsc{P}}

\def\NP{\textsc{NP}}
\def\opt{\mathrm{opt}}
\def\vmu{\vec \mu}
\def\vSigma{\vec \Sigma} 
\def\<{\langle}
\def\>{\rangle}
\def\b{\musFlat}

\def\eps{\epsilon}
\def\wt{\widetilde}
\def\wh{\widehat}
\def\poly{\mathrm{poly}}

\def\dtv{\mathrm{TV}}

\def\Cov{\mathop{\mathbf{Cov}}}

\begin{document}

	\title{Efficient Algorithms for Learning from Coarse Labels}	\author{
    \begin{tabular}{cc}	
        \begin{tabular}{c}
        \textbf{Dimitris Fotakis} \\
		\small National Technical University of Athens \\
		\url{fotakis@cs.ntua.gr}
        \end{tabular}
        & 	
		\begin{tabular}{c}
		\textbf{Alkis Kalavasis} \\
		\small National Technical University of Athens \\
		\url{kalavasisalkis@mail.ntua.gr}
		\end{tabular}
		\\
		\\
		\begin{tabular}{c}
		\textbf{Vasilis Kontonis} \\
		\small University of Wisconsin-Madison \\
		\url{kontonis@wisc.edu}
		\end{tabular}
		 & 
		 \begin{tabular}{c}
		 \textbf{Christos Tzamos} \\
		\small University of Wisconsin-Madison \\
		\url{tzamos@wisc.edu}\\
		\end{tabular}
		\end{tabular}
	}
	\maketitle
	\thispagestyle{empty}

	\begin{abstract}
	\small
	For many learning problems one may not have access to fine grained label information; e.g.,
an image can be labeled as husky, dog, or even animal depending on the expertise of the annotator.
In this work, we formalize these settings and study the problem of learning from such coarse data.  
Instead of observing the actual labels from a set $\mathcal{Z}$,
we observe coarse labels corresponding to a partition of $\mathcal{Z}$ (or a mixture of partitions).

Our main algorithmic result is that essentially any problem learnable from fine grained labels
can also be learned efficiently when the coarse data are sufficiently informative.  
We obtain our result through a generic reduction for answering Statistical Queries (SQ) 
over fine grained labels given only coarse labels.  The number of coarse labels 
required depends polynomially on the information distortion due to coarsening 
and the number of fine labels $|\mathcal{Z}|$.

We also investigate the case of (infinitely many) real valued labels
focusing on  a central  problem in censored and truncated statistics:
Gaussian mean estimation from coarse data.
We provide an efficient algorithm when the sets in the partition are convex and  establish that the problem is NP-hard even for very simple non-convex sets.
	\end{abstract}
	


\newpage	
	
\section{Introduction}
\label{section:intro}
Supervised learning from labeled examples is a classical problem in machine learning and statistics:
given labeled examples, the goal is to train some model to achieve low classification error.  
In most modern applications, where we train complicated models such as neural nets, 
large amounts of labeled examples are required. Large datasets such as Imagenet, \cite{imagenet},
often contain thousands of different categories such as
animals, vehicles, etc., each one of those containing many \emph{fine grained} subcategories: animals may contain dogs and  cats
and dogs may be further split into different breeds etc.  
In the last few years, there have been many works that focus on fine grained recognition,
\cite{guo2018cnn,chen2018understanding,touvron2020grafit,qin2020learning,lei2017weakly,jiao2019weakly,jiao2020fine,bukchin2020fine,taherkhani2019weakly}.
Collecting a sufficient amount of accurately labeled training examples is a hard and expensive task 
that often requires hiring experts to annotate the examples.
This has motivated the problem of learning from \emph{coarsely} labeled datasets, where a dataset is not fully annotated with fine
grained labels but a combination of fine, e.g., cat, and coarse labels, e.g., animal, is given, \cite{JKF13,RGGG15}.

Inference from coarse data naturally arises also in unsupervised, i.e., distribution
learning settings: instead of directly observing samples from the target distribution, we observe ``representative" points
that correspond to larger sets of samples.  For example, instead of observing samples from a real valued
random variable, we round them to the closest integer.  An important unsupervised problem that fits in
the coarse data framework is censored statistics, \cite{Cohen91,wolynetz79, breen1996regression, Schneider86}.
Interval censoring, that arises in insurance adjustment applications, corresponds to observing
points in some interval and point masses at the endpoints of the interval instead of observing fine grained data
from the whole real line.  Moreover, the problem of learning the distribution of the output of neural networks with non-smooth 
activations (e.g., ReLU networks, \cite{WDS19}) also fits in our model of distribution learning with coarse data,
see \Cref{fig:convex}(c).

Even though the problem of learning from coarsely labeled data has attracted significant attention from the applied
community, from a theoretical perspective little is known.  In this work, we provide efficient algorithms
that work in both the supervised and the unsupervised coarse data settings.

\subsection{Our Model and Results}
We start by describing the generative model of coarsely labeled  data in the supervised setting.
We model coarse labels as subsets of the domain of all possible fine labels.
For example, assume that we hire an expert on dog breeds and an expert
on cat breeds to annotate a dataset containing images of dogs and cats.  
With probability $1/2$, we get samples labeled by the dog expert, i.e., labeled according to the partition
\[
\{ \mathrm{cat} = \{\mathrm{persian~cat}, \mathrm{bengal~cat}, \ldots \}, \{\mathrm{maltese~dog}\}, \{\mathrm{husky~dog}\}, \ldots\ \}\,.
\]
On the other hand, the cat expert will provide a fine grained partition over cat breeds 
and will group together all dog breeds.  Our coarse data model captures exactly this mixture of different label partitions.
\begin{definition}[Generative Process of Coarse Data with Context]
\label{definition:intro-gen-proc-context}
Let $\mathcal{X}$ 
be an arbitrary  domain, and let $\mathcal{Z} = \{1,\ldots, k\}$ be the discrete domain of all possible fine labels.
We generate coarsely labeled examples as follows:
\begin{enumerate}
\item Draw a finely labeled example $(x,z)$ from a distribution $\D$ on $\mathcal X \times \mathcal Z$.
\item Draw a coarsening partition $\mathcal{S}$ (of $\mathcal{Z}$) from a distribution $\pi$.
\item Find the unique set $S \in \mathcal{S}$ that contains the fine label $z$.
\item Observe the coarsely labeled example $(x,S)$.
\end{enumerate}
We denote $\D_{\pi}$ the distribution of the coarsely labeled example $(x, S)$.
\end{definition}
In the supervised setting, our main focus is to answer the following question.
\begin{question}
\label{que:main-question}
Can we train a model, using coarsely labeled examples $(x, S) \sim \D_\pi$, that classifies finely labeled examples 
$(x,z) \sim \D$ with accuracy comparable to that of a classifier that was trained on examples with fine grained labels?
\end{question}
\Cref{definition:intro-gen-proc-context} does not impose any restrictions on the distribution over partitions $\pi$.
It is clear that if partitions are very rough, e.g., we split $\mathcal{Z}$ into two large disjoint subsets,
we lose information about the fine labels and we cannot hope to train a classifier that performs well over
finely labeled examples.  In order for \Cref{que:main-question} to be  information theoretically 
possible, we need to assume that the partition distribution $\pi$ preserves fine-label information.
    The following definition quantifies this by stating that reasonable partition distributions $\pi$ 
    are those that preserve the total variation distance between different distributions supported on the 
domain of the fine labels $\mathcal Z$.  We remark that the following definition does not 
require $\mathcal{D}$ to be supported on pairs $(x, z)$ but is a general statement for the unsupervised
version of the problem, see also \Cref{definition:unsupervised-coarse}.
\begin{definition}[Information Preserving Partition Distribution]
\label{def:intro-information-preserving}
Let $\mathcal Z$ be any domain and let $\alpha \in (0,1]$.
We say that $\pi$ is an $\alpha$-information preserving partition distribution 
if  for every two distributions $\D^1, \D^2$ supported on $\mathcal{Z}$,
it holds that $\dtv(\D^1_{\pi}, \D^2_{\pi}) \geq \alpha \cdot \dtv(\D^1, \D^2)$,
where $\dtv(\D^1, \D^2)$ is the total variation distance of $\D^1$ and $\D^2$.
\end{definition}
For example, the partition distribution defined in the dog/cat dataset scenario, discussed before \Cref{definition:intro-gen-proc-context},
is $1/2$-information preserving, since we observe fine labels with probability $1/2$.
In this case, it is easy, at the expense of losing the statistical power of the coarse labels,
to combine the finely labeled examples from both experts in order to obtain
a dataset consisting only of fine labels.  However, our model allows the partitions to 
have arbitrarily complex combinatorial structure that makes the 
process of ``inverting" the partition transformation computationally challenging.
For example, specific fine labels may be complicated functions of coarse labels:
``medium sized" and ``pointy ears" and ``blue eyes" may be mapped to the ``husky dog" fine label.

Our first result is a positive answer to \Cref{que:main-question} in essentially full generality:
we show that concept classes that are efficiently learnable in the Statistical Query (SQ) model, 
\cite{kearns1998efficient}, are also learnable from coarsely labeled examples. 
Our result is similar in spirit with the result of~\cite{kearns1998efficient}, where it is proved that 
SQ learnability implies learnability under random classification noise.  

\begin{inftheorem}[SQ Learnability implies Learnability from Coarse Examples]
\label{infthm:sq-reduction}
Any concept class~$\mathcal{C}$ that is efficiently learnable with $M$ statistical queries from finely labeled examples $(x,z) \sim \D$, can be efficiently learned from $O(\poly(k/\alpha)) \cdot M$ coarsely labeled examples 
$(x, S) \sim \D_\pi$ under any $\alpha$-information preserving partition distribution $\pi$.
\end{inftheorem}

Statistical Queries are queries of the form $\E_{(x,z) \sim \D}[q(x,z)]$ for some query function $q(x,z)$.  
It is known that almost all known machine learning algorithms \cite{aslam1998general,blum1998polynomial,blum2005practical, dunagan2008simple,balcan2015statistical,feldman2017statistical} can be implemented in the SQ model.
In particular, in \cite{FeldmanGV15}, it is shown that (Stochastic) Gradient Descent can be simulated
by statistical queries.  This implies that our result can be applied, even in cases where it is 
not possible to obtain formal optimality guarantees, e.g., training deep neural nets.  We can train such models
using coarsely labeled data and guarantee the same performance as if we had direct access to fine labels (see also \Cref{appendix:training}).
\footnote{
Given any objective of the form
$L(\vec v)$ 
$ = \mathbb{E}_{(\vec x,y) \sim \mathcal{D}}[ \ell(\vec v; \vec x,y) ]$,
its gradients correspond to 
$\nabla_{\vec v} L(\vec v) = 
\mathbb{E}_{(\vec x,y) \sim \mathcal{D}}[ \nabla_{\vec v}\ell(\vec v; \vec x,y) ]$.  
Having Statistical Query access to the distribution of $(x,y)$, we can directly obtain estimates
of the above gradients  using the query functions 
$q_i(\vec x, y) =  (\nabla_{\vec v}\ell(\vec v; \vec x,y))_i$.
In~\cite{FeldmanGV15}, 
the precise accuracy required for specific SQ implementations of first order methods depends on the complexity of the underlying distribution and the particular objective function $\ell(\cdot)$.
}
As another application, we  consider the problem of multiclass logistic regression with coarse labels.  
It is known, see e.g., \cite{friedman2001elements}, that  given finely labeled examples $(x,z) \sim \D$, the 
likelihood objective for multiclass logistic regression is concave with respect to the weight matrix.  
Even though the likelihood objective is no-longer concave when we consider coarsely labeled examples $(x,S) \sim \D_{\pi}$,
our theorem bypasses this difficulty and allows us to efficiently perform multiclass logistic regression with coarse
labels.

Formally, we design an algorithm (\Cref{algo:sq}) that, given coarsely labeled examples $(x,S)$, efficiently simulates
statistical queries over finely labeled examples $(x,z)$.  Surprisingly, the runtime and sample
complexity of our algorithm do not depend on the combinatorial structure of the partitions, 
but only on the number of fine labels $k$ and the information preserving parameter $\alpha$ of the partition distribution $\pi$.
\begin{theorem}
[SQ from Coarsely Labeled Examples] \label{theorem:intro-reduction}
Consider a distribution $\D_\pi$ over coarsely labeled examples in $\reals^d \times [k]$, (see \Cref{definition:intro-gen-proc-context})
with $\alpha$-information preserving partition distribution $\pi$.
Let $q: \reals^d \times [k] \to [-1,1]$ be a query function, that can be evaluated on any input in time $T$, and $\tau, \delta \in (0,1)$.  
There exists an algorithm (\Cref{algo:sq}), that draws 
$N = \wt{O}(k^4/(\tau^3 \alpha^2) \log(1/\delta))$ coarsely labeled examples from $\D_{\pi}$ 
and, in $\poly(N, T)$ time, computes an estimate 
$\hat r$ such that, with probability at least $1-\delta$, it holds
\(
\big| \E_{(x, z) \sim \D}[ q(x, z) ] - \hat r \big| \leq \tau\,.
\)
\end{theorem}

\paragraph{Learning Parametric Distributions from Coarse Samples.} 
In many important applications, instead of a discrete distribution over fine labels, a continuous parametric
model is used. A popular example is when the domain $\mathcal{Z}$ of \Cref{definition:intro-gen-proc-context} is the entire Euclidean space $\mathbb{R}^d$, and the distribution of finely labeled examples is a Gaussian distribution whose parameters possibly depend on the context $x$.  Such censored regression settings are known as Tobit models 
\cite{tobin1958estimation,maddala1986limited,gourieroux2000econometrics}.  Lately, significant progress
has been made from a computational point of view in such censored/truncated settings in the distribution
specific setting, e.g., when the underlying distribution is Gaussian~\cite{DGTZ18,KTZ19}, mixtures of Gaussians~\cite{NP19}, linear 
regression~\cite{daskalakis2019computationally, ilyas2020theoretical, daskalakis2020truncated}.
In this distribution specific setting, we consider the most fundamental problem of learning the mean of a 
Gaussian distribution  given coarse data.  
\begin{definition}[Coarse Gaussian Data]
\label{definition:intro-gaussian-coarse}
Consider the Gaussian distribution $\N(\vmu^{\star})$, with mean $\vmu^{\star} \in \reals^{\dim}$ and 
identity covariance matrix. 
We generate a sample as follows:
\begin{enumerate}
    \item Draw $\vec z$ from $\N(\vmu^{\star})$.
    \item Draw a partition $\S$ (of $\reals^d$) from $\pi$.
    \item Observe the set $S \in \S$ that contains $\vec z$.
\end{enumerate}
We denote the distribution of $S$ as $\N_{\pi}(\vmu^\star)$.
\end{definition}
\begin{remark}
We remark that we only require membership oracle access to the subsets of the partition $\mathcal S$.
A set $S \subseteq \reals^d$ corresponds to a membership oracle $\mathcal{O}_S : \reals^d \to \{0,1\}$
that given $\vec x \in \reals^d$ outputs whether the point lies inside the set $S$ or not. 
\end{remark}
We first study the above problem, from a computational viewpoint.
For the corresponding problems in censored and truncated statistics
no geometric assumptions are required for the sets: in \cite{DGTZ18} 
it was shown that an efficient algorithm exists for arbitrarily
complex truncation sets.  In contrast in our more general model of coarse data 
we show that having sets with geometric structure is necessary.  In
particular we require that every set of the partition is convex, see \Cref{fig:convex}(b,c).
We show that when the convexity assumption is dropped, learning from coarse data is a computationally hard problem 
even under a mixture of very simple sets.
\begin{theorem} 
[Hardness of Matching the Observed Distribution with General Partitions]
\label{theorem:impossibility-gaussian}
Let $\pi$ be a general partition distribution. Unless $\textsc{RP} = \NP$, no algorithm with sample access to $\N_{\pi}({\vmu}^\star),$ can compute,
in $\poly(d)$ time,  a $\wt{\vmu} \in  \reals^d$ such that $\dtv( \N_{\pi}(\wt\vmu), \N_{\pi}(\vmu^\star)) < 1/d^c$
for some absolute constant $c > 1$.
\end{theorem}
We prove our hardness result using a reduction from the well known \maxcut problem, which is known to be NP-hard, 
even to  approximate \cite{haastad2001some}. In our reduction, we use partitions that consist of simple sets: fat hyperplanes, 
ellipsoids and their complements: the computational hardness of this problem is  rather inherent and not 
due to overly complicated sets.

\begin{figure}[ht]
   \centering
    \begin{subfigure}{0.3\textwidth}
    \centering
    \begin{tikzpicture}[scale=0.38]
        \def\pts{}
        
        \edef\pts{\pts, (-3,0.2)}
        \edef\pts{\pts, (-1,0.2)}
        \edef\pts{\pts, (1,0.2)}
        \edef\pts{\pts, (3,0.2)}
        
        \xintForpair #1#2 in \pts \do{
          \edef\pta{#1,#2}
          \begin{scope}
            \xintForpair \#3#4 in \pts \do{
              \edef\ptb{#3,#4}
              \ifx\pta\ptb\relax 
                \tikzstyle{myclip}=[];
              \else
                \tikzstyle{myclip}=[clip];
              \fi;
              \path[myclip] (#3,#4) to[half plane] (#1,#2);
            }
            \clip (-\maxxy,-\maxxy) rectangle (\maxxy,\maxxy); 
            \pgfmathsetmacro{\randhue}{rnd}
            \definecolor{randcolor}{hsb}{\randhue,.5,1}
            \fill[red, opacity=0.1] (#1,#2) circle (4*\biglen); 
            \fill[draw=black] (#1,#2) circle (2pt); 
          \end{scope}
        }
        \pgfresetboundingbox
        \draw (-\maxxy,-\maxxy) rectangle (\maxxy,\maxxy);
  \end{tikzpicture}
  \caption{Non-Identifiable Case}
 \end{subfigure}%
 \begin{subfigure}{0.3\textwidth}
    \centering
    \begin{tikzpicture}[scale=0.38]
    \def\pts{}
    \pgfmathsetseed{1908}
    \xintFor* #1 in {\xintSeq {1}{12}} \do{
      \pgfmathsetmacro{\ptx}{.9*\maxxy*rand} 
      \pgfmathsetmacro{\pty}{.9*\maxxy*rand} 
      \edef\pts{\pts, (\ptx,\pty)} 
    }
    \edef\pts{\pts, (0.4,0.2)}
    \edef\pts{\pts, (0.2,0.7)}
    \edef\pts{\pts, (-3.2,0.7)}
    \edef\pts{\pts, (-3.2,0)}
    \edef\pts{\pts, (-3.2,-1.0)}
    
    \xintForpair #1#2 in \pts \do{
      \edef\pta{#1,#2}
      \begin{scope}
        \xintForpair \#3#4 in \pts \do{
          \edef\ptb{#3,#4}
          \ifx\pta\ptb\relax 
            \tikzstyle{myclip}=[];
          \else
            \tikzstyle{myclip}=[clip];
          \fi;
          \path[myclip] (#3,#4) to[half plane] (#1,#2);
        }
        \clip (-\maxxy,-\maxxy) rectangle (\maxxy,\maxxy); 
        \pgfmathsetmacro{\randhue}{rnd}
        \definecolor{randcolor}{hsb}{\randhue,.5,1}
        \fill[blue, opacity=0.1] (#1,#2) circle (4*\biglen); 
        \fill[draw=black] (#1,#2) circle (2pt); 
      \end{scope}
    }
    \pgfresetboundingbox
    \draw (-\maxxy,-\maxxy) rectangle (\maxxy,\maxxy);
  \end{tikzpicture}
    \caption{Convex Partition Case}
\end{subfigure}%
 \begin{subfigure}{0.3\textwidth}
    \centering
    \begin{tikzpicture}[scale=0.38]
    
    \draw[name path=y] (0,0) -- (0,4);
    \draw[name path=x] (0,0) -- (4,0);
    
    \draw[black, name path=yneg] (0,0) -- (-4,0);
    \draw[black, name path=xneg] (0,0) -- (0,-4);
    \tikzfillbetween[of= yneg and xneg ]{blue, opacity=0.1};
    
    \draw[name path=yneg2] (-4,0) -- (-4,-4);
    \draw[name path=xneg2] (-4,-4) -- (0,-4);
    \tikzfillbetween[of= yneg2 and xneg2 ]{blue, opacity=0.1};
    
    \foreach \i in {1,...,16}
    {
        \draw[black] (0, 0.25*\i) -- (-4, 0.25*\i);
    }
    
    \foreach \i in {1,...,16}
    {
        \draw[black] (0.25*\i, 0) -- (0.25*\i, -4);
    }
    
    \foreach \i in {1,...,19}
    {
        \foreach \j in {1,...,19}
        {
             \fill[black] (0.2*\i,0.2*\j) circle (1.5pt); 
        }
    }
    
    \pgfresetboundingbox
    \draw (-\maxxy,-\maxxy) rectangle (\maxxy,\maxxy);
    
    \end{tikzpicture}
    \caption{ReLU Case}
    \end{subfigure}
    \caption{
    (a) is a very rough partition, that makes learning the mean impossible: 
    Gaussians $\N((0,z))$ centered along the same vertical line $(0,z)$ assign exactly the same probability to all cells of the
partitions and therefore, $\dtv(\N_\pi((0, z_1)), \N_\pi((0,z_2)) ) = 0$: it is impossible to learn the second coordinate of the mean.  
    (b) is a convex partition of $\reals^2$, that makes recovering the Gaussian possible.
    (c) is the convex partition corresponding to the output distribution of one layer ReLU networks. 
    When both coordinates are positive, we observe a fine sample (black points correspond to singleton sets). When exactly one 
    coordinate (say $\vec x_1$) is positive, we observe the line $L_z = \{ \vec x : \vec x_2 < 0, \vec x_1 = z >0\}$ that corresponds to the ReLU output $(\vec x_1, 0)$.  If both coordinates are negative, we observe the set $\{ \vec x : \vec x_1<0, \vec x_2 < 0\}$, that corresponds to the point $(0,0)$. 
    }
    \label{fig:convex}
\end{figure}

On the positive side, we identify a geometric property that enables us to design a computationally
efficient algorithm for this problem: namely we require all the sets of the partitions to be \emph{convex},
e.g., \Cref{fig:convex}(b,c). We remark that having finite or countable subsets, 
is not a requirement of our model.  For example, we can handle convex partitions of the form (c) that
correspond to the output distribution of a ReLU neural network,  see \cite{WDS19}.
We continue with our theorem for learning Gaussians from coarse data.
\begin{inftheorem}
[Gaussian Mean Estimation with Convex Partitions]
\label{inftheorem:intro-mean-estimation-gaussian}
Let $\eps \in  (0,1)$.
Consider the generative process of coarse $\dim$-dimensional Gaussian data $\N_{\pi}(\vmu^{\star})$.
Assume that the partition distribution $\pi$ is  $\alpha$-information preserving and is supported on convex partitions of $\reals^d$.
Then, the empirical log-likelihood objective
\[
\L_N(\vmu) = \frac{1}{N} \sum_{i=1}^N \log \N(\vmu; S_i)
\]
is concave with respect to $\vec \mu$ for $S_i \sim \N_{\pi}(\vmu^\star)$. Moreover, 
it suffices to draw $N = \wt{O}(d/(\eps^2 \alpha^2))$ samples from $\N_{\pi}(\vmu^{\star})$ so that the maximizer $\wt{\vmu}$ of the empirical log-likelihood satisfies
\[
\dtv(\N(\wt{\vmu}), \N(\vmu^\star)) \leq \eps\,,
\]
with probability
at least $99\%$.
\end{inftheorem}


Our algorithm for mean estimation of a Gaussian distribution relies on the log-likelihood being concave when the partitions
are convex.  We remark that, similar to our approach, one can use the concavity of likelihood to get efficient
algorithms for regression settings, e.g., Tobit models, where the mean of the Gaussian is given by a linear function of the context 
$\vec A \vec x$ for some unknown matrix $\vec A$.

\subsection{Related Work}

Our work is closely related to the literature of learning from censored-truncated data and learning with noise. There has been a large number of recent works dealing inference with truncated data from a Gaussian distribution~\cite{DGTZ18,KTZ19}, mixtures of Gaussians~\cite{NP19}, linear regression~\cite{daskalakis2019computationally, ilyas2020theoretical, daskalakis2020truncated}, sparse Graphical models~\cite{bhattacharyya2020efficient} or Boolean product distributions~\cite{fotakis2020efficient},
and non-parametric estimation \cite{daskalakis2021statistical}. A significant feature of our work is that it can capture the closely related field of censored statistics~\cite{Cohen91, breen1996regression, wolynetz79}. 

The area of robust statistics~\cite{huber2004robust} is also very related to our work as it also
deals with biased data-sets and aims to identify the distribution that generated the data. Recently, there has been a large volume of theoretical work for computationally-efficient robust estimation of high-dimensional distributions~\cite{ DKK+16b,CSV17,LRV16,DKK+17,DKK+18,KlivansKM18,hopkins2019hard,DKK+19-sever,cheng2020high,bakshi2020outlier} in the presence of arbitrary corruptions to a small $\varepsilon$ fraction of the samples.

The line of research dealing with statistical queries~\cite{kearns1998efficient, blum1998polynomial,FeldmanPV15,FeldmanGV15,Feldman16,feldman2017statistical,diakonikolas2017statistical, diakonikolas2020algorithms} is closely related to one of our main results (\Cref{theorem:intro-reduction}). It is generally believed that SQ algorithms capture all
reasonable machine learning algorithms~\cite{aslam1998general,blum1998polynomial,blum2005practical, dunagan2008simple,feldman2017statistical,balcan2015statistical,FeldmanGV15} 
and there is a rich line of research indicating SQ lower-bounds for these classes of algorithms~\cite{feldman2017statistical,diakonikolas2017statistical,shamir2018distribution,vempala2019gradient, diakonikolas2020algorithms,diakonikolas2020near,goel2020superpolynomial,goel2020statistical}.

Learning from coarse labels is also referred in the ML literature as Partial Label Learning ~\cite{cour2011learning,chen2014ambiguously,yu2016maximum} (a weakly supervised learning problem where each training example is associated with a set of candidate labels among which only one is true). 
We refer to \Cref{appendix:partial-literature} for an extensive discussion.

\section{Notation and Preliminaries}
We let $[n] = \{1, \ldots, n\}$ and $[0..n] = \{0, \ldots, n\}$. We use lowercase bold letters $\vec x$ to denote vectors and capital bold letters $\vec X$ for matrices. We let $\vec x_i$ be the $i$-th coordinate of $\vec x$.
We let $\|\vec x\|_p$ denote the $L_p$ norm of $\vec x$.
We denote the indicator function $\vec 1_S(\vec x) = \vec 1\{ \vec x \in S\}$ for some set $S \subseteq \reals^d$.
We let $\mathrm{sgn}(\cdot)$ denote the sign function and we slightly overload the notation as follows: $\mathrm{sgn} : \reals^d \to \{-1,+1\}^d$ stands for the sign function applied to each coordinate of a vector $\vec x \in \reals^d$.
For a graph $G$, we usually let $\vec L_G$ denote its Laplacian matrix.
We denote $\mathcal{B}(\vec x, \rho)$ the Euclidean ball of radius $\rho$ centered at $\vec x$; we simply refer to $\mathcal{B}$ if the radius and the center are clear from the context and we denote the associated sphere $\partial \mathcal{B}$, i.e., its boundary.
The probability simplex is denoted by $\Delta^n$ and discrete distributions $\D$ supported on $[n]$ will usually be represented by their associated probability vectors $\vec p \in \Delta^n$. For any distribution $\D$, we overload the notation and we use the same notation for the corresponding density and denote $\D(S) = \sum_{x \in S}\D(x)$ for any $S \subseteq [n]$. 
We denote the support of the probability distribution $\D$ by $\supp(\D)$.
The $\dim$-dimensional Gaussian distribution will be denoted by $\N(\vmu, \vSigma)$. 
When the covariance matrix is known, we simplify to $\N(\vmu)$.
For a set $S \subseteq \reals^d,$ we let $\N_S$ denote the conditional Gaussian distribution on the set $S$, i.e., $\N_S(\vec \mu, \vec \Sigma; \vec x ) = \vec 1\{ \vec x \in S \} \N(\vmu, \vSigma; \vec x)/ \N(\vmu, \vSigma; S)$.
We denote $\Phi$ (resp. $\phi$) the cdf (resp. pdf) of the standard Normal distribution.
The total variation distance of $\vec p, \vec q \in \Delta^n$ is $\dtv(\vec p, \vec q) = \max_{S \subseteq [n]} \vec p(S) - \vec q(S) = \|\vec p - \vec q\|_1/2$. 
For a random variable $x$, we let $\E[x], \Var(x), \Cov(x)$ be the expected value, the variance and the covariance of $x$. For a joint distribution $\D$ of two random
variables $x$ and $z$ over the 
space $\mathcal{X} \times \mathcal{Z}$, we let $\D_x$ (resp. $\D_z$) be the marginal 
distribution of $x$ (resp. $z$). 
Let $\D$ be a joint distribution over labeled examples $\mathcal{X} \times \mathcal{Z}$, with 
$\mathcal{X}$ be the input space and $\mathcal{Z}$ the label space. A statistical query (SQ) oracle $\mathrm{STAT}(\D, \tau)$ with tolerance parameter $\tau \in [0,1]$ takes as input a statistical query defined by a real-valued
function $q : \mathcal{X} \times \mathcal{Z} \rightarrow [-1,1]$
and outputs an estimate of $\E_{(x,z) \sim \D}[q(x,z)]$ that is accurate to within an additive $\pm \tau$.

\section{Supervised Learning from Coarse Data}
\label{section:supervised}

In this section, we consider the problem of \emph{supervised} learning from coarse data. 
In this setting, there exists some underlying distribution over finely labeled examples, $\D$.
However, we have sample access only to the distribution associated with coarsely labeled examples $\D_\pi$,
see \Cref{definition:intro-gen-proc-context}.   
As discussed in \Cref{section:intro}, under this setting, even problems that are naturally convex when we have access to 
examples with fine labels, become non-convex when we introduce coarse labels (e.g., multiclass logistic regression). 
The main result of this section is \Cref{theorem:intro-reduction},
which allows us to compute statistical queries over finely labeled examples. 

\subsection{Overview of the Proof of \Cref{theorem:intro-reduction}}

In order to simulate a statistical query we take a two step approach. 
Our first building block considers the unsupervised version of the problem, 
see \Cref{definition:unsupervised-coarse}, i.e., 
we marginalize the context $x$ and try to learn the distribution of the fine labels $z$
given coarse samples $S$.  This can be viewed as learning a general discrete distribution
supported on $\mathcal{Z} = \{1, \ldots, k\}$  given coarse samples, i.e., subsets of $\mathcal{Z}$.
We show that, when the partition distribution $\pi$ is $\alpha$-information preserving,
this can be done efficiently, see \Cref{prop:unsupervised}.  Our algorithm (\Cref{algo:sq})
exploits the fact that even though in general having coarse data results in 
non-concave likelihood objectives, when we consider parametric models 
(see, for example, the case of logistic regression in \Cref{appendix:logistic}),
this is not true when we maximize over all discrete distributions.
In \Cref{prop:unsupervised}, we show that $\wt{O}(k/(\eps \alpha)^2)$ samples are sufficient
for this step. For the details of this step, see \Cref{subsection:marginal}.

Using the above algorithm, one could try to separately learn the
marginal distribution over $x$, $\D_x$ and the distribution of the fine 
labels $z$ \emph{conditional on some fixed}  $x$; let us denote this distribution as $\D_z^x$.  
Then one could generate finely labeled examples $(x,z)$ and use them to estimate 
the query $\E_{(x,z) \sim \D}[q(x,z)]$.  The reason that this naive approach fails
is that it requires many coarse examples $(x, S)$ with exactly the same value
of $x$.  Unless the domain $\mathcal{X}$ is very small, the probability that we observe
samples with the same value of $x$ is going to be tiny.   In order to overcome this obstacle, at a high level, 
our approach is to split the domain $\mathcal{X}$ into larger sets and then, 
learn the conditional distribution of the labels, not on a fixed point $x$,
but on these larger sets of non-trivial mass.  


Intuitively, in order to have an effective partition of the domain $\mathcal{X}$, we want 
to group together points $x$ whose values $q(x, z)$ are close.  Since $z$ belongs
in a discrete domain $\mathcal{Z} = [k]$, we can decompose the query $q(x,z)$ as
$q(x,z) = \sum_{i=1}^k q(x, i) \vec 1\{z = i\}$.  We estimate the value 
of $\E_{(x,z) \sim \D}[ q(x, i) \vec 1\{z = i\}]$ separately.  
To find a suitable  reweighting of the domain $\mathcal{X}$,  we perform rejection sampling, 
accepting a pair $(x, S) \sim \D$ with probability $q(x,i)$ 
\footnote{It is easy to handle the case where this  function takes negative values, see the proof of \Cref{theorem:intro-reduction}.}: 
points $x$ that have small value $q(x, i)$ contribute less in the expectation and are less likely to be sampled.  
After performing this rejection sampling process based on $x$, we have pairs 
$(x,S)$, conditional that $x$ was accepted.  Now, using our previous maximum likelihood learner of \Cref{prop:unsupervised}
we learn the marginal distribution over fine labels and use it to answer the query.
We provide the details of this rejection sampling step in the full proof of \Cref{theorem:intro-reduction},
see \Cref{subsection:proof-of-reduction}.

For a description of the corresponding algorithm that simulates statistical queries, see \Cref{algo:sq}.
To keep the presentation simple we state the
algorithm for the case where the query function $q(x, z)$ is positive.  It is straightforward to generalize
it for general queries, see \Cref{subsection:proof-of-reduction}.
\begin{algorithm}[ht!] 
\caption{Statistical Queries from Coarse Labels.}
\label{algo:sq}
\begin{algorithmic}[1]
\State \textbf{Input:} Query $q: \mathcal{X} \times \mathcal{Z} \mapsto (0,1]$, tolerance $\tau \in [0,1]$, confidence $\delta \in[0, 1]$.
 
\State \textbf{Oracle}: Access to 
coarsely labeled samples $(x, S) \sim \D_\pi$, $\pi$ is $\alpha$-information preserving.
\State \textbf{Output:} Estimate $\wh{r}$ such that $\big| \E_{(x, z) \sim \D}[ q(x, z) ] - \wh r \big| \leq \tau$ with probability 
at least  $1-\delta$. 
\vspace{2mm}
\Procedure{StatQuery}{$q, \tau, \delta$}
 \State Compute $\wh{r}_i \gets \mathrm{SQ}(q,i,O(\tau/k), \delta/k)$ for any $i \in \mathcal{Z}$.
\State Output $\wh{r} \gets \sum_{i=1}^{k} \wh{r}_i$. 
\EndProcedure
\vspace{2mm}
\Procedure{SQ}{$q, i, \rho, \delta$}
\State Draw $N_1 = \wt{\Theta}\big(\frac{\log(1/\delta)}{\rho^2}\big)$ samples $(x_j, S_j)$ from $\D_{\pi}$. 
\State Compute $\wh{\mu}_i \gets \frac{1}{N_1} \sum_{j=1}^{N_1} q(x_j, i)$. 
\State \textbf{if}~{$\wh{\mu}_i \leq \rho$}~\textbf{do}
\State~~~ Output $\wh{r}_i \gets 0$.
\State \textbf{end}
\State Draw $N_2 = \wt{\Theta} \big( \frac{k \log(1/\delta)}{\rho^3\alpha^2} \big)$ samples $(x_j, S_j)$ from $\D_{\pi}$. \Comment{\emph{$\wt{\Theta} \big(\frac{k^4 \log(1/\delta)}{\tau^3 \alpha^2} \big)$ examples overall}.}
\State $T_{accept} \gets \emptyset$. \Comment{\emph{Training set of accepted samples}.}
\State Add $S_j$ in $T_{accept}$ with probability $q(x_j, i)$, $\forall j \in [N_2]$. \Comment{\emph{Rejection Sampling Process}.}
\State Compute $\wt{\D}$ using \Cref{prop:unsupervised} with input $(T_{accept}, \rho, \delta).$ 
\State Output $\wh{r}_i \gets \wh{\mu}_i \cdot \wt{\D}(i) $. 
\EndProcedure

\end{algorithmic}
\end{algorithm}

\begin{remark}
[Empirical Likelihood Approach]
One could try to use the empirical likelihood directly 
over the coarsely labeled data (as defined in~\cite{owen2001empirical}).
However, in general, these empirical likelihood objectives are non-convex 
when the data are coarse and therefore it is computationally hard to optimize them directly.
Our approach for simulating statistical queries consists of two ingredients: 
reweighting the feature space via rejection sampling in order to group together points 
and learning discrete distributions from coarse data.  To learn the discrete distributions 
(\Cref{subsection:marginal}), we use a (direct) empirical likelihood approach similar to that of ~\cite{owen1988empirical, owen1990empirical, owen2001empirical}.  
However, our main contribution is the use of rejection sampling to 
reduce the initial non-convex problem to the special case of learning a discrete distribution 
(with small support) from coarse data which, as we prove, is a tractable (convex) problem. 
For more connections with censored statistics techniques, we refer the reader to~\cite{thomas1975confidence, owen1988empirical, gill1997coarsening, owen2001empirical}.
\end{remark}


\subsection{Learning Marginals Over Fine Labels}
\label{subsection:marginal}
In this subsection, we deal with \emph{unsupervised} learning from coarse data in discrete domains. 
Although this is an ingredient of our main result for simulating statistical queries in
a supervised setting where labeled data $(x,S)$ are given, the result of this section does
not depend on the points $x$ and concerns the unsupervised version of the problem.
To keep the notation simple, we will use $\D$ to denote a distribution over finite labels $\mathcal{Z}$.
\begin{definition}
[Generative Process of Coarse Data]
\label{definition:unsupervised-coarse}
Let $\mathcal Z$ be a discrete domain and $\D$ be a distribution supported on $\mathcal{Z}$.
Moreover, let $\pi$ be a distribution supported on partitions of $\mathcal{Z}$.
We consider the following generative process:
\begin{enumerate}
\item Draw $z$ from $\D$.
\item Draw a partition $\mathcal{S}$ from the distribution over all partitions $\pi$.
\item Observe the set $S\in \mathcal{S}$ that contains $z$.
\end{enumerate}
We denote the distribution of $S$ as $\D_{\pi}$.
\end{definition}
The assumption that we require is that the partition distribution $\pi$ is $\alpha$-information preserving, see \Cref{def:intro-information-preserving}.  
At this point we give some examples of  information preserving partition distributions.
We first observe that $\alpha = 0$ if and only if the problem is not identifiable. 
For instance, if $\pi$ is supported only on the partition $\S = \{ \{1,2\}, \{3, \ldots, k\} \}$, 
the problem is not identifiable, since, for example, the fine label $1$ is indistinguishable from the fine label $2$. 
The value $\alpha = 1$ is attained when the partition totally preserves the distribution distance. 
Intuitively, the value $1-\alpha$ corresponds to the distortion that the coarse labeling introduces to a finely labeled dataset.
  

In many cases most fine labels may be missing.  Consider two data providers that use different methods to  round their samples. The rounding's uncertainty can be viewed as a coarse labeling of the data. 
Assume that we add discrete (balanced Bernoulli) noise $\xi$ to some true value $x \in [0..k]$.
Consider two partitions $\{\S_1, \S_2\}$ with $\S_1 = \{ \{0,1\}, \{2,3\}, \dots, \{k-1,k\}, \{k+1\} \}$ and $\S_2 = \{ \{0\}, \{1,2\},\dots \{k-1,k\}\}$. Observe that, when $x + \xi$ is odd, we can think of the rounded sample, as a draw from 
$\S_1$ and when $x + \xi$ is even, as a 
draw from $\S_2$.  This example shows that we can capture the problem of 
deconvolution of two distributions $\D_1,\D_2$, where one of them is known and we observe samples $x_1+x_2, x_i \sim \D_i$.









The following proposition establishes the sample complexity of unsupervised learning 
of discrete distributions with coarse data. 
Our goal is to compute an estimate 
of the discrete distribution $\D^\star$ with probability vector $\vec p^\star \in \Delta^k$ from $N$ coarse samples $S_1,\dots, S_N$ drawn from the distribution
$\D_{\pi}^\star$.  Our algorithm maximizes the empirical likelihood. 
Analyzing the empirical log-likelihood objective $\mathcal{L}_N(\vec p) = \frac{1}{N}\sum_{n=1}^{N}
\log \big(\sum_{i \in S_n} \vec p_i \big)$, where $\vec p \in \Delta^k$ is a guess probability vector, 
we observe that the problem is concave and, therefore, can be efficiently optimized (e.g., by gradient descent). 
\begin{proposition}
\label{prop:unsupervised}
Let $\mathcal{Z}$ be a discrete domain of cardinality $k$ and let $\D$ be a distribution
supported on $\mathcal{Z}$. 
Moreover, let $\pi$ be an $\alpha$-information preserving
partition distribution for some $\alpha \in (0,1]$.
Then, with $N = \wt{O}(k/(\eps^2\alpha^2) \log(1/\delta))$ samples from 
$\D_{\pi}$ and in time polynomial in the number of samples $N$,
we can compute a distribution $\wt{\D}$ supported 
on $\mathcal{Z}$ such that  $\dtv(\wt{\D}, \D) \leq \eps$.
\end{proposition}
\begin{proof}
Let $\D^\star$ be the target discrete distribution, supported on a discrete domain of size $k$, and let $\vec p^\star \in \Delta^k$ be the corresponding probability vector.
For some distribution $\D$ supported on a discrete domain of size $k$,
we define the following population log-likelihood objective.
\begin{equation}
\label{eq:unsupervised-ll}
\mathcal{L}(\D) = \E_{S \sim \D^{\star}_{\pi}} [\log \D(S)] 
=
\E_{S \sim \D^{\star}_\pi} \Big[\log \big(\sum_{i \in S} \D(i) \big)\Big] \,.
\end{equation}
Since $\D$ is a discrete distribution for simplicity we may identify 
with its probability vector $\vec p$, where $\vec p_i = \D(i)$.
Therefore, for any $\vec p$ in the probability simplex $\Delta^k$,
we  define
\begin{equation}
\label{eq:unsupervised-ll-simplex}
\mathcal{L}(\vec p) = \E_{S \sim \D^{\star}_\pi} \Big[\log \sum_{i \in S} \vec p_i \Big] \,.
\end{equation}
The corresponding empirical log-likelihood objective after drawing $N$ independent samples 
$S_1,\ldots, S_N$ from $\D_\pi^{\star}$ is given by
\begin{equation}
\label{eq:unsupervised-ll-simplex-empirical}
\mathcal{L}_N(\vec p) = \frac{1}{N}\sum_{n=1}^{N}
\log \left(\sum_{i \in S_n} \vec p_i \right)  \,.
\end{equation}
We first observe that the log-likelihood (both the population and the empirical) 
is a concave function and therefore can be efficiently optimized (e.g., by gradient descent).
Thus, our main focus in this proof is to bound its sample complexity.
We first observe that when the guess $\vec p \in \Delta^k$ has some very biased 
coordinates, i.e., for some subset $S$ the corresponding $\vec p_i$'s are close to $0$, 
the probability of a set $S$, $\sum_{i \in S} p_i$ will be close to zero and
therefore $\log \big(\sum_{i \in S} \vec p_i \big)$ will be large.  Thus, we have to restrict
our search to a subset of the probability simplex, i.e., have 
$\vec p_i \geq \eps/k$.  We set $\wt \Delta^k = \{ \vec p \in \Delta^k, \vec p_i \geq \eps/k \text{ for all $i=1,\ldots,k$ } \}$.
We now prove that, given roughly $k/(\eps^2 \alpha^2)$ samples, we can guarantee that probability vectors 
that are far from the optimal vector $\vec p^{\star}$ will also be significantly sub-optimal 
in the sense that they are far from being maximizers of the empirical log-likelihood.
\begin{claim}
Let $N \geq \wt{\Omega}( k/(\eps^2 \alpha^2) \log(1/\delta))$.
With probability at least $1-\delta$, we have that, for every $\vec p \in \wt \Delta^k$
such that $\| \vec p - \vec p^{\star} \|_1 \geq \eps$, it holds 
\[
\max_{\vec q \in \wt{\Delta}^k} \mathcal{L}_N(\vec q) - \mathcal{L}_N(\vec p)  \geq \Omega \left((\eps \alpha)^2\right)\,.
\]
\end{claim}
\begin{proof}
We first construct a cover of the probability simplex $\wt \Delta^k$ by discretizing each coordinate $\vec p_i$
to integer multiples of $O((\eps^{3/2}\alpha/k)^2)$.  The resulting cover $\mathcal{C}$ contains $O((k/(\eps^{3/2} \alpha))^{2k})$ elements.
We first observe that we can replace any element $\vec p \in \wt \Delta^k$ with an element $\vec p'$ inside our cover $\mathcal{C}$ 
without affecting the value of the objective $\mathcal{L}_N(\vec p)$ by a lot.  In particular, 
using the fact that $x \mapsto \log(x)$ is $1/r$-Lipschitz in the interval $[r, +\infty)$, we have
that for any set $S \subseteq \{1,\ldots,k\}$ it holds
\[
\Big |\log\Big (\sum_{i \in S} \vec p_i\Big ) - \log\Big ( \sum_{i \in S} \vec q_i\Big ) \Big | \leq 
\frac{1}{\sum_{i \in S} \vec p_i} \Big|\sum_{i \in S} (\vec p_i - \vec q_i) \Big|
\leq \frac{k}{\eps} \|\vec p - \vec q\|_1 \,,
\]
where we used the fact that, since $\vec p \in \wt \Delta^k$, it holds $\vec p_i \geq \eps/k$.
Therefore, when we round each coordinate of a vector $\vec p$ to the closest integer multiple of 
$O((\eps^{3/2}\alpha/k)^2)$ we get a vector $\vec p' \in \mathcal{C}$ such that for any set $S$ it holds
$ |\log(\sum_{i \in S} \vec p_i) - \log( \sum_{i \in S} \vec q_i)| \leq \eps^2 \alpha^2/6$ which implies
that the empirical log-likelihood satisfies $|\mathcal{L}_N(\vec p) - \mathcal{L}_N(\vec p')| \leq \eps^2 \alpha^2/6$.
We will now show that, with high probability, any element $\vec p$ of the cover $\mathcal{C}$ 
such that $\|\vec p - \vec p^{\star}\|_1 \geq \eps$,
satisfies $ \mathcal{L}_N(\vec p^{\star})-\mathcal{L}_N(\vec p) \geq \eps^2 \alpha^2/2$.
We will use the following concentration result on likelihood ratios.
\begin{lemma}
[Proposition 7.27 of \cite{massart2007concentration}]
\label{lemma:massart}
Let $\D_1, \D_2$ be two distributions (on any domain) with positive density functions $f,g$ respectively.
For any $x \in \mathbb{R},$ it holds 
\[
\Pr_{x_1,\ldots, x_N \sim \D_1}\left[ \frac{1}{N} \sum_{n=1}^N \log \frac{f(x_n)}{g(x_n)} \leq (\dtv(\D_1, \D_2))^2 - 2 x/N\right]
\leq e^{-x} \,.
\]
\end{lemma}
\noindent Using the above lemma with $x = O(\log(|\mathcal{C}|/\delta)) = O(k \log(k/(\eps \delta)))$ 
and
\[
N = \Theta( k \log(k /(\eps \delta))/(\alpha^2 \eps^2))\,,
\]
we obtain that, with probability at least $1-\delta/|\mathcal{C}|$,
it holds $\mathcal{L}_N(\vec p^{\star}) - \mathcal{L}_N(\vec p) \geq \dtv(D_\pi, D^{\star}_\pi)^2 - \alpha^2 \eps^2/2$.
From the union bound, we obtain that the same is true for all vectors $\vec p \in \mathcal{C}$ with
probability at least $1-\delta$.
We are now ready to finish the proof of the claim.  
Let $\vec p \in \wt \Delta^k$ be any probability vector such that $\|\vec p - \vec p^{\star}\|_1 \geq \eps$.
Let $\bar{\vec p} \in \wt \Delta^k$ be the maximizer of the empirical likelihood
constrained on $\wt \Delta^k$, i.e., $\bar{\vec p} = \arg \max_{\vec q \in \wt{\Delta}^k} \mathcal{L}_N(\vec q)$ and let $\wt {\vec p}^{\star}$ be the closest vector of the cover $\mathcal{C}$ to $\vec p^{\star}$.  
We have
\[
\mathcal{L}_N(\bar{\vec p}) 
-
\mathcal{L}_N(\vec p) 
\geq
\mathcal{L}_N(\wt{\vec p}^{\star}) 
-
\mathcal{L}_N(\vec p) 
\geq 
\mathcal{L}_N({\vec p}^{\star}) - \eps^2 \alpha^2/6
-
\mathcal{L}_N(\vec p) \,.
\]
The first inequality holds since both $\bar{\vec p}$ and $\wt{\vec p}^\star$ lie in $\wt{\Delta}^k$. The 
second inequality holds since we can replace the point of the cover $\wt{\vec p}^\star \in \mathcal{C}$, 
with each closest point in the simplex $\vec p^\star$ without affecting the likelihood value by a lot. 
Finally, since $\vec p$ lies in $\wt \Delta^{k}$, we can replace it with a point $\vec p'$ in the cover 
with $\| \vec p' - \vec p^\star \|_1 \geq \eps$, and get that 
\[
\mathcal{L}_N(\bar{\vec p}) 
-
\mathcal{L}_N(\vec p) 
\geq
\mathcal{L}_N({\vec p}^{\star}) - \eps^2 \alpha^2/6
-
\mathcal{L}_N(\vec p') - \eps^2\alpha^2/6 \,,
\]
and, since $\mathcal{L}_N(\vec p^{\star})-\mathcal{L}_N(\vec p') \geq \eps^2 \alpha^2/2$, we have that
\(
\mathcal{L}_N(\bar{\vec p}) 
-
\mathcal{L}_N(\vec p) = \Omega(\eps^2 \alpha^2)\,.
\)
\end{proof}
\noindent This concludes the proof of \Cref{prop:unsupervised}.
\end{proof}


\subsection{The Proof of \Cref{theorem:intro-reduction}} 
\label{subsection:proof-of-reduction}
In this subsection, we prove \Cref{theorem:intro-reduction}.
Our goal is to simulate a statistical query oracle which takes as input a query function $q$ with domain $\mathcal{X} \times \mathcal{Z}$ and outputs an estimate of its expectation with respect to finely labeled examples $\E_{(x,z) \sim \D}[q(x,z)]$, using coarsely labeled examples. 
Recall that since we have sample access only to coarsely labeled examples $(x,S) \sim \D_{\pi}$, we cannot directly estimate this expectation. The key idea is to perform rejection sampling on each coarse sample $(x,S)$ with acceptance probability $q(x,j)$ for any fine label $j \in \mathcal{Z}$. 
Because of the rejection sampling process, this marginal distribution is not the marginal of $\D$ 
on the fine labels $\mathcal{Z}$, but the marginal of $\D$ on the fine labels, conditional on the accepted samples.
However, the task of estimating from this marginal distribution can be still 
reduced to the unsupervised problem (see \Cref{prop:unsupervised}) of the previous section. 
Consider an arbitrary query function $q: \mathcal{X} \times \mathcal{Z} \rightarrow [-1,1] $ and, without loss of generality, let $\mathcal{Z} = [k]$. Recall that $\D$ is the joint probability distribution on the finely labeled examples $(x,z)$. 
We have that
\begin{equation}
\label{eq:expectation}
\E_{(x,z) \sim \D}[q(x,z)] = \sum_{j = 1}^{k} \E_{(x,z) \sim \D}\Big [q(x,j) \vec 1\{z=j\}\Big ] = \sum_{j = 1}^{k} \E_{(x,z) \sim \D}\Big [q_j(x) \vec 1\{z=j\}\Big ] \,.
\end{equation}
Since we would like to estimate the expectation of the query $q(x,z)$ with tolerance $\tau,$ it suffices to estimate the expectation of each query $q_j(x) \vec 1\{z=j\}$ with tolerance $\tau/k$ for any $j \in [k].$ 
Hence, it suffices to estimate expectations of the form $\E_{(x,z) \sim \D}[f(x)\vec 1\{ z = j\}]$ for arbitrary functions $f : \mathcal{X} \rightarrow [0,1]$\footnote{Any function $f : \mathcal{X} \rightarrow [-1,1]$ can be decomposed into $f = f^+ - f^-$ with $f^+, f^- \geq 0$ and, by linearity of expectation, it suffices to work with functions $f$ with image in $[0,1]$.} and $j \in [k]$. 

Let $\D_x$ denote the marginal distribution of the examples $x \in \mathcal{X}$. The algorithm performs rejection sampling. Each coarsely labeled example $(x,S) \sim \D_{\pi}$ is accepted with probability $f(x)$, that does not depend on the coarse label $S$. Hence, the rejection sampling process induces a distribution $\D^f$ over finely labeled examples $(x,z) \in \mathcal{X} \times \mathcal{Z}$ with density
\[
\D^f(x,z) = \frac{f(x) }{\E_{x \sim \D_x}[f(x)]}  \D(x,z)
\,.
\]
We remark that, we do not have sample access to $\D^f$ because we do not have sample access to the distribution $\D$ of the fine examples;
we introduced the above notation for the purposes of the proof.  Similarly, to $\D_x$, we define $\D^f_x$ to be the marginal distribution 
of $x$ conditional on its acceptance, i.e.,
\begin{equation}
\label{eq:acc}
\D_x^f(x) = \frac{f(x)}{\E_{x \sim \D_x}[f(x)]}\D_x(x) \,.
\end{equation}
Let $\D_z$ denote the marginal distribution of the fine labels $[k]$ and 
let $\D_z(\cdot|x)$ be the marginal distribution conditional on the example $x$. We have that
\[
\E_{(x, z) \sim \D}\Big [f(x) \vec 1\{z=j\}\Big ] = 
\int_{\mathcal{X}} f(x) \D(x,j) dx =
\int_{\mathcal{X}} f(x) \D_x(x) \D_z(j|x) dx\,.
\]
The above expectation can be equivalently written, by multiplying and dividing by $\D_x^f$,
\[
\E_{(x, z) \sim \D}\Big [f(x) \vec 1\{z=j\}\Big ] = 
\int_{\mathcal{X}} \Big(\frac{f(x) \D_x(x)}{\D_{x}^f(x)} \Big) \Big(\D_{x}^f(x) \D_z(j|x) \Big) dx\,.
\]
The first term in the integral is equal to $\E_{x \sim \D_x}[f(x)]$, by substituting Equation~\eqref{eq:acc} and, hence, is constant. The second term corresponds to the probability of observing the fine label $j$, given an example $x$, that has been accepted from the rejection sampling process. Similarly, to the marginal $\D_z$, we define $\D^f_z$ to be the marginal distribution of the fine labels $z$ conditional on acceptance. Hence, we can write
\begin{equation}
\label{equation:product}
\E_{(x, z) \sim \D}\Big [f(x) \vec 1\{z=j\}\Big ] = 
\E_{x \sim \D_x}[f(x)] \cdot
\Pr_{z \sim \D_{z}^f}[z = j] \,. 
\end{equation}
The decomposition of the expectation of Equation~\eqref{equation:product} is a key step:
we now only need to learn the marginal distribution of fine labels conditional on acceptance $\D_z^f$.


Recall that our goal is to estimate the left hand side expectation of Equation~\eqref{equation:product} with tolerance $\tau/k$. We claim that it suffices to estimate each term of the right hand side product of Equation~\eqref{equation:product} with tolerance $\tau/(2k)$. 
This is implied from the following: consider an estimate $\wt{\mu}$ of the value $\E_{x \sim \D_x}[f(x)]$ and an estimate $\wt{p}$ of the value $\Pr_{z \sim \D_{z}^f}[z = j]$. Then, using Equation~\eqref{equation:product}, we have that
\begin{equation*}
\Big | \wt{\mu}\cdot \wt{p} - \E_{(x, z) \sim \D}[f(x) \vec 1\{z=j\}] \Big | = 
\Big | \wt{\mu} \cdot \wt{p} - \E_{x \sim \D_x}[f(x)] \cdot \Pr_{z \sim \D_{z}^f }[z = j] \Big |\,,
\end{equation*}
and, hence, by adding and subtracting the term $\wt{\mu}~\Pr_{z \sim \D_{z}^f}[z = j]$, using the triangle inequality and, since both $\E_{x \sim \D_x}[f(x)]$ and $\Pr_{z \sim \D_{z}^f }[z = j]$ are at most $1$, we get that
\[
\Big | \wt{\mu} \cdot \wt{p} - \E_{(x, z) \sim \D}[f(x) \vec 1\{z=j\}] \Big | \leq
\Big |\wt{\mu} - \E_{x \sim \D_x}[f(x)] \Big | + \Big | \wt{p} - \Pr_{z \sim \D_{z}^f }[z = j] \Big |\,.
\]
We will show that $O(k^4/(\tau^3 \alpha^2)\log(1/\delta))$ samples are sufficient to bound each term of the right hand side by $\tau/(2k)$, with high probability. In order to estimate the expectation $\E_{(x, z) \sim \D}[q(x,z)]$, the algorithm applies (in parallel) the above process $k$ times with $f = q_j$ for any $j \in [k]$ (using Equation~\eqref{eq:expectation}) using a single training set of size $N = O(k^4/(\tau^3 \alpha^2)\log(1/\delta))$ drawn from the distribution $\D_{\pi}$ of coarsely labeled examples. Moreover, the running time is polynomial in the number of samples $N$. To conclude the proof, it suffices to show the following claims.
\begin{claim}
\label{appendix:claim:unsupervised}
There exists an algorithm that, uses $N = \wt{O}(k^4/(\tau^3 \alpha^2)\log(1/\delta))$ samples from $\D_{\pi}$ and computes an estimate $\wt{p}$, that satisfies
$
\Big | \wt{p} - \Pr_{z \sim \D_{z}^f}[z = j] \Big| \leq \tau/(2k)\,,
$
with probability at least $1-\delta$.
\end{claim}
\begin{proof}Recall that the distribution $\D_{z}^f$ is the marginal distribution of the fine labels $z \in \mathcal{Z} = [k]$, conditional that the example $x \sim \D_x^f$, i.e., that the example $x \in \mathcal{X}$ has been accepted by the rejection sampling process. Hence, the distribution $\D_{z}^f$ is supported on $\mathcal{Z}$. We can then directly apply \Cref{prop:unsupervised}, using as training set the set of \emph{accepted} coarsely labeled samples $(x,S)$ and can compute an estimate $\wt{\D}$, that is $\eps$-close in total variation distance to $\D_{z}^f$. By setting $\eps = \tau/(2k)$, the algorithm uses $\wt{O}(k^3/(\tau^2 \alpha^2)\log(1/\delta))$ samples from the set of accepted samples and outputs the estimate $\wt{p} = \wt{\D}(j).$ For the example $x \in \mathcal{X}$, the acceptance probability $f(x)$ can be considered $\Omega(\tau/k)$. Otherwise, we can set the desired expectation equal to $0$. Hence, the algorithm needs to draw in total $\wt{O}(k^4/(\tau^3 \alpha^2)\log(1/\delta))$ samples from $\D_{\pi}$ in order to compute an estimate $\wt{p}$ that satisfies
\[
\Big| \wt{p} - \Pr_{z \sim \D_{z}^f}[z = j] \Big| \leq \tau/(2k)\,,
\]
with probability at least $1-\delta$.
\end{proof}
\begin{claim}
\label{appendix:claim:hoeffding}
There exists an algorithm that, uses $N = O((k^2/\tau^2)\log(1/\delta))$ samples from $\D_{\pi}$ and computes an estimate $\wt{\mu}$, that satisfies
$
\Big | \wt{\mu} - \E_{x \sim \D_x}[f(x)] \Big| \leq \tau/(2k)\,,
$
with probability at least $1-\delta$.
\end{claim}
\begin{proof}
The algorithms draws $N$ coarsely labeled examples from $\D_{\pi}$ and computes the estimate $\wt{\mu} = \frac{1}{N}\sum_{i=1}^N f(x_i)$. From the Hoeffding bound, since the estimate is a sum of independent bounded random variables, we get
\[
\Pr\left [~ \Big|\wt{\mu} - \E_{x \sim \D_x}[f(x)] \Big| \geq \tau/(2k)\right ] \leq 2\exp(- N \tau^2/(2k^2))\,.
\]
Using $N = O((k^2/\tau^2)\log(1/\delta))$ samples, the algorithm estimates the desired expectation with error $\tau/(2k),$ with probability at least $1-\delta$. Note that, if $\wt{\mu} < \tau/(2k)$, the algorithm can output $0$, since the estimated value will lie in the desired tolerance interval.
\end{proof}

\section{Learning Gaussians from Coarse Data}
\label{section:implicit}
In this section, we focus on an unsupervised learning problem with coarse data. 
Recall that we have already solved such a problem in the discrete setting as an ingredient of our supervised learning result, 
see \Cref{subsection:marginal}.  In this section, we study the fundamental problem of learning a Gaussian
distribution given coarse data.  
In  \Cref{subsection:gaussian-non-convex}, we show that, under general partitions, 
this problem is \NP-hard.
In \Cref{subsection:gaussian-convex}, we show that we can efficiently estimate the Gaussian mean 
under convex partitions of the space. 

\subsection{Computational Hardness under General Partitions }
\label{subsection:gaussian-non-convex}
In this section, we consider general partitions of the $\dim$-dimensional Euclidean space, that may contain non-convex subsets. For instance, a compact convex body and its complement define a non-convex partition of $\mathbb{R}^{d}$. 
In order to get this computational hardness result, we reduce from \maxcut and make use of its hardness of approximation (see \cite{haastad2001some}).  Recall that \maxcut can be viewed as a maximization problem, where the objective function corresponds to a particular quadratic function (associated with the Laplacian matrix of the given graph instance) and the constraints restrict 
the solution to lie in the Boolean hypercube (the constraints can be seen geometrically as the intersection of bands, see 
\Cref{fig:maxcut}). 


We first define \maxcut and a variant of \maxcut where the optimal cut score is given as part of the input. Let $G = (V,E)$ be a graph\footnote{We are going to work with graphs with unit weights.} with $d$ vertices.  
A \emph{cut} is a partition of $V$ into two subsets $S$ and $S' = V \setminus S$ and the value of the cut $(S,S')$ is $c(S,S') = \sum_{u,v \in E} \vec 1\{u \in S, v \in S'\}$. The goal of the problem is find the maximum value cut in $G$, i.e., to partition the vertices into two sets so that the number of edges crossing the cut is maximized.
We can define \maxcut as the following maximization problem for the graph $G = (V,E)$ with $|V| = d$:
\[
\max \sum_{(i,j) \in E} (\vec x_i - \vec x_j)^2\,,~\text{ subj. to }~\vec x_i \in \{-1, +1\}~\forall i \in [\dim]\,.
\]
The objective function is the quadratic form $\vec x^T \vec L_G \vec x$, where $\vec L_G$ is the Laplacian matrix of the graph $G$. 
We may also assume that the value of the optimal cut is known and is equal to $\opt$.\footnote{Observe that this problem is still hard, since the maximum value of a cut is bounded by $d^2$ and, hence, if this problem 
could be solved efficiently, one would be able to solve \maxcut by trying all possible values of $\opt$.}
Before proceeding with the overview of the proof, we state 
a key result of~\cite{haastad2001some} about the inapproximability of \maxcut.
\begin{lemma}
[Inapproximability of Maximum Cut Problem~\cite{haastad2001some}]
It is \NP-hard to approximate \maxcut to any factor higher than $16/17$.
\end{lemma}
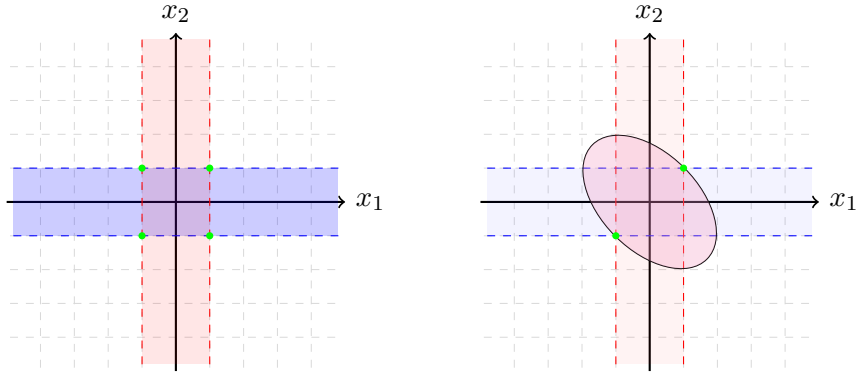
\begin{figure}[ht]
\centering
    \begin{subfigure}{0.4\textwidth}
         \centering
    \begin{tikzpicture}[scale=0.45]
    \draw[help lines, color=gray!30, dashed] (-4.9,-4.9) grid (4.9,4.9);
    \draw[->,thick] (-5,0)--(5,0) node[right]{$x_1$};
    \draw[->,thick] (0,-5)--(0,5) node[above]{$x_2$};
    
    \draw[dashed, blue, name path=A1] (-4.8,-1) coordinate (a_11) --(4.8,-1) coordinate (a_12);
    \draw[dashed, blue, name path=A2] (-4.8,1) coordinate (a_21 )--(4.8,1) coordinate (a_22) ;
    
    \tikzfillbetween[of=A1 and A2]{blue, opacity=0.2};
    
    \draw[dashed, red,  name path=B1] (-1,-4.8) coordinate (b_11) --(-1,4.8) coordinate (b_12) ;
    \draw[dashed, red, name path=B2] (1,-4.8) coordinate (b_21)--(1,4.8) coordinate (b_22) ;
    \tikzfillbetween[of=B1 and B2]{red, opacity=0.1};
    
    \coordinate (c1) at (intersection of a_11--a_12 and b_11--b_12);
    \fill[green] (c1) circle (3pt);
    \fill[green] (1,1) circle (3pt);
    \fill[green] (-1,1) circle (3pt);
    \fill[green] (1,-1) circle (3pt);
    
    \end{tikzpicture}
         \centering
    \end{subfigure}%
    \begin{subfigure}{0.4\textwidth}
         \centering
    \begin{tikzpicture}[scale=0.45]
    \draw[help lines, color=gray!30, dashed] (-4.9,-4.9) grid (4.9,4.9);
    \draw[->,thick] (-5,0)--(5,0) node[right]{$x_1$};
    \draw[->,thick] (0,-5)--(0,5) node[above]{$x_2$};
    
    \draw[dashed, blue, name path=A1] (-4.8,-1) coordinate (a_11) --(4.8,-1) coordinate (a_12);
    \draw[dashed, blue, name path=A2] (-4.8,1) coordinate (a_21 )--(4.8,1) coordinate (a_22) ;
    
    \tikzfillbetween[of=A1 and A2]{blue, opacity=0.05};
    
    \draw[dashed, red,  name path=B1] (-1,-4.8) coordinate (b_11) --(-1,4.8) coordinate (b_12) ;
    \draw[dashed, red, name path=B2] (1,-4.8) coordinate (b_21)--(1,4.8) coordinate (b_22) ;
    \tikzfillbetween[of=B1 and B2]{red, opacity=0.05};
    
    \draw[scale =3.4, rotate=45] (0,0)ellipse (12pt and 20pt);
    \fill[magenta!50, fill opacity=0.3][scale =3.4, rotate=45]  (0,0) ellipse (12pt and 20pt);  
    \coordinate (c1) at (intersection of a_11--a_12 and b_11--b_12);
    \fill[green] (c1) circle (3pt);
    
    \fill[green] (1,1) circle (3pt);

    \end{tikzpicture}
    \end{subfigure}
    \caption{The geometry of the \maxcut instance. The left figure corresponds to the fat hyperplanes, i.e., the constraints of \maxcut and the right figure (the ellipsoid) corresponds to the objective function of \maxcut. The green points lie in the Boolean hypercube. }
    \label{fig:maxcut}
\end{figure}



\subsubsection{Sketch of the Proof of \Cref{theorem:impossibility-gaussian}}
The first step of the proof is to construct the distribution over partitions of $\reals^d$. 
The \maxcut problem can be viewed as a collection of $d+1$ non-convex partitions of the $\dim$-dimensional Euclidean space. Consider an 
instance of \maxcut with $|V| = d$ and optimal cut value $\opt$.  Consider the collection of $d+1$ partitions 
$\mathcal{B} = \{ \S_1, \dots, \S_d,  \mathcal{T} \}$. 
We define the partitions as follows: for any $i = 1,\dots, \dim$, we let $S_i = \{\vec x:  -1 \leq \vec x_i \leq 1\}$ be the 
sets that correspond to fat hyperplanes of \Cref{fig:maxcut}(a) and the partitions $\S_i = \{ S_i, S_i^c \}$, i.e., pairs 
of fat hyperplanes and their complements (see \Cref{fig:partition}(a,b)). These $d$ partitions will simulate the \maxcut  constraints, i.e., that the solution vector lies in the hypercube $\{-1,1\}^d$.  It remains to construct $\mathcal{T}$, which intuitively corresponds
to the quadratic objective of $\maxcut.$

\begin{figure}[ht]
    \centering
    \includegraphics[scale=0.43]{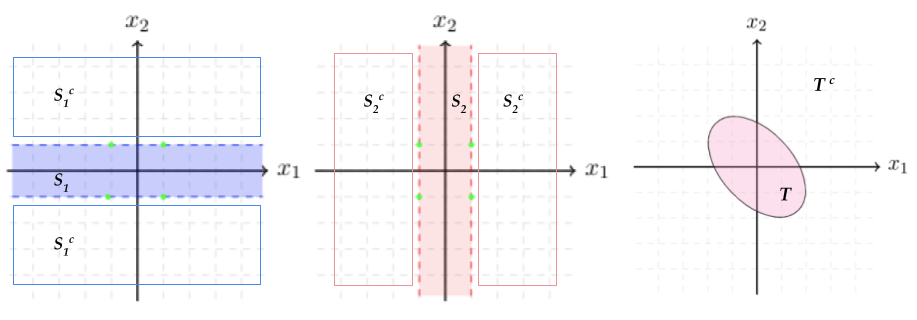}
    \caption{The mixture of partitions that corresponds to the \maxcut problem. In figures $(a)$ and $(b)$, we partition the Euclidean space using fat hyperplanes (the blue set $S_1$ and the red set $S_2$ respectively) and their complements $S_1^c = \reals^d \setminus S_1$ and $S_2^c = \reals^d \setminus S_2$. The third figure $(c)$ partitions $\reals^d$ using the ellipsoid $T = \{ \vec x : \vec x^T \vSigma^{-1} \vec x \leq q\}$ and its complement $T^c = \reals^d \setminus T$ (for some $d \times d$ covariance matrix $\vSigma$ and positive real $q$).}
    \label{fig:partition}
\end{figure}

Fix the covariance matrix $\vSigma = \vec L_G^{-1} \opt$ \footnote{In fact, $\vec L_G$ has zero eigenvalue with eigenvector $(1,\ldots,1)$: we 
have to project the Laplacian to the subspace orthogonal to $(1,\ldots, 1)$  to avoid this.  We ignore this technicality here for simplicity.} 
, i.e., $\vSigma$ is the inverse of the Laplacian normalized by $\opt$.  
We let $T = \{ \vec x : \vec x^T \vSigma^{-1} \vec x \leq q\}$
 for some positive value 
$q$ to be defined later (see \Cref{fig:maxcut}(b) and \Cref{fig:partition}(c)).  Then, we let $\mathcal{T} = \{T, T^c\}$.  We construct a mixture $\pi$ of these partitions by picking each one uniformly at random, i.e., with probability $1/(d+1)$.

Let us assume that there exists an algorithm that, given access to 
    samples from $\N_\pi(\vmu^\star, \vSigma)$, with \emph{known covariance} $\vSigma$, computes, 
    in time $\poly(d)$, a mean vector $\vmu$ 
so that the output distributions are matched, i.e.,  $\dtv(\N_{\pi}(\vmu,\vSigma), \N_{\pi}(\vmu^\star,\vSigma))$ is upper bounded 
by $1/d^c$ for some 
absolute constant $c > 1$.  Equivalently this means that the mass that $\N(\vmu, \vSigma)$
assigns to each set $S_i$ and $T$ is within $\poly(1/d)$ of the corresponding mass
that $\N_\pi(\vmu^\star, \vSigma)$ assigns to the same set.  There are two main challenges in order to prove the reduction:
\begin{enumerate}
\item  How can we generate coarse samples from $\N_\pi(\vmu^\star, \vSigma)$ since $\vmu^\star$ is the solution
of the \maxcut problem and therefore is unknown? 
\item  Given $\opt$, is it possible to pick the threshold $q$ 
of the ellipsoid $T = \{ \vec x \in \reals^{\dim} : \vec x^T \vSigma^{-1} \vec x \leq q\}$ 
so that any vector $\vmu$ (rounded to belong in $\{-1, 1\}^d$), that achieves $\N(\vmu, \vSigma; T) \approx \N(\vmu^\star, \vSigma;T)$ 
and $\N(\vmu, \vSigma; S_i) \approx \N(\vmu^\star, \vSigma;S_i)$, also achieves an approximation ratio better than $16/17$ 
for the \maxcut objective ?
\end{enumerate}

The key observation to answer the first question is that, by the rotation invariance of the Gaussian distribution, 
the probability $\N(\vmu^\star, \vSigma;T) =  \Pr_{\vec x \sim \N(\vmu^\star, \vSigma)}\big[\vec x^T \vSigma^{-1} \vec x \leq q \big]$
is a constant $p$ that only depends on the value $\opt$ of the \maxcut problem.  Therefore, having this value $p$,
we can flip a coin with this probability and give the coarse sample $T$ if we get heads and $T^c$ otherwise.
Similarly, the value of $\N(\vmu^\star, \vSigma;S_i)$ is an absolute constant that does not depend on $\vmu^\star \in \{-1, 1\}^d$
and therefore we can again simulate coarse samples by flipping a coin with probability equal to $\N(\vmu^\star, \vSigma;S_i)$.

To resolve the second question, we first show that any vector $\vmu$ that approximately 
matches the probabilities of the $d$ fat halfspaces, lies very close to a corner of the
hypercube, see \Cref{lemma:band-sensitivity}. 
Therefore, by rounding this guess $\vmu$, we obtain exactly a corner of the
hypercube without affecting the probability assigned to the ellipsoid constraint by a lot.
We then show that any vector of the hypercube that almost matches the probability 
of the ellipsoid achieves large cut value.
In particular, we  prove that there exists a value for the threshold $q$ of
the ellipsoid $ \vec x^T \vSigma^{-1} \vec x \leq q$ that makes the probability 
$\N(\vmu, \vSigma;T)$ \emph{very sensitive to changes of $\vmu$}.
Therefore, the only way for the algorithm to match the observed probability is to find
a $\vmu$ that achieves large cut value.  We show the following lemma.
\begin{lemma}[Sensitivity of Gaussian Probability of Ellipsoids]
\label{lemma:ellipsoid-sensitivity}
Let $\N(\vmu^\star, \vSigma)$, $\N(\vmu, \vSigma)$ be $\dim$-dimensional Gaussian distributions. 
Let $\vec v^\star = \vSigma^{-1/2} \vmu^\star$, $\vec v = \vSigma^{-1/2} \vmu$ and
assume that $\|\vec v \|_2 \leq \|\vec v^\star \|_2 = 1$.
Denote $q = d + \| \vec v^\star \|^2_2 + \sqrt{2 d + 4 \|\vec v^\star\|_2^2}$.
Then, assuming $d$ is larger than some sufficiently large absolute constant, it holds that
\[
\Big|
\Pr_{\vec x \sim \N(\vmu^\star, \vSigma)}\big[\vec x^T \vSigma^{-1} \vec x \leq q \big] 
-
\Pr_{\vec x \sim \N(\vmu, \vSigma)}\big[\vec x^T \vSigma^{-1} \vec x \leq q \big]  \Big| 
\geq \frac{\|\vec v^\star\|_2^2 - \|\vec v\|_2^2}{6 \sqrt{2 d + 4} } - o(1/\sqrt{d}) \,.
\]
\end{lemma}
Notice that with $\vSigma = \vec L_G^{-1} \opt$, in the above lemma, we have 
$\|\vec v^\star\|_2^2 = 1$, since $\vmu^\star$ achieves cut value $\opt$.
By assumption, we know that the learning algorithm can find a guess $\vmu$ that makes the
left hand side of the inequality of \Cref{lemma:ellipsoid-sensitivity} smaller than 
$\poly(1/d)$. Thus, we obtain that, for $d$ large enough, it must be that $\|\vec v\|_2^2 
= \vmu^T \vec L_G \vmu / \opt \geq 16/17$. Therefore, $\vmu$ achieves value greater 
than $(16/17) \opt$.

\begin{remark}
The transformation $\pi$ used in the above hardness result is not information preserving.
In \Cref{theorem:impossibility-gaussian}, we prove that it is computationally hard to find a vector $\vec \mu \in \reals^d$ that matches in total variation 
the \emph{observed distribution over coarse labels}.  In contrast, as we will see in the upcoming \Cref{subsection:gaussian-convex}, when the sets of the partitions are convex,
we show that there is an efficient algorithm that can solve the same problem and compute some $\vec \mu \in \reals^d$ such that 
$\mathrm{TV}(\mathcal{N}_\pi(\vec \mu^\star), \mathcal{N}_\pi(\vec  \mu))$ is small
regardless of whether the transformation $\pi$ is information preserving.   When the
transformation is information preserving, we can further show that the vector $\vec \mu$ that we compute
will be close to $\vec \mu^\star$.

\end{remark}

\subsubsection{Sensitivity of Gaussian Probabilities} 
We now prove \Cref{lemma:ellipsoid-sensitivity}, namely
that the probability of an ellipsoid with respect to the Gaussian
distribution is sensitive to small changes of its mean.
\begin{proof}[\textit{Proof of}~\Cref{lemma:ellipsoid-sensitivity}]
We first observe that 
\begin{align*}
\Pr_{\vec x \sim \N(\vmu, \vSigma)} \big[\vec x^T \vSigma^{-1} \vec x \leq q \big] 
&=
\Pr_{\vec x \sim \N(\vec 0, \vec I)} \big[\vec x^T \vec x + 2 \vmu^T \vSigma^{-1/2}\vec x \leq q - \vmu^T \vSigma^{-1} \vmu\big]
\\
&=
\Pr_{\vec x \sim \N(\vec 0, \vec I)} \big[\vec x^T \vec x + 2 \vec v^T \vec x \leq q - \|\vec v\|_2^2\big],
\end{align*}
where $\vec v = \vSigma^{-1/2} \vmu$.  Similarly, we have
$
\Pr_{\vec x \sim \N(\vmu^{\star}, \vSigma)} \big[\vec x^T \vSigma^{-1} \vec x \leq q \big] 
= 
\Pr_{\vec x \sim \N(\vec 0, \vec I)} \big[\vec x^T \vec x + 2 (\vec v^{\star})^T \vec x \leq q - \|\vec v^{\star}\|_2^2\big],
$
where $\vec v^{\star} = \vSigma^{-1/2} \vmu^{\star}$. From the rotation invariance of the Gaussian distribution, we may assume, 
without loss of generality, that $\vec v = \|\vec v\| \vec e_1 $ and 
$\vec v^{\star} = \|\vec v^{\star}\| \vec e_1$.  
Notice that $(\|\vec v\|_2+\vec x_1)^2 + \sum_{i=2}^d \vec x_i^2$ is a sum of independent random variables. 
To estimate these probabilities we are going to use the central limit theorem. 
\begin{lemma}[CLT, Theorem 1, Chapter XVI in \cite{feller1957introduction} ]
Let $X_1, \ldots, X_n$ be independent random variables with $\E[|X_i|^3] < +\infty$ for all $i$.
Let $m_1 = \E[\sum_{i=1}^n X_i]$ and $m_j =  \sum_{i=1}^n \E[ (X_i  - \E[X_i])^j]$.
Then, 
\[
\Pr\left[ \frac{(\sum_{i=1}^n X_i) - m_1}{\sqrt{m_2}} \leq x\right]  - \Phi(x) 
= m_3 \frac{(1-x^2) \phi(x)}{6 m_2^{3/2}} + o \left(n/m_2^{3/2}\right) \,,
\]
where $\Phi(\cdot)$, resp., $\phi(\cdot)$ is the CDF resp., PDF of the standard normal distribution
and the convergence is uniform for all $x \in \reals$.
\end{lemma}
\noindent Using the above central limit theorem we obtain
\[
\Pr_{\vec x \sim \N(\vec 0, \vec I)} \left[ ( \|\vec v^\star\|_2 + \vec x_1)^2 + \sum_{i=2}^d \vec x_i^2 \leq q \right]
= \Phi(\bar q_1) + O\left(\frac{1}{\sqrt{d}} \right) (1-\bar{q_1}^2) \phi(\bar q_1) + o\left(1/\sqrt{d}\right) 
\,,
\]
where $\bar q_1 = \frac{q - (d + \|\vec v^\star\|^2_2)}{ \sqrt{2 d + 4 \|\vec v^\star\|^2_2}}$.
Since $q = d + \|\vec v^\star\|^2 + \sqrt{2 d + 4 \| \vec v^\star\|^2_2} $ we obtain $\bar{q}_1 = 1$
and therefore
\[
\Pr_{\vec x \sim \N(\vec 0, \vec I)} \left[\vec x^T \vec x + 2 (\vec v^\star)^T \vec x \leq q - \|\vec v^\star\|_2^2\right]
= \Phi(1) +o \left(1/\sqrt{d} \right) \,.
\]
Similarly, from the central limit theorem, we obtain
\[
\Pr_{\vec x \sim \N(\vec 0, \vec I)} \left[ \left( \|\vec v\|_2 + \vec x_1 \right)^2 + \sum_{i=2}^d \vec x_i^2 \leq q \right]
= \Phi(\bar q_2) + O\left(\frac{1}{\sqrt{d}} \right) (1-\bar{q_2}^2) \phi(\bar q_2) + o \left(1/\sqrt{d} \right)\,,
\]
where $\bar{q}_2 =  \frac{q - (d + \|\vec v\|^2_2)}{ \sqrt{2 d + 4 \|\vec v\|^2_2}}  =  1 + O(1/\sqrt{d})$.  Therefore, we have
\[
\Pr_{\vec x \sim \N(\vec 0, \vec I)} \left[\vec x^T \vec x + 2 \vec v^T \vec x \leq q - \|\vec v\|_2^2\right]
= \Phi(\bar q_2) + o \left(1/\sqrt{d} \right) \,.
\]
Moreover, we have that $\bar q_2 \geq 1 + (\| \vec v^\star \|_2^2 - \| \vec v \|_2^2)/(\sqrt{2 d + 4 \| \vec v \|^2_2})$.
Using the fact that $d$ is sufficiently large and standard approximation results on the Gaussian CDF, we obtain 
\[
\Phi \left (1 +  \frac{\| \vec v^\star \|_2^2 - \| \vec v \|_2^2}{\sqrt{2 d + 4 \| \vec v \|_2^2}} \right) - \Phi(1) \geq 
(\| \vec v^\star \|_2^2 - \| \vec v \|_2^2)/ \left(6 \sqrt{2 d + 4 \| \vec v \|_2^2} \right)\,,
\]
and, since $\| \vec v \|_2 \leq 1$, we conclude that
the left-hand side satisfies
\[
\Phi \left (1 +  \frac{\| \vec v^\star \|_2^2 - \| \vec v \|_2^2}{\sqrt{2 d + 4 \| \vec v \|_2^2}} \right) - \Phi(1) \geq 
(\| \vec v^\star \|_2^2 - \| \vec v \|_2^2)/ \left (6\sqrt{2 d + 4} \right)\,.
\]
The result follows.
\end{proof}

We will also require the following sensitivity lemma about the Gaussian
probability of bands, i.e., sets of the form $\{\vec x: |\vec x_i| \leq 1\}$.  
We show that the probabilities of such regions are also sensitive under perturbations
of the mean of the Gaussian.  This means that any vector $\vec \mu$ that has
$
\Pr_{\vec x \sim \N(\vmu, \vSigma)} \big[-1 \leq \vec x_i \leq 1 \big] 
$
close to 
$
\Pr_{\vec x \sim \N(\vmu^\star, \vSigma)} \big[-1 \leq \vec x_i  \leq 1 \big] 
$
must be very close to a corner of the hypercube.
\begin{lemma}
[Sensitivity of Gaussian Probability of Bands]
\label{lemma:band-sensitivity}
Let $\N(\vmu^\star, \vSigma), \N(\vmu, \vSigma)$ be two $\dim$-dimensional Gaussian distributions with 
$\vec e_i^T \vSigma \vec e_i \leq Q$, and $|\vmu_i^\star| = 1$ 
for all $i \in [d]$.  
Then, for any $i \in [d]$, it holds that
\[
\left|
\Pr_{\vec x \sim \N(\vmu^\star, \vSigma)} \big[-1 \leq \vec x_i  \leq 1 \big] 
- 
\Pr_{\vec x \sim \N(\vmu, \vSigma)} \big[-1 \leq \vec x_i \leq 1 \big] 
\right|
\geq c \cdot \frac{\min(1, (1 - | \vmu_i|)^2)}{Q^4} \,,
\]
for some absolute constant $c \in (0,1]$.
\end{lemma}
\begin{proof}
Let us fix $i \in [d]$, define $\mu^\star$ (resp. $\mu$) for $\vmu^\star_i$ (resp. $\vmu_i$), and 
$\sigma^2 = \vSigma_{ii}$.
Without loss of generality since both Gaussians have the same variance $\sigma$ 
by symmetry we may assume that $\mu^\star = 1$ and $\mu \in [0,+\infty)$.
We first deal with the case $\mu > 1$. We have 
\begin{align*}
\Pr_{\vec x \sim \N(\vmu^\star, \vSigma)} \big[-1 \leq \vec x_i  \leq 1 \big] 
&- 
\Pr_{\vec x \sim \N(\vmu, \vSigma)} \big[-1 \leq \vec x_i \leq 1 \big] 
\\
&= \E_{t \sim \N(1,\sigma^2)} \left[ \vec 1\{|t| \leq 1\} \left(1 - \frac{\N(\mu, \sigma^2; t)}{\N(1, \sigma^2;t)}\right) \right]\,.
\end{align*}
We have that since $\mu > 1$ the ratio $\frac{\N(\mu, \sigma^2; t)}{\N(1, \sigma^2;t)} = e^{\frac{(\mu -1) (-\mu +2 t-1)}{2 \sigma ^2}} $ 
is maximized for $t=1$ and has maximum value $e^{- \frac{(\mu-1)^2}{2 \sigma^2} }$.  
By taking the derivative with respect to $\sigma$ we observe that the probability that 
$N(1,\sigma)$ assigns to $[-1,1]$ is decreasing with respect to $\sigma$ and therefore
it is minimized for $\sigma = 1$. We have that $\Pr_{t \sim \N(1,\sigma)}[ -1 < t < 1] = \Omega(1/\sigma)$
and therefore 
$
\Pr_{\vec x \sim \N(\vmu^\star, \vSigma)} \big[-1 \leq \vec x_i  \leq 1 \big] 
- 
\Pr_{\vec x \sim \N(\vmu, \vSigma)} \big[-1 \leq \vec x_i \leq 1 \big] 
\geq C \cdot \left ( 1- e^{- \frac{(\mu-1)^2}{2 \sigma^2} } \right).
$
We can obtain the significantly weaker lower bound of $c \min(1, (1 - |\mu|)^2)$ for some absolute 
constant $c \in (0,1]$ by using the inequality
$1- e^{-x} \geq 1/2 \min(1, x)$ that holds for all $x \in [0,+\infty)$.

We now deal with the case $\mu \in [0,1)$.  In that case the expression of their ratio 
of the densities of  $\N(1,\sigma)$ and $\N(\mu, \sigma)$ derived above
shows us that they cross at $t = (1+\mu)/2$.   Therefore, they completely cancel out in
the interval $[\mu, 1]$.  We have 
$
\Pr_{\vec x \sim \N(\vmu, \vSigma)} [-1 \leq \vec x_i \leq 1 ] 
- 
\Pr_{\vec x \sim \N(\vmu^\star, \vSigma)} [-1 \leq \vec x_i  \leq 1] 
=
\Pr_{t \sim \N(\mu,\sigma)} [-1 \leq t \leq \mu ] 
- 
\Pr_{t \sim \N(1, \sigma)} [-1 \leq t  \leq \mu ] 
= \Omega((1-\mu)/ (1 + \sigma^4)) \,,
$
where to obtain the last inequality we use standard approximations of Gaussian integrals.
Combining the above two cases we obtain the claimed lower bound.

\end{proof}

\subsubsection{The Proof of~\Cref{theorem:impossibility-gaussian} }
We are now ready to provide the complete proof of \Cref{theorem:impossibility-gaussian}.
Consider an instance of \maxcut with $|V| = \dim$ and optimal value $\mathrm{opt} =
O(d^2)$. Let $\vec L_G$ be the 
Laplacian matrix of the (connected) graph $G$. Since the minimum 
eigenvalue of $\vec L_G$ is $0$, we project the matrix onto the subspace $V$ that is orthogonal to $\vec 1 = (1,\ldots,1)$.  We introduce a $(\dim-1) \times d$ partial isometry $\vec R$, that satisfies $\vec R \vec R^T = \vec I$ and $\vec R \vec 1 = \vec 0$, i.e., $\vec R$ projects vectors
to the subspace $V$.  We consider $\vec L_G' = \vec R \vec L_G \vec R^T$. 
It suffices to find a solution $\vec x \in V$  and then project back to $\reals^d$: $\vec y = \vec R^T \vec x$.  We note that the matrix $\vec L'_G$ is positive definite (the smallest eigenvalue of $\vec L_G'$ is equal to the second smallest eigenvalue of $\vec L_G$) and preserves the optimal score value, in the sense that
\[
\mathrm{opt} = \max_{\vec y \in \reals^d} \vec y^T \vec L_G \vec y = \max_{\vec x \in \reals^d} (\vec R^T \vec x)^T \vec L_G (\vec R^T \vec x) = \max_{ \vec x \in V} \vec x^T \vec L_G' \vec x \,.
\]
Assume that there exists an efficient black-box algorithm 
$\mathcal{A}$, that, given sample access to a generative process of coarse Gaussian data $\N_{\pi}(\vmu^\star, \vSigma)$ with known covariance
\footnote{We remark that our hardness result is stated for identity covariance matrix (and not for an arbitrary known covariance matrix). In order to handle this case, we provide a detailed discussion after the end of the proof of  \Cref{theorem:impossibility-gaussian}.}
matrix $\vSigma$, computes an estimate $\wt{\vmu}$ in $\poly(d)$ time, that satisfies
\[
\dtv(\N_{\pi}(\wt{\vmu}, \vSigma), \N_{\pi}(\vmu^\star, \vSigma)) < 1/d^c\,. 
\]
We choose the known covariance matrix to be equal to $\vSigma = (\vec L_G')^{-1} \mathrm{opt}$, where $\mathrm{opt}$ is the given optimal \maxcut value and let $\vmu^\star \in \{ -1,1\}^{\dim-1}$ be the unknown mean vector. Recall that, not only the black-box algorithm $\mathcal{A}$, but also the generative process that we design is agnostic to the true mean. However, as we will see the knowledge of the optimal value $\mathrm{opt}$ and the fact that the true mean lies in the hypercube $\{-1,1\}^{d-1}$ suffice to generate samples from the true coarse generative process $\N_{\pi}(\vmu^\star, \vSigma)$.

In what follows, we will construct such a coarse generative process using the objective function and the constraints of the \maxcut problem. Specifically, we will design a collection $\mathcal{B} = \{ \S_1, \dots, \S_{d-1}, \mathcal{T} \}$ of $\dim$ partitions of the $\dim$-dimensional Euclidean space and let the partition distribution $\pi$ be the uniform probability measure over $\mathcal{B}$.  

We define the partitions as follows: for any $i = 1,\dots, \dim-1$, let $S_i = \{\vec x:  -1 \leq \vec x_i \leq 1\}$ and $\S_i = \{ S_i, S_i^c \}$. 
These $d-1$ partitions simulate the integrality constraints of \maxcut, i.e., the solution vector should lie in the hypercube $\{-1,1\}^{d-1}$. It remains to construct $\mathcal{T}$, which corresponds to the quadratic objective of $\maxcut.$ We let $T = \{ \vec x \in \reals^{\dim} : \vec x^T \vSigma^{-1} \vec x \leq q \}$, for $q > 0$ to be decided. Then, we let $\mathcal{T} = \{T, T^c\}$. Recall that the known covariance matrix $\vSigma = (\vec L_G')^{-1} \mathrm{opt} $ lies in $\reals^{(\dim - 1) \times (\dim -1)}$ and, so, we will use $d-1$ bands (i.e., fat hyperplanes).



The main question to resolve is how to generate 
efficiently samples from the designed general partition, i.e., the distribution $\N_{\pi}(\vmu^\star, \vSigma)$, \emph{without} knowing the value of $\vmu^\star$. The key observation is that, by the rotation invariance of the Gaussian distribution, 
the probability $\N(\vmu^\star, \vSigma;T) =  \Pr_{\vec x \sim \N(\vmu^\star, \vSigma)}\big[\vec x^T \vSigma^{-1} \vec x \leq q \big]$
is a constant $p$ that only depends on the value $\opt$ of the maximum cut (see the proof of \Cref{lemma:ellipsoid-sensitivity}). Therefore, having this value $p$, we can flip a coin with this probability and give the coarse sample $T$ if we get heads and $T^c$ otherwise. At the same time,
the value of $\N(\vmu^\star, \vSigma;S_i)$ is an absolute constant that does not depend on $\vmu^\star \in \{-1, 1\}^{d-1}$
and, therefore, we can again simulate coarse samples by flipping a coin with probability equal to $\N(\vmu^\star, \vSigma;S_i)$.
More precisely, since $S_i$ is a symmetric interval around $0$, we have that
\[
\Pr_{\vec x \sim \N(\vmu^\star, \vSigma)}\big[-1 \leq \vec x_i \leq 1]
= 
\Pr_{t \sim \N(1, \vSigma_{ii})}\big[-1 \leq t \leq 1] \,.
\]
Notice that the above constant only depends on the \emph{known} constant $\vSigma_{ii}$ and 
can be computed to very high accuracy using well known approximations of the Gaussian integral 
or rejection sampling.  Moreover, all the probabilities $\N(\vmu^\star, \vSigma;S_i), \N(\vmu^\star, \vSigma;T)$ 
are at least polynomially small in $1/d$.
In particular, $\N(\vmu^\star, \vSigma;S_i)$, is always larger than $\Omega(1/\sigma) \geq \poly(1/d)$
and smaller than $1/2$
 and  $\N(\vmu^\star, \vSigma;T) =  \Phi(1) + o(1/\sqrt{d})$ \footnote{$\Phi(\cdot)$ is the CDF of the standard Normal distribution.}, see the proof of \Cref{lemma:ellipsoid-sensitivity}.  
Having these values we can generate samples from $\N_\pi$ as follows:
\begin{enumerate}
\item  Pick one of the $d$ sets $S_1,\ldots, S_{d-1}, T$ uniformly at random.
\item  Flip a coin with success probability equal to the probability of the corresponding sets and
return either the set or its complement.
\end{enumerate}

Giving sample access to the designed oracle with
$\mathcal{B} = \{\S_1, \dots, \S_{d-1}, \mathcal{T} \}$, 
the black-box algorithm $\mathcal{A}$ computes efficiently and returns an estimate $\wt{\vmu} \in \reals^{d-1}$, that satisfies
\[
\dtv(\N_{\pi}(\wt{\vmu}, \vSigma), \N_{\pi}(\vmu^\star, \vSigma)) < o(1/d^c)\,.
\] 
We proceed with two claims: $(i)$ the algorithm's output $\wt{\vmu}$ should lie in a ball of radius $\poly(1/d)$, centered at one of the vertices of the hypercube $\{-1,1\}^{\dim-1}$ and $(ii)$ it will hold that the rounded vector $\wh{\vmu} = (\mathrm{sgn}(\wt{\vmu}_i))_{1 \leq i \leq \dim-1} \in \{-1,1\}^{\dim - 1}$ will attain a cut score, that approximates the \maxcut within a factor larger than $16/17$. By the algorithm's guarantee, since $\pi$ is the uniform distribution, we get that
\[
|\N(\wt{\vmu}, \vSigma; T) - \N(\vmu^\star, \vSigma; T)| +  \sum_{i=1}^{d-1}|\N(\wt{\vmu}, \vSigma; S_i) - \N(\vmu^\star, \vSigma; S_i)| = o(1/d^{c-1}) \,.
\]
Hence, we get that each of the above $d$ summands is at most $o(1/d^{c-1})$.
\begin{claim}
It holds that $\| \wt{\vmu} - \wh{\vmu} \|_{\infty} < \eps$, where $\wt{\vmu}$ is the black-box algorithm's estimate and $\wh{\vmu}$ its rounding to $\{-1,1\}^{d-1}$.
\end{claim}
\begin{proof}
For any coordinate $i \in [d-1],$ we will apply \Cref{lemma:band-sensitivity} in order to bound the distance between the estimated guess and the true, based on the Gaussian mass gap in each one of the $d-1$ bands.

Note that $|\vmu^\star_i| = 1$ for all $i \in [d-1]$. Also, note that the $(d-1) \times (d-1)$ matrix $\vec L_G'$ is positive definite and the minimum eigenvalue $\lambda (\vec L_G')$ is equal to the second smallest eigenvalue of the $d \times d$ Laplacian matrix $\vec L_G$. It holds that $\lambda (\vec L_G') > 0.$
Hence, the maximum entry of the covariance matrix $\vSigma = (\vec L_G')^{-1} \mathrm{opt}$ is upper bounded by $1/(\mathrm{opt} \cdot \lambda(\vec L_G')) < Q = \poly(d)$ for some value $Q$. 
Using \Cref{lemma:band-sensitivity} and the algorithm's guarantee, we have that
\[
(|\wt{\vmu}_i| - 1)^2/ Q^4 \leq |\N(\wt{\vmu}, \vSigma; S_i) - \N(\vmu^\star, \vSigma; S_i)| = o \left(1/d^{c-1} \right) \,.
\]
For sufficiently large $c$, we get that each coordinate of the estimated vector $\wt{\vmu}$ lies in an interval, centered at either $1$ or $-1$ of length $o(1/d^{c-1})$. This implies that $\| \wt{\vmu} - \vec w \|_{\infty} < \eps$ for some $\epsilon = o(1/d^{c-1})$ and some vertex $\vec w$ of the hypercube $\{-1,1\}^{\dim -1}$. Hence, we have that $\wt{\vmu}$ should lie in a ball, with respect to the $L_{\infty}$ norm, centered at one of the vertices of the $(d-1)$-hypercube with radius of order $\eps$ and note that this vertex corresponds to the rounded vector $\wh{\vmu}$ of the estimated vector. 
\end{proof}

We continue by claiming that the rounded vector $\wh\vmu$ attains a \maxcut value, that approximates the optimal value $\mathrm{opt}$ withing a factor strictly larger than $16/17$.
\begin{claim}
The \maxcut value of the rounded vector $\wh{\vmu} \in \{-1,1\}^{d-1}$ satisfies
\[
    \wh{\vmu}^T \vec L_G' \wh{\vmu} > (16/17) \cdot \mathrm{opt}\,.
\]
\end{claim}
\begin{proof}
We will make use of \Cref{lemma:ellipsoid-sensitivity}, in order to get the desired result via the Gaussian mass gap between the two means on the designed ellipsoid. In order to apply this Lemma, note that, for the true mean $\vmu^\star$, we have that $ \| \vec v^\star \|_2^2 = \| (\vSigma^{\star})^{-1/2} \vmu^{\star} \|_2^2 = ((\vmu^\star)^T \vec L_G' \vmu^\star) / \mathrm{opt} = 1 $, since the true mean attains the optimal \maxcut score. Similarly, for the rounded estimated mean $\wh{\vmu}$, the associated vector $\wh{\vec v}$ satisfies $\| \wh{\vec v} \|_2 \leq 1$, since its cut value is at most $\mathrm{opt}$. So, we can apply \Cref{lemma:ellipsoid-sensitivity} with $\vec v^{\star} = \vSigma^{-1/2}\vmu^\star$ and $\vec v = \vSigma^{-1/2} \wh{\vmu}$ and get that
\[
\frac{1 - \left(\wh{\vmu}^T \vec L_G' \wh{\vmu} \right)/\mathrm{opt}}{6\sqrt{2d+4}} - o \left(1/\sqrt{d} \right) < o\left(1/d^{c-1} \right)\,,
\]
which implies that, for some small constant $c'$, the value of the estimated mean satisfies $\wh{\vmu}^T \vec L_G' \wh{\vmu} > (1 - c' - 1/d^{c-1})\mathrm{opt}$. This implies that the algorithm $\mathcal{A}$ can approximate the \maxcut value within a factor higher than $16/17$.
\end{proof}

\paragraph{Known Covariance vs. Identity Covariance.} 
Recall that our hardness result (\Cref{theorem:impossibility-gaussian}) states that there is no algorithm with sample access to $\N_{\pi}({\vmu}^\star) = \N_{\pi}({\vmu}^\star, \vec I),$ that can compute a mean $\wt{\vmu} \in  \reals^d$ in $\poly(d)$ time such that $\dtv( \N_{\pi}(\wt\vmu), \N_{\pi}(\vmu^\star)) < 1/d^c$
for some absolute constant $c > 1$. In order to prove our hardness result, we assume that there exists such a black-box algorithm $\mathcal A$.
Hence, to make use of $\mathcal{A}$, one should provide samples generated by a coarse Gaussian with \emph{identity} covariance matrix. However, in our reduction, we show that we can generate samples from a coarse Gaussian (which is associated with the \maxcut instance) that has \emph{known} covariance matrix $\vec \Sigma$. Let us consider a sample $S \sim \N_{\pi}(\vec \mu^\star, \vec \Sigma)$. Since $\vec \Sigma$  is known, we can 
rotate the sets and give as input to the algorithm $\mathcal{A}$ the set
\[
\vec \Sigma^{-1/2} \cdot S := \left \{ \vec \Sigma^{-1/2} \vec x : \vec x \in S \right \}\,,
\]
i.e., we can implement the membership oracle 
$\mathcal{O}_{\vec \Sigma^{-1/2} \cdot S}(\cdot),$ assuming oracle access to $\mathcal{O}_S(\cdot )$. 
We have that  $\mathcal{O}_{\vec \Sigma^{-1/2} \cdot S}(\vec x) =  \mathcal{O}_S( \vec \Sigma^{1/2} \vec x )$. 
We continue with a couple of observations.
\begin{enumerate}
\item We first observe that, for any partition $\mathcal{S}$ of the $d$-dimensional Euclidean space, 
there exists another partition $\vec \Sigma^{-1/2} \cdot \mathcal{S}$ consisting of the sets $\vec \Sigma^{-1/2} \cdot S$, where $S \in \mathcal{S}$. Note that since $\vec \Sigma^{-1/2}$ is full rank, the mapping $\vec x \mapsto \vec \Sigma^{-1/2} \vec x$ is a bijection and so $\vec \Sigma^{-1/2} \cdot \mathcal{S}$ is a partition of the space with $\pi(\vec \Sigma^{-1/2} \cdot \mathcal{S}) = \pi(\mathcal{S})$.

\item We have that $\vec x \in S$ if and only if $\vSigma^{-1/2} \vec x \in \vSigma^{-1/2} \cdot S$ and so
\[
\E_{\vec x \sim \N(\vmu, \vSigma)}[\vec 1\{ \vec x \in S\}] =
\E_{\vec x \sim \N(\vmu, \vSigma)}[\vec 1\{ \vSigma^{-1/2} \vec x \in \vSigma^{-1/2} \cdot S\}]\,.
\]
Since it holds that $\vec w \sim \N(\vmu, \vSigma)$ if and only if $\vec w = \vSigma^{1/2} \vec z + \vmu$ with $\vec z \sim \N(\vec 0, \vec I)$, we get for an arbitrary subset $S \subseteq \reals^d$ that
\begin{align*}
\E_{\vec x \sim \N(\vmu, \vSigma)}[\vec 1\{ \vec x \in S\}] 
&=
\E_{\vec x \sim \N(\vec 0, \vec I)}\left [\vec 1 \left\{ \vSigma^{-1/2} \left( \vSigma^{1/2} \vec x + \vmu \right) \in \vSigma^{-1/2} \cdot S \right\} \right]
\\
&=
\E_{\vec x \sim \N(\vSigma^{-1/2} \vmu, \vec I)} \left[ \vec 1\{ \vec x \in \vSigma^{-1/2} \cdot S \} \right]
\,.
\end{align*}

\end{enumerate}
Let us consider a set $S \subseteq \reals^d$ distributed as $\mathcal{N}_{\pi}(\vec \mu^\star, \vec \Sigma).$ This set is the one that the algorithm with the known covariance matrix works with. We are now ready to combine the above two observations in order to understand what is the input to the identity covariance matrix algorithm. We have that
\begin{equation*}
\begin{split}
\Pr_{S \sim \N_{\pi}(\vec \mu^\star, \vec \Sigma)}[S] 
& = \sum_{\mathcal{S}} \vec 1\{S \in \mathcal{S} \} \pi(\mathcal{S}) \N(\vec \mu^\star, \vec \Sigma; S)\\
& = \sum_{\mathcal{S}} \vec 1\{S \in \mathcal{S} \} \pi(\mathcal{S}) \N(\vec \Sigma^{-1/2} \vec \mu^\star, \vec I ; \vSigma^{-1/2} \cdot S)\\
& = \sum_{\vSigma^{-1/2} \cdot \mathcal{S}} \vec 1\{\vSigma^{-1/2} \cdot S \in \vSigma^{-1/2} \cdot \mathcal{S} \} \pi(\vSigma^{-1/2} \cdot \mathcal{S}) \N(\vec \Sigma^{-1/2} \vec \mu^\star, \vec I ; \vSigma^{-1/2} \cdot S) \\
& = \Pr_{S' \sim \N_{\pi'}(\vec \Sigma^{-1/2} \vec \mu^\star, \vec I)}[S']\,,
\end{split}
\end{equation*}
where the set $S'$ is distributed as $\N_{\pi'}(\vec \Sigma^{-1/2} \vec \mu^\star, \vec I)$ where $\pi'$ is the 'rotated' partition distribution supported on the rotated partitions $\vSigma^{-1/2} \cdot \mathcal{S}$ for each $\mathcal{S}$ with $\pi(\mathcal{S}) > 0$.
We remark that the second equation follows from the second observation and the third equation from the first one. Hence, the algorithm $\mathcal{A}$ (the one that works with identity matrix) obtains the rotated sets (i.e., membership oracles) $\vSigma^{-1/2} \cdot S$ and 
the (unknown) target mean vector is $\vec u = \vSigma^{-1/2} \vmu^\star$.

\subsection{Efficient Mean Estimation under Convex Partitions}
\label{subsection:gaussian-convex}

In this section, we formally state and prove \Cref{inftheorem:intro-mean-estimation-gaussian}, 
which is stated in \Cref{section:intro}:
we provide an efficient algorithm for Gaussian mean estimation under \emph{convex} partitions.
The following definition of information preservation
is very similar with the one given in the introduction, see \Cref{def:intro-information-preserving}.
The difference is that we only require from $\pi$ to preserve the distances of Gaussians around the true Gaussian $\N(\vmu^\star)$
as opposed to the distance of any pair of Gaussians $\N(\vmu^\star)$: this is a somewhat more flexible assumption
about the partition distribution $\pi$ and the true Gaussian $\N(\vmu^*)$ as a pair.  
\begin{definition}
[Information Preserving Partition Distribution for Gaussians]
Let $\alpha \in [0,1]$ and 
consider a $\dim$-dimensional Gaussian distribution $\N(\vmu^{\star})$.
We say that $\pi$ is an $\alpha$-information preserving partition distribution with respect to the true Gaussian $\N(\vmu^{\star})$
if  for any Gaussian distribution $\N(\vmu)$,
it holds that $\dtv(\N_{\pi}(\vmu), \N_{\pi}(\vmu^{\star})) \geq \alpha \cdot \dtv(\N(\vmu), \N(\vmu^{\star}))$.
\end{definition}
We refer to \Cref{appendix:preservation} for a geometric condition, under which a partition
is $\alpha$-information preserving.  In particular, we prove that a partition is $\alpha$-information
preserving if, for any hyperplane, it holds that the mass of the cells of the partition that do not intersect
with the hyperplane is at least $\alpha$.  This is true for most natural partitions, see e.g., the Voronoi
diagram of \Cref{fig:convex}.

We continue with a formal statement of \Cref{inftheorem:intro-mean-estimation-gaussian}.
\begin{theorem}[Gaussian Mean Estimation with Convex Partitions]
\label{theorem:intro-mean-estimation-gaussian}
Let $\eps, \delta \in  (0,1)$.
Consider the generative process of coarse $\dim$-dimensional Gaussian data $\N_{\pi}(\vmu^{\star})$, as in \Cref{definition:intro-gaussian-coarse}.
Assume that the partition distribution $\pi$ is  $\alpha$-information preserving and is supported on convex partitions of $\reals^d$. The following hold.
\begin{enumerate}
    \item The empirical log-likelihood objective
    \[
    \L_N(\vmu) = \frac{1}{N} \sum_{i=1}^N \log \N(\vmu; S_i)
    \]
    is concave with respect to $\vec \mu$ where the sets $S_i$ for $i \in [N]$ are i.i.d. samples from $\N_{\pi}(\vmu^\star)$.
    \item There exists an algorithm, that draws $N = \wt{O}(d/(\eps^2 \alpha^2)\log(1/\delta))$ samples from $\N_{\pi}(\vmu^{\star})$
    and computes an estimate $\wt{\vmu}$ that satisfies
    \(
    \dtv(\N(\wt{\vmu}), \N(\vmu^\star)) \leq \eps\,,
    \)
    with probability at least $1-\delta$.
\end{enumerate}
\end{theorem}

In this section, we discuss and establish the two structural 
lemmata required in order to prove
\Cref{theorem:intro-mean-estimation-gaussian}. 
Our goal is to maximize the empirical log-likelihood objective
\begin{equation}
\label{equation:appendix:likelihood}
    \mathcal{L}_N(\vmu) = \frac{1}{N}\sum_{i=1}^N \log \N(\vmu; S_i)\,,
\end{equation}
where the $N$ (convex) sets $S_1,\ldots, S_N$ are drawn from the coarse Gaussian generative process $\N_{\pi}(\vmu^\star)$.
We first show that the above empirical likelihood is a concave objective with respect
to $\vec \mu \in \reals^d$. In the following lemma, we show that the log-probability of a convex set $S$, i.e., the function
$\log \N(\vec \mu; S)$ is a concave function of the mean $\vec \mu$.
\begin{lemma}
[Concavity of Log-Likelihood]
\label{lemma:concavity}
Let $S \subseteq \reals^{\dim}$ be a convex set. The function $\log\N(\vmu;S)$ is concave with respect to the mean vector $\vmu \in \reals^{\dim}$.
\end{lemma}
In order to prove that the Hessian matrix of this objective is negative semi-definite, 
we use a variant of the Brascamp-Lieb inequality. 
Having established the concavity of the empirical log-likelihood, 
we next have to bound the sample complexity of the empirical log-likelihood.
We prove the following lemma.
\begin{lemma}
[Sample Complexity of Empirical Log-Likelihood]
\label{lemma:sample-complexity-convex-gaussian}
Let $\eps, \delta \in (0,1)$ and consider a generative process for coarse $\dim$-dimensional Gaussian data $\N_{\pi}(\vmu^{\star})$ (see \Cref{definition:intro-gaussian-coarse}). Also, assume that every $\S \in \mathrm{supp}(\pi)$ is a convex partition of the Euclidean space. 
Let $N = \wt\Omega({\dim}/(\eps^2\alpha^2)\log(1/\delta))$. 
Consider the empirical log-likelihood objective 
\[
\mathcal{L}_N(\vmu) = \frac{1}{N}\sum_{i=1}^N \log \N(\vmu; S_i)\,.
\]
Then, with probability at least $1-\delta$, we have that, for any Gaussian distribution $\N(\vmu)$ that satisfies $\dtv(\N(\vmu),\N(\vmu^{\star})) \geq \eps$, it holds that
\(
\max_{\wt{\vmu} \in \reals^{\dim}} \mathcal{L}_N(\wt{\vmu}) - \mathcal{L}_N(\vmu) \geq \Omega(\eps^2 \alpha^2) \,.
\)
\end{lemma}
The above lemma states that, given roughly $\wt{O}({\dim}/(\eps^2 \alpha^2))$ samples from $\N_{\pi}(\vmu^{\star})$,
we can guarantee that the maximizer $\wt{\vmu}$ of the empirical log-likelihood achieves a total variation gap at most $\eps$ against the true mean vector $\vmu^\star,$ i.e., $\dtv(\N(\wt{\vmu}), \N(\vmu^\star)) \leq \eps$. 
In fact, thanks to the concavity of the empirical log-likelihood objective, it suffices to show that Gaussian distributions $\N(\vmu)$, that satisfy $\dtv(\N(\vmu),\N(\vmu^{\star})) > \eps$, will also be significantly sub-optimal solutions of the empirical log-likelihood maximization. The key idea in order to attain the desired sample complexity, is that is suffices to focus on guess vectors $\vmu$ that lie in a sphere of radius $\Omega(\eps)$. Technically, the proof of \Cref{lemma:sample-complexity-convex-gaussian} relies on a concentration result of likelihood ratios and in the observation that, while the empirical log-likelihood objective $\mathcal{L}_N$ is concave (under convex partitions), the regularized objective $\mathcal{L}_N(\vmu) + \|\vmu\|_2^2$ is convex with respect to the guess mean vector $\vmu$. 

\subsubsection{Concavity of Log-likelihood: Proof of \Cref{lemma:concavity}}
In this section, we show that the log-likelihood is concave when the underlying partitions are convex. 
The Hessian of the log-likelihood $\mathcal{L}$ for the set $S$ has a notable property. When restricted to a direction $\vec v \in \reals^d$, the quadratic $\vec v^T (\nabla^2 \mathcal{L}) \vec v$ quantifies the variance reduction, observed between the distributions $\N_S$ (Gaussian conditioned on $S$) and $\N$ (unrestricted Gaussian, i.e., $S = \reals^d$). 
When the set $S$ is convex (and, hence the indicator function $\vec 1_S$ is log-concave), the variance of the unrestricted Gaussian is always larger than the conditional one. This intriguing result is an application of a variation of the Brascamp-Lieb inequality, due to Hargé (see \Cref{lemma:bl-ineq} for the inequality that we utilize). Recall that, both the empirical and the population log-likelihood objectives are convex combinations of the function $f(\vmu, \vSigma; S) = \log \N(\vmu, \vSigma; S)$ and, hence, it suffices to show that $f$ is concave with respect to $\vmu \in \reals^d$, when the set $S$ is convex.
\newline
\begin{proof}[\textit{Proof of}~\Cref{lemma:concavity}]
Without loss of generality, we can take $\vSigma = \vec I\in \reals^{d \times d}$. Let $f(\vmu;S) = \log \N(\vmu, \vec I; S)$ for an arbitrary convex set $S\subseteq \reals^d$. The gradient $\nabla_{\vmu} f(\vmu)$ of $f$ with respect to $\vmu$ is equal to
\[
\nabla_{\vmu} \left ( \log\int_S\frac{1}{\sqrt{(2\pi)^{\dim}}}\exp\left (-\frac{(\vec x - \vmu)^T (\vec x - \vmu)}{2}\right ) d\vec x \right) =  \frac{\int_S \vec x \exp(-(\vec x - \vmu)^T (\vec x - \vmu)/2) d\vec x }{\int_S \exp(-(\vec x - \vmu)^T (\vec x - \vmu)/2) d\vec x } - \vmu\,.
\]
Hence, we get that
\[
\nabla_{\vmu} f(\vmu) = \E_{\vec x \sim \N_S(\vec \mu, \vec I)}[\vec x] - \vmu\,.
\]
We continue with the computation of the Hessian of the function $f$ with respect to $\vmu$
\[
\nabla_{\vmu}^2 f(\vmu) = -\vec I + \frac{\int_S \vec x(\vec x - \vmu)^T \N(\vmu, \vec I ; \vec x) d\vec x }{\N(\vmu, \vec I; S)} - \frac{\Big (\int_S \vec x \N(\vmu, \vec I ; \vec x)d\vec x\Big )\Big (\int_S(\vec x- \vmu)^T \N(\vmu, \vec I; \vec x)d\vec x\Big )}{\N(\vmu, \vec I; S)^2}\,,
\]
and, so, we have that
\[
\nabla_{\vmu}^2 f(\vmu) = -\vec I + \Big  (\E_{\vec x \sim \N_S(\vmu, \vec I)}[\vec x \vec x^T] - \E_{\vec x \sim \N_S(\vmu, \vec I)}[\vec x]\E_{\vec x \sim \N_S(\vmu, \vec I)}[\vec x^T]\Big ) = \Cov_{\vec x \sim \N_S(\vmu, \vec I)}[\vec x] - \vec I \,.
\]
Observe that, when $S = \reals^{\dim}$, we get that both the gradient and the Hessian vanish. In order to show the concavity of $f$ with respect to the mean vector $\vmu,$ consider an arbitrary vector $\vec v \in \reals^{\dim}$ in the ball $\| \vec v\|_2 = 1$. We have the quadratic form
\[
\vec v^T \nabla_{\vmu}^2 f(\vmu) \vec v = \vec v^T \Cov_{\vec x \sim \N_S(\vmu, \vec I)}[\vec x] \vec v - 1 = \E_{\vec x \sim \N_S(\vmu, \vec I) }\Big [\vec (\vec v^T \vec x)^2\Big ] - \Big (\E_{\vec x \sim \N_S(\vmu, \vec I) }[\vec v^T \vec x]\Big )^2  - 1 \,.
\]
In order to show the desired inequality, we will apply the following variant of the Brascamp-Lieb 
inequality.
\begin{lemma}
[Brascamp-Lieb Inequality, Hargé (see~\cite{guionnet2009large})]
\label{lemma:bl-ineq}
Let $g$ be convex function on $\reals^{\dim}$ and let $S$ be a convex set on $\reals^{\dim}$. Let $\N(\vmu, \vSigma)$ be the Gaussian distribution on $\reals^{\dim}$. It holds that
\begin{equation}
    \E_{\vec x \sim N_S} \left [g \left( \vec x + \vmu -  \E_{\vec x \sim \N_S}[\vec x] \right) \right]  \leq \E_{\vec x \sim \N}[g(\vec x)]\,.
\end{equation}
\end{lemma}
\noindent We apply the above Lemma with $g(\vec x) = (\vec v^T \vec x)^2$. We get that
\[
\int_{\reals^d}(\vec v^T (\vec x + \vmu - \E_{y \sim \N_S(\vmu, \vec I)}\vec y))^2 \cdot \frac{\vec 1_S(\vec x) \N(\vmu, \vec I ; \vec x) d\vec x}{\int_{\reals^d} \vec 1_S(\vec x) \N(\vmu, \vec I ; \vec x) d\vec x} \leq \int_{\reals^d} (\vec v^T\vec x)^2 \N(\vmu, \vec I; \vec x) d\vec x\,.
\]
Hence, we get the desired variance reduction in the direction $\vec v$
\[
\Var_{\vec x \sim \N_S(\vmu, \vec I)}[\vec v^T \vec x] \leq \Var_{\vec x \sim \N(\vmu, \vec I)}[\vec v^T \vec x]\,,
\]
that implies the concavity of the function $\log \N(\vmu, \vSigma; S)$ for convex sets $S$ with respect to the mean vector $\vmu \in \reals^d.$
\end{proof}

\subsubsection{Sample Complexity of Empirical Log-Likelihood:
Proof of \Cref{lemma:sample-complexity-convex-gaussian} }
\label{appendix:sample-complexity-convex-gaussian}
In this section, we provide the proof of \Cref{lemma:sample-complexity-convex-gaussian}.  
This lemma analyzes the sample complexity of the empirical log-likelihood maximization $\mathcal{L}_N$, whose concavity (in convex partitions) was established in \Cref{lemma:concavity}.
We show that, given roughly $N = \wt{O}({\dim}/(\eps^2 \alpha^2))$ samples from $\N_{\pi}(\vmu^{\star})$, we can guarantee that Gaussian distributions $\N(\vmu)$ with mean vectors $\vmu$, that are far from the true Gaussian $\N(\vmu^{\star})$ in total variation distance, will also be sub-optimal solutions of the empirical maximization of the log-likelihood objective, i.e., they are far from being maximizers of the empirical log-likelihood objective. We first give an overview of the proof of \Cref{lemma:sample-complexity-convex-gaussian}. 
In \Cref{prop:unsupervised} we provided a similar sample complexity bound
for an empirical log-likelihood objective.
However, in contrast to the analysis of \Cref{prop:unsupervised}, the parameter space 
is now unbounded -- $\vec \mu$ can be any vector of $\reals^d$ --
and we cannot construct a cover of the whole space with finite size.  
However, thanks to the concavity of the empirical log-likelihood objective $\mathcal{L}_N$, we can show that it suffices to focus on guess vectors $\vmu$ that lie in a sphere $\partial \mathcal{B}$ (i.e., the boundary of a ball $\mathcal{B}$) of radius $\Omega(\eps)$. This argument heavily relies on the claim that the maximizer of the empirical log-likelihood $\mathcal{L}_N$ lies inside $\mathcal{B}$, which can be verified by monotonicity properties of the log-likelihood. Afterwards, we consider a discretization $\mathcal{C}$ of the sphere and, for any vector $\vmu \in \mathcal{C}$, we can prove that $\mathcal{L}_N(\vmu^\star) - \mathcal{L}_N(\vmu) \geq \Omega(\alpha^2 \eps^2)$. The main technical tool for this claim is a concentration result on likelihood ratios and the fact that the partition distribution is $\alpha$-information preserving. In order to extend this property to the whole sphere, we exploit the convexity (with respect to $\vmu$) of a regularized version of the empirical log-likelihood objective $\mathcal{L}_N(\vmu) + \|\vmu\|_2^2$. The complete proof follows.
\newline
\begin{proof}[\textit{Proof of}~\Cref{lemma:sample-complexity-convex-gaussian}]
Let $\wt{\vmu}$ be the maximizer of the empirical log-likelihood objective
 \[
 \wt{\vmu} = \arg \max_{\vmu \in \reals^{\dim}} \frac{1}{N} \sum_{i=1}^{N}\log \N(\vmu; S_i) \,.
 \]
 Since $\wt \vmu$ is the maximizer of the empirical objective, it is sufficient to 
 prove that for any Gaussian $\N(\vmu)$ whose total variation distance with $\N(\vmu^\star)$
 is greater than $\eps$, it holds that
 $\mathcal{L}_N(\vmu^\star) - \mathcal{L}_N(\vmu) \geq \Omega(\alpha^2 \eps^2)$.
 
 Moreover, we know that when $\|\vmu_1 - \vmu_2\|_2$ is smaller than some sufficiently small absolute constant, it 
 holds $\dtv(\N(\vmu_1), \N(\vmu_2)) \geq \Omega(\| \vmu_1 - \vmu_2 \|_2)$.
 Therefore, any Gaussian whose mean $\vmu$ is far from $\vmu^\star$, i.e., 
 $\|\vmu - \vmu^\star\|_2 \geq \Omega(\eps)$ will be in 
 total variation distance at least $\eps$ from $\N(\vmu^\star)$
 Therefore, to prove the lemma, it suffices to prove it for Gaussians whose means lie 
 outside of a ball $\mathcal{B}$ of radius $\rho := \Omega(\eps)$ around $\vmu^\star$.
 
 Since all observed sets $S_i$ are convex, the empirical log-likelihood  objective $\mathcal{L}_N(\vmu)$ 
is concave with respect to $\vmu$, see \Cref{lemma:concavity}.  Since $\mathcal{L}_N$ 
is concave, it suffices to prove that for any $\vmu$ that lies exactly on the sphere
of radius $\rho$, i.e., the surface of the ball $\mathcal{B}$ it holds 
$\mathcal{L}_N(\vec \mu^\star) - \mathcal{L}_N(\vec \mu) \geq \Omega(\alpha^2 \eps^2)$.
To prove this we first show that the maximizer of the empirical objective 
$\wt \vmu$ has to lie inside the ball $\mathcal{B}$. 
Assuming that $\wt\vmu $ lies outside of $\mathcal{B}$, let $\vec r_1$ and $\vec r_2$ 
be the antipodal points on the sphere $\partial \mathcal B$ that belong to the line $\wt \vmu$ connecting
$\wt \vmu$ and $\vmu^\star$ and assume that $\vec r_2$ lies between $\vmu^\star$ and $\wt\vmu$. 
In that case the restriction of $\mathcal{L}_N$ on that line
cannot be concave, since it has to be increasing from $\vec r_1$ to $\vmu^\star$,
decreasing from $\vmu^\star$ to $\vec r_2$ and then increase again from $\vec r_2$ to $\wt \vmu$.
Thus, $\wt \vmu$ lies inside $\mathcal{B}$.  Now, by concavity of $\mathcal{L}_N$, we obtain
that, by projecting any point $\vmu$ that lies outside of the ball $\mathcal{B}$ onto $\mathcal{B}$, 
we can only increase its empirical likelihood.  Therefore, it suffices to consider
only points that lie on the sphere $\partial \mathcal{B}$.

We will now show that the claim is true for any $\vmu \in \partial \mathcal{B}$.  We can create 
a cover of the sphere of radius $\rho \sqrt{ 1+ c \alpha^2}$, centered at $\vmu^\star$ 
for some sufficiently small absolute 
constant $c >0$, whose convex hull contains $\mathcal{B}$.  The following lemma shows that
such a cover can be constructed with $(1/(\alpha \eps))^{O(d)}$ points.
\begin{lemma}[see, e.g., Corollary 4.2.13 of \cite{vershynin2018high}]
\label{lem:cover}
  For any $\eps>0$, there exists an $\eps$-cover $\mathcal{C}$ of the unit sphere in $\reals^k$, with
  respect to the $\ell_2$-norm, of size $O((1/\eps)^k)$.  Moreover, the convex hull of the cover $\mathcal{C}$
  contains the sphere of radius $1-\eps$.
\end{lemma}
 
 Since the partition distribution $\pi$ is $\alpha$-information preserving we obtain
 that for any $\vmu \in \mathcal{C}$, it holds $\dtv(\N_{\pi}(\vmu), \N_{\pi}(\vmu^\star)) \geq \Omega(\alpha \eps)$.
 Applying \Cref{lemma:massart} with $x = O(\log(|\mathcal{C}|/\delta)) = O(\dim \log(1/(\eps \delta)))$, we get that, with $N = \wt{O}(d/(\alpha^2 \eps^2) \log(1/\delta))$, with probability at 
 least $1-\delta$, it holds that, for any $\vmu$ in the cover $\mathcal{C}$, we have
 \begin{equation}
 \label{eq:likelihood-on-cover}
 \mathcal{L}_N(\vmu^{\star}) - \mathcal{L}_N(\vmu) \geq \dtv(\N_{\pi}(\vmu^{\star}), \N_{\pi}(\vmu))^2 - \alpha^2 \eps^2/2
 \geq \Omega(\alpha^2 \eps^2)\,.
 \end{equation}
 
Next, we need to extend this bound from the elements of  the cover $\mathcal{C}$ to all elements of
the sphere $\partial \mathcal{B}$.  In what follows, in order to simplify notation, we may assume
without loss of generality that $\vmu^\star = \vec 0$. We are going to use the  fact that  
$\log(\N(\vmu;S_i)) + \| \vmu \|_2^2/2$ is convex. 
To see that, write
\[
\log(\N(\vmu;S_i)) + \| \vmu \|_2^2/2
= 
\log\Big( 
e^{\| \vmu\|_2^2/2} \int_S e^{-\| \vec x - \vec \mu\|_2^2/2} d \vec x
\Big)
= 
\log\Big( 
 \int_S e^{-\|\vec x\|_2^2/2 + \vec x^T \vmu } 
 d \vec x
\Big),
\]
which is a log-sum-exp function and thus convex (this can also be verified by directly computing the Hessian with respect to $\vmu$).
This means that $\mathcal{L}_N(\vmu) + \|\vmu\|_2^2$ is  also convex with respect to $\vmu$.
Let $\vmu \in \partial B$. From the construction of the cover $\mathcal{C}$, we have that 
its convex hull contains the sphere $\partial B$.  Therefore, $\vmu$ can be written 
as a convex combination of points of the cover, i.e., $\vmu = \sum_{i=1}^{|\mathcal{C}|} \alpha_i \vmu_i$,
where $\vmu_i \in \mathcal C$.  The convexity of $\mathcal{L}_N(\vmu) + \|\vmu\|_2^2$ implies that
\[
\mathcal{L}_N(\vmu) + \|\vmu\|_2^2 \leq  \sum_{i=1}^{|\mathcal{C}|} \alpha_i (\mathcal{L}_N(\vmu_i) + \| \vmu_i\|_2^2)
\leq \max_i \mathcal{L}_N(\vmu_i) + \rho^2(1+ c \alpha^2) \,,
\]
where to get the last inequality we used the fact that all points of our cover $\mathcal{C}$ 
belong to the sphere of radius $\rho \sqrt{1 + c \alpha^2}$.  Since $\|\vmu\|_2^2 = \rho^2$ 
the above inequality implies that $\mathcal{L}_N(\vmu) \leq \max_{i} \mathcal{L}_N (\vmu_i) + c \alpha^2 \rho^2$.
Combining this inequality with Equation~\eqref{eq:likelihood-on-cover}, we obtain that, 
since $c$ is sufficiently small and $\rho = \Theta (\eps)$, it holds
$\mathcal{L}_N(\vmu) \leq \mathcal{L}_N(\vmu^\star) - \Omega(\eps^2 \alpha^2)$.
\end{proof}

\subsubsection{The Proof of \Cref{theorem:intro-mean-estimation-gaussian}}
We conclude this section with the proof of \Cref{theorem:intro-mean-estimation-gaussian}.
Since the likelihood function is concave (and therefore can
be efficiently optimized) we focus mainly on bounding the sample complexity of our algorithm.
\begin{proof}[\textit{Proof of}~\Cref{theorem:intro-mean-estimation-gaussian}]

Let us assume that the partition distribution $\pi$ is $\alpha$-information preserving and that is supported on \emph{convex partitions} of $\reals^d$. Our goal is to show that there exists an algorithm, that draws $\wt{O}(d/(\eps^2 \alpha^2)\log(1/\delta))$ samples from $\N_{\pi}(\vmu^{\star})$ 
and computes an estimate $\wt{\vmu} \in \reals^d$ so that $ \dtv(\N(\wt{\vmu}), \N(\vmu^\star)) \leq \eps$
with probability at least $1-\delta$. The algorithm works as follows: it optimizes the empirical log-likelihood objective
\[
\L_N(\vmu) = \frac{1}{N} \sum_{i=1}^{N}\log \N(\vmu; S_i)\,,
\]
where the samples are i.i.d. and $S_i \sim \N_{\pi}(\vmu^\star)$ for any $i \in [N]$. 
Using \Cref{lemma:concavity}, we establish that the function $\L_N$ is concave with respect to the mean $\vmu \in \reals^d$. This follows from the fact that convex combinations of concave functions remain concave.
From  \Cref{lemma:sample-complexity-convex-gaussian}, we obtain that it suffices to compute 
a point $\vec \mu$ such that 
$\L_N(\vmu) \geq \max_{\vmu'} \L_N(\vmu') - O(\alpha^2 \eps^2)$. Specifically, given roughly $\wt{O}({\dim}/(\eps^2 \alpha^2))$ samples from $\N_{\pi}(\vmu^{\star})$,
we can guarantee, with high probability, that the maximizer $\wt{\vmu}$ of the empirical log-likelihood achieves a total variation gap at most $\eps$ against the true mean vector $\vmu^\star,$ i.e., $\dtv(\N(\wt{\vmu}), \N(\vmu^\star)) \leq \eps$.
\end{proof}

We proceed with a discussion about the running time of the above algorithm. Since $\L_N(\vmu)$ is a concave function with respect to $\vmu$, this can be done efficiently.
For example, we may perform gradient-ascent:  for a fixed convex set $S \subseteq \reals^d$ 
the gradient of the function $f(\vmu) = \log \N(\vmu; S) = \log \E_{\vec x \sim \N(\vmu)} \left[ \vec 1\{ \vec x \in S \} \right]$
(see \Cref{lemma:concavity}) is equal to
\[
\nabla_{\vmu} f(\vmu) = \E_{\vec x \sim \N_S(\vec \mu)}[\vec x] - \vmu\,.
\]
In order to compute the gradient of $f$, it suffices to approximately compute 
$
\E_{\vec x \sim \N_S(\vec \mu)}[\vec x] $ 
$= \E_{\vec x \sim \N(\vmu)}[\vec x $ $\vec 1\{\vec x \in S\}]$ $/\N(\vec \mu; S)\,.
$
Both terms of this ratio can be estimated using independent samples from 
the distribution $\N(\vmu)$ and access to the oracle 
$\mathcal{O}_{S}(\cdot)$, since the mean $\vmu$ is known (the current guess of the learning algorithm). 
Hence, the running time will be polynomial in the number of samples using, e.g., the ellipsoid algorithm.

\begin{remark}
\label{rem:runtime}
We remark that a precise calculation of the runtime would 
also depend on the regularity of the concave objective (Lipschitz or smoothness assumptions
etc.) which in turn depend on the geometric properties of
the sets.  We opt not to track such dependencies since our 
main result is that, in this setting, the likelihood objective is concave and therefore can be efficiently optimized using standard black-box optimization techniques.
\end{remark}

\paragraph{Acknowledgments} 
We thank the anonymous reviewers for useful remarks and comments on the presentation of our manuscript.
Dimitris Fotakis and Alkis Kalavasis were supported by the Hellenic Foundation for Research and Innovation (H.F.R.I.) under the ``First Call for H.F.R.I. Research Projects to support Faculty members and Researchers and the procurement of high-cost research equipment grant'', project BALSAM, HFRI-FM17-1424. Christos Tzamos and Vasilis Kontonis were supported by the NSF grant CCF-2008006.

\bibliography{refs.bib}

\newcommand{\etalchar}[1]{$^{#1}$}
\begin{thebibliography}{CSGGSR14}

\bibitem[AD98]{aslam1998general}
Javed~A Aslam and Scott~E Decatur.
\newblock General bounds on statistical query learning and pac learning with
  noise via hypothesis boosting.
\newblock {\em Information and Computation}, 141(2):85--118, 1998.

\bibitem[AL88]{angluin1988learning}
Dana Angluin and Philip Laird.
\newblock Learning from noisy examples.
\newblock {\em Machine Learning}, 2(4):343--370, 1988.

\bibitem[B{\etalchar{+}}96]{breen1996regression}
Richard Breen et~al.
\newblock {\em Regression models: Censored, sample selected, or truncated
  data}, volume 111.
\newblock Sage, 1996.

\bibitem[BDH{\etalchar{+}}20]{bakshi2020outlier}
Ainesh Bakshi, Ilias Diakonikolas, Samuel~B Hopkins, Daniel Kane, Sushrut
  Karmalkar, and Pravesh~K Kothari.
\newblock Outlier-robust clustering of gaussians and other non-spherical
  mixtures.
\newblock In {\em 2020 IEEE 61st Annual Symposium on Foundations of Computer
  Science (FOCS)}, pages 149--159. IEEE Computer Society, 2020.

\bibitem[BDMN05]{blum2005practical}
Avrim Blum, Cynthia Dwork, Frank McSherry, and Kobbi Nissim.
\newblock Practical privacy: the sulq framework.
\newblock In {\em Proceedings of the twenty-fourth ACM SIGMOD-SIGACT-SIGART
  symposium on Principles of database systems}, pages 128--138, 2005.

\bibitem[BDNP21]{bhattacharyya2020efficient}
Arnab Bhattacharyya, Rathin Desai, Sai~Ganesh Nagarajan, and Ioannis Panageas.
\newblock Efficient statistics for sparse graphical models from truncated
  samples.
\newblock In {\em International Conference on Artificial Intelligence and
  Statistics}, pages 1450--1458. PMLR, 2021.

\bibitem[BF15]{balcan2015statistical}
Maria~Florina Balcan and Vitaly Feldman.
\newblock Statistical active learning algorithms for noise tolerance and
  differential privacy.
\newblock {\em Algorithmica}, 72(1):282--315, 2015.

\bibitem[BFKV98]{blum1998polynomial}
Avrim Blum, Alan Frieze, Ravi Kannan, and Santosh Vempala.
\newblock A polynomial-time algorithm for learning noisy linear threshold
  functions.
\newblock {\em Algorithmica}, 22(1):35--52, 1998.

\bibitem[BS14]{blanchard2014decontamination}
Gilles Blanchard and Clayton Scott.
\newblock Decontamination of mutually contaminated models.
\newblock In {\em Artificial Intelligence and Statistics}, pages 1--9. PMLR,
  2014.

\bibitem[BSS{\etalchar{+}}20]{bukchin2020fine}
Guy Bukchin, Eli Schwartz, Kate Saenko, Ori Shahar, Rogerio Feris, Raja Giryes,
  and Leonid Karlinsky.
\newblock Fine-grained angular contrastive learning with coarse labels.
\newblock {\em arXiv preprint arXiv:2012.03515}, 2020.

\bibitem[CDCM18]{chen2018understanding}
Zhuo Chen, Ruizhou Ding, Ting-Wu Chin, and Diana Marculescu.
\newblock Understanding the impact of label granularity on cnn-based image
  classification.
\newblock In {\em 2018 IEEE International Conference on Data Mining Workshops
  (ICDMW)}, pages 895--904. IEEE, 2018.

\bibitem[CDGS20]{cheng2020high}
Yu~Cheng, Ilias Diakonikolas, Rong Ge, and Mahdi Soltanolkotabi.
\newblock High-dimensional robust mean estimation via gradient descent.
\newblock In {\em International Conference on Machine Learning}, pages
  1768--1778. PMLR, 2020.

\bibitem[CGAD22]{cauchois2022predictive-theory-duchi}
Maxime Cauchois, Suyash Gupta, Alnur Ali, and John Duchi.
\newblock Predictive inference with weak supervision.
\newblock {\em arXiv preprint arXiv:2201.08315}, 2022.

\bibitem[CODA08]{come2008mixture}
Etienne C{\^o}me, Latifa Oukhellou, Thierry Den{\oe}ux, and Patrice Aknin.
\newblock Mixture model estimation with soft labels.
\newblock In {\em Soft Methods for Handling Variability and Imprecision}, pages
  165--174. Springer, 2008.

\bibitem[Coh16]{Cohen91}
A~Clifford Cohen.
\newblock {\em Truncated and censored samples: theory and applications}.
\newblock CRC press, 2016.

\bibitem[CPCP14]{chen2014ambiguously}
Yi-Chen Chen, Vishal~M Patel, Rama Chellappa, and P~Jonathon Phillips.
\newblock Ambiguously labeled learning using dictionaries.
\newblock {\em IEEE Transactions on Information Forensics and Security},
  9(12):2076--2088, 2014.

\bibitem[CRB20]{cabannnes2020structured-icml-theory}
Vivien Cabannnes, Alessandro Rudi, and Francis Bach.
\newblock Structured prediction with partial labelling through the infimum
  loss.
\newblock In {\em International Conference on Machine Learning}, pages
  1230--1239. PMLR, 2020.

\bibitem[CS12]{cid2012proper}
Jes{\'u}s Cid-Sueiro.
\newblock Proper losses for learning from partial labels.
\newblock {\em Advances in neural information processing systems}, 25, 2012.

\bibitem[CSGGSR14]{cid2014consistency-theory}
Jes{\'u}s Cid-Sueiro, Dar{\'\i}o Garc{\'\i}a-Garc{\'\i}a, and Ra{\'u}l
  Santos-Rodr{\'\i}guez.
\newblock Consistency of losses for learning from weak labels.
\newblock In {\em Joint European Conference on Machine Learning and Knowledge
  Discovery in Databases}, pages 197--210. Springer, 2014.

\bibitem[CSJT09]{cour2009learning-ambiguous-theory}
Timothee Cour, Benjamin Sapp, Chris Jordan, and Ben Taskar.
\newblock Learning from ambiguously labeled images.
\newblock In {\em 2009 IEEE Conference on Computer Vision and Pattern
  Recognition}, pages 919--926. IEEE, 2009.

\bibitem[CST11a]{cour2011learning}
Timothee Cour, Ben Sapp, and Ben Taskar.
\newblock Learning from partial labels.
\newblock {\em The Journal of Machine Learning Research}, 12:1501--1536, 2011.

\bibitem[CST11b]{cour2011learning-theory}
Timothee Cour, Ben Sapp, and Ben Taskar.
\newblock Learning from partial labels.
\newblock {\em The Journal of Machine Learning Research}, 12:1501--1536, 2011.

\bibitem[CSV17]{CSV17}
Moses Charikar, Jacob Steinhardt, and Gregory Valiant.
\newblock Learning from {U}ntrusted {D}ata.
\newblock In {\em Proceedings of the 49th Annual {ACM} {SIGACT} Symposium on
  Theory of Computing, {STOC} 2017, Montreal, QC, Canada, June 19-23, 2017},
  pages 47--60, 2017.

\bibitem[CSZ06]{chapelle-book}
Olivier Chapelle, Bernhard Sch\"{o}lkopf, and Alexander Zien.
\newblock {\em Semi-Supervised Learning (Adaptive Computation and Machine
  Learning)}.
\newblock The MIT Press, 2006.

\bibitem[CW01]{carbery2001distributional}
Anthony Carbery and James Wright.
\newblock Distributional and {$L^q$} norm inequalities for polynomials over
  convex bodies in {$\mathbb{R}^n$}.
\newblock {\em Mathematical Research Letters}, 8, 05 2001.

\bibitem[d'A08]{d2008smooth}
Alexandre d'Aspremont.
\newblock Smooth optimization with approximate gradient.
\newblock {\em SIAM Journal on Optimization}, 19(3):1171--1183, 2008.

\bibitem[DGN14]{devolder2014first}
Olivier Devolder, Fran{\c{c}}ois Glineur, and Yurii Nesterov.
\newblock First-order methods of smooth convex optimization with inexact
  oracle.
\newblock {\em Mathematical Programming}, 146(1):37--75, 2014.

\bibitem[DGTZ18]{DGTZ18}
Constantinos Daskalakis, Themis Gouleakis, Christos Tzamos, and Manolis
  Zampetakis.
\newblock Efficient {S}tatistics, in {H}igh {D}imensions, from {T}runcated
  {S}amples.
\newblock In {\em 59th Annual IEEE Symposium on Foundations of Computer Science
  (FOCS)}, pages 639--649. IEEE, 2018.

\bibitem[DGTZ19]{daskalakis2019computationally}
Constantinos Daskalakis, Themis Gouleakis, Christos Tzamos, and Manolis
  Zampetakis.
\newblock Computationally and statistically efficient truncated regression.
\newblock In {\em Conference on Learning Theory}, pages 955--960. PMLR, 2019.

\bibitem[DKFF13]{JKF13}
Jia Deng, Jonathan Krause, and Li~Fei-Fei.
\newblock Fine-grained crowdsourcing for fine-grained recognition.
\newblock In {\em Proceedings of the 2013 IEEE Conference on Computer Vision
  and Pattern Recognition}, CVPR '13, page 580–587, USA, 2013. IEEE Computer
  Society.

\bibitem[DKK{\etalchar{+}}16]{DKK+16b}
Ilias Diakonikolas, Gautam Kamath, Daniel~M. Kane, Jerry Li, Ankur Moitra, and
  Alistair Stewart.
\newblock Robust {E}stimators in {H}igh {D}imensions without the
  {C}omputational {I}ntractability.
\newblock In {\em {IEEE} 57th Annual Symposium on Foundations of Computer
  Science, {FOCS} 2016, 9-11 October 2016, Hyatt Regency, New Brunswick, New
  Jersey, {USA}}, pages 655--664, 2016.

\bibitem[DKK{\etalchar{+}}17]{DKK+17}
Ilias Diakonikolas, Gautam Kamath, Daniel~M. Kane, Jerry Li, Ankur Moitra, and
  Alistair Stewart.
\newblock Being {R}obust (in {H}igh {D}imensions) {C}an {B}e {P}ractical.
\newblock In {\em Proceedings of the 34th International Conference on Machine
  Learning, {ICML} 2017, Sydney, NSW, Australia, 6-11 August 2017}, pages
  999--1008, 2017.

\bibitem[DKK{\etalchar{+}}18]{DKK+18}
Ilias Diakonikolas, Gautam Kamath, Daniel~M. Kane, Jerry Li, Ankur Moitra, and
  Alistair Stewart.
\newblock Robustly {L}earning a {G}aussian: {G}etting {O}ptimal {E}rror,
  {E}fficiently.
\newblock In {\em Proceedings of the Twenty-Ninth Annual {ACM-SIAM} Symposium
  on Discrete Algorithms, {SODA} 2018, New Orleans, LA, USA, January 7-10,
  2018}, pages 2683--2702, 2018.

\bibitem[DKK{\etalchar{+}}19]{DKK+19-sever}
I.~Diakonikolas, G.~Kamath, D.~Kane, J.~Li, J.~Steinhardt, and Alistair
  Stewart.
\newblock Sever: {A} robust meta-algorithm for stochastic optimization.
\newblock In {\em Proceedings of the 36th International Conference on Machine
  Learning, {ICML} 2019}, pages 1596--1606, 2019.

\bibitem[DKKZ20]{diakonikolas2020algorithms}
Ilias Diakonikolas, Daniel~M Kane, Vasilis Kontonis, and Nikos Zarifis.
\newblock Algorithms and sq lower bounds for pac learning one-hidden-layer relu
  networks.
\newblock In {\em Conference on Learning Theory}, pages 1514--1539. PMLR, 2020.

\bibitem[DKS17]{diakonikolas2017statistical}
Ilias Diakonikolas, Daniel~M Kane, and Alistair Stewart.
\newblock Statistical query lower bounds for robust estimation of
  high-dimensional gaussians and gaussian mixtures.
\newblock In {\em 2017 IEEE 58th Annual Symposium on Foundations of Computer
  Science (FOCS)}, pages 73--84. IEEE, 2017.

\bibitem[DKTZ21]{daskalakis2021statistical}
Constantinos Daskalakis, Vasilis Kontonis, Christos Tzamos, and Manolis
  Zampetakis.
\newblock A statistical taylor theorem and extrapolation of truncated
  densities, 2021.

\bibitem[DKZ20]{diakonikolas2020near}
Ilias Diakonikolas, Daniel Kane, and Nikos Zarifis.
\newblock Near-optimal sq lower bounds for agnostically learning halfspaces and
  relus under gaussian marginals.
\newblock In H.~Larochelle, M.~Ranzato, R.~Hadsell, M.~F. Balcan, and H.~Lin,
  editors, {\em Advances in Neural Information Processing Systems}, volume~33,
  pages 13586--13596. Curran Associates, Inc., 2020.

\bibitem[DRZ20]{daskalakis2020truncated}
Constantinos Daskalakis, Dhruv Rohatgi, and Emmanouil Zampetakis.
\newblock Truncated linear regression in high dimensions.
\newblock In H.~Larochelle, M.~Ranzato, R.~Hadsell, M.~F. Balcan, and H.~Lin,
  editors, {\em Advances in Neural Information Processing Systems}, volume~33,
  pages 10338--10347. Curran Associates, Inc., 2020.

\bibitem[DV08]{dunagan2008simple}
John Dunagan and Santosh Vempala.
\newblock A simple polynomial-time rescaling algorithm for solving linear
  programs.
\newblock {\em Mathematical Programming}, 114(1):101--114, 2008.

\bibitem[Fel57]{feller1957introduction}
William Feller.
\newblock An introduction to probability theory and its applications.
\newblock {\em John Wiley}, 1957.

\bibitem[Fel17]{Feldman16}
Vitaly Feldman.
\newblock A general characterization of the statistical query complexity.
\newblock In {\em Conference on Learning Theory}, pages 785--830. PMLR, 2017.

\bibitem[FGR{\etalchar{+}}17]{feldman2017statistical}
Vitaly Feldman, Elena Grigorescu, Lev Reyzin, Santosh~S Vempala, and Ying Xiao.
\newblock Statistical algorithms and a lower bound for detecting planted
  cliques.
\newblock {\em Journal of the ACM (JACM)}, 64(2):1--37, 2017.

\bibitem[FGV17]{FeldmanGV15}
Vitaly Feldman, Crist{\'o}bal Guzm{\'a}n, and Santosh Vempala.
\newblock Statistical query algorithms for mean vector estimation and
  stochastic convex optimization.
\newblock In {\em Proceedings of the Twenty-Eighth Annual ACM-SIAM Symposium on
  Discrete Algorithms}, pages 1265--1277, 2017.

\bibitem[FHT{\etalchar{+}}01]{friedman2001elements}
Jerome Friedman, Trevor Hastie, Robert Tibshirani, et~al.
\newblock {\em The elements of statistical learning}, volume~1.
\newblock Springer series in statistics New York, 2001.

\bibitem[FKT20]{fotakis2020efficient}
Dimitris Fotakis, Alkis Kalavasis, and Christos Tzamos.
\newblock Efficient parameter estimation of truncated boolean product
  distributions.
\newblock In {\em Conference on Learning Theory}, pages 1586--1600. PMLR, 2020.

\bibitem[FLH{\etalchar{+}}20]{feng2020provably-consistent-theory}
Lei Feng, Jiaqi Lv, Bo~Han, Miao Xu, Gang Niu, Xin Geng, Bo~An, and Masashi
  Sugiyama.
\newblock Provably consistent partial-label learning.
\newblock {\em Advances in Neural Information Processing Systems},
  33:10948--10960, 2020.

\bibitem[FPV15]{FeldmanPV15}
V.~Feldman, W.~Perkins, and S.~Vempala.
\newblock On the complexity of random satisfiability problems with planted
  solutions.
\newblock In {\em Proceedings of the Forty-Seventh Annual {ACM} on Symposium on
  Theory of Computing, {STOC}, 2015}, pages 77--86, 2015.

\bibitem[GGJ{\etalchar{+}}20]{goel2020superpolynomial}
Surbhi Goel, Aravind Gollakota, Zhihan Jin, Sushrut Karmalkar, and Adam
  Klivans.
\newblock Superpolynomial lower bounds for learning one-layer neural networks
  using gradient descent.
\newblock In {\em International Conference on Machine Learning}, pages
  3587--3596. PMLR, 2020.

\bibitem[GGK20]{goel2020statistical}
Surbhi Goel, Aravind Gollakota, and Adam Klivans.
\newblock Statistical-query lower bounds via functional gradients.
\newblock In H.~Larochelle, M.~Ranzato, R.~Hadsell, M.~F. Balcan, and H.~Lin,
  editors, {\em Advances in Neural Information Processing Systems}, volume~33,
  pages 2147--2158. Curran Associates, Inc., 2020.

\bibitem[GLB{\etalchar{+}}18]{guo2018cnn}
Yanming Guo, Yu~Liu, Erwin~M Bakker, Yuanhao Guo, and Michael~S Lew.
\newblock Cnn-rnn: a large-scale hierarchical image classification framework.
\newblock {\em Multimedia tools and applications}, 77(8):10251--10271, 2018.

\bibitem[Gou00]{gourieroux2000econometrics}
Christian Gourieroux.
\newblock {\em Econometrics of qualitative dependent variables}.
\newblock Cambridge university press, 2000.

\bibitem[Gui09]{guionnet2009large}
Alice Guionnet.
\newblock {\em Large random matrices}, volume 1957.
\newblock Springer Science \& Business Media, 2009.

\bibitem[GVDLR97]{gill1997coarsening}
Richard~D Gill, Mark~J Van Der~Laan, and James~M Robins.
\newblock Coarsening at random: Characterizations, conjectures,
  counter-examples.
\newblock In {\em Proceedings of the First Seattle Symposium in Biostatistics},
  pages 255--294. Springer, 1997.

\bibitem[H{\aa}s01]{haastad2001some}
Johan H{\aa}stad.
\newblock Some optimal inapproximability results.
\newblock {\em Journal of the ACM (JACM)}, 48(4):798--859, 2001.

\bibitem[HB06]{hullermeier2006learning}
Eyke H{\"u}llermeier and J{\"u}rgen Beringer.
\newblock Learning from ambiguously labeled examples.
\newblock {\em Intelligent Data Analysis}, 10(5):419--439, 2006.

\bibitem[HC15]{hullermeier2015superset}
Eyke H{\"u}llermeier and Weiwei Cheng.
\newblock Superset learning based on generalized loss minimization.
\newblock In {\em Joint European Conference on Machine Learning and Knowledge
  Discovery in Databases}, pages 260--275. Springer, 2015.

\bibitem[HL19]{hopkins2019hard}
Samuel~B Hopkins and Jerry Li.
\newblock How {H}ard is {R}obust {M}ean {E}stimation?
\newblock In {\em Conference on Learning Theory}, pages 1649--1682, 2019.

\bibitem[Hub04]{huber2004robust}
Peter~J Huber.
\newblock {\em Robust statistics}, volume 523.
\newblock John Wiley \& Sons, 2004.

\bibitem[INHS17]{ishida2017learning-complementary-labels}
Takashi Ishida, Gang Niu, Weihua Hu, and Masashi Sugiyama.
\newblock Learning from complementary labels.
\newblock {\em Advances in neural information processing systems}, 30, 2017.

\bibitem[IZD20]{ilyas2020theoretical}
Andrew Ilyas, Emmanouil Zampetakis, and Constantinos Daskalakis.
\newblock A theoretical and practical framework for regression and
  classification from truncated samples.
\newblock In {\em International Conference on Artificial Intelligence and
  Statistics}, pages 4463--4473. PMLR, 2020.

\bibitem[JG02]{jin2002learning}
Rong Jin and Zoubin Ghahramani.
\newblock Learning with multiple labels.
\newblock {\em Advances in neural information processing systems}, 15, 2002.

\bibitem[JLL{\etalchar{+}}20]{jiao2020fine}
Qihan Jiao, Zhi Liu, Gongyang Li, Linwei Ye, and Yang Wang.
\newblock Fine-grained image classification with coarse and fine labels on
  one-shot learning.
\newblock In {\em 2020 IEEE International Conference on Multimedia \& Expo
  Workshops (ICMEW)}, pages 1--6. IEEE, 2020.

\bibitem[JLYW19]{jiao2019weakly}
Qihan Jiao, Zhi Liu, Linwei Ye, and Yang Wang.
\newblock Weakly labeled fine-grained classification with hierarchy
  relationship of fine and coarse labels.
\newblock {\em Journal of Visual Communication and Image Representation},
  63:102584, 2019.

\bibitem[Kea98]{kearns1998efficient}
Michael Kearns.
\newblock Efficient noise-tolerant learning from statistical queries.
\newblock {\em Journal of the ACM (JACM)}, 45(6):983--1006, 1998.

\bibitem[KKM18]{KlivansKM18}
A.~R. Klivans, P.~K. Kothari, and R.~Meka.
\newblock Efficient algorithms for outlier-robust regression.
\newblock In {\em Conference On Learning Theory, {COLT} 2018}, pages
  1420--1430, 2018.

\bibitem[KTZ19]{KTZ19}
Vasilis Kontonis, Christos Tzamos, and Manolis Zampetakis.
\newblock Efficient {T}runcated {S}tatistics with {U}nknown {T}runcation.
\newblock In {\em 260th Annual IEEE Symposium on Foundations of Computer
  Science (FOCS)}, pages 1578--1595. IEEE, 2019.

\bibitem[LBMK20]{lukasik2020does-icml-theory}
Michal Lukasik, Srinadh Bhojanapalli, Aditya Menon, and Sanjiv Kumar.
\newblock Does label smoothing mitigate label noise?
\newblock In {\em International Conference on Machine Learning}, pages
  6448--6458. PMLR, 2020.

\bibitem[LD14]{liu2014learnability-superset-label}
Liping Liu and Thomas Dietterich.
\newblock Learnability of the superset label learning problem.
\newblock In {\em International Conference on Machine Learning}, pages
  1629--1637. PMLR, 2014.

\bibitem[LGW17]{lei2017weakly}
Jie Lei, Zhenyu Guo, and Yang Wang.
\newblock Weakly supervised image classification with coarse and fine labels.
\newblock In {\em 2017 14th Conference on Computer and Robot Vision (CRV)},
  pages 240--247. IEEE, 2017.

\bibitem[LRV16]{LRV16}
Kevin~A. Lai, Anup~B. Rao, and Santosh Vempala.
\newblock Agnostic {E}stimation of {M}ean and {C}ovariance.
\newblock In {\em {IEEE} 57th Annual Symposium on Foundations of Computer
  Science (FOCS)}, pages 665--674, 2016.

\bibitem[LXF{\etalchar{+}}20]{lv2020progressive}
Jiaqi Lv, Miao Xu, Lei Feng, Gang Niu, Xin Geng, and Masashi Sugiyama.
\newblock Progressive identification of true labels for partial-label learning.
\newblock In {\em International Conference on Machine Learning}, pages
  6500--6510. PMLR, 2020.

\bibitem[Mad86]{maddala1986limited}
Gangadharrao~S Maddala.
\newblock {\em Limited-dependent and qualitative variables in econometrics}.
\newblock Number~3. Cambridge university press, 1986.

\bibitem[Mas07]{massart2007concentration}
Pascal Massart.
\newblock {\em Concentration inequalities and model selection}, volume~6.
\newblock Springer, 2007.

\bibitem[NC08]{nguyen2008classification-old-intro}
Nam Nguyen and Rich Caruana.
\newblock Classification with partial labels.
\newblock In {\em Proceedings of the 14th ACM SIGKDD international conference
  on Knowledge discovery and data mining}, pages 551--559, 2008.

\bibitem[NDRT13]{natarajan2013learning-theory}
Nagarajan Natarajan, Inderjit~S Dhillon, Pradeep~K Ravikumar, and Ambuj Tewari.
\newblock Learning with noisy labels.
\newblock {\em Advances in neural information processing systems}, 26, 2013.

\bibitem[NP19]{NP19}
Sai~Ganesh Nagarajan and Ioannis Panageas.
\newblock On the {A}nalysis of {EM} for truncated mixtures of two {G}aussians.
\newblock In {\em 31st International Conference on Algorithmic Learning Theory
  (ALT)}, pages 955--960, 2019.

\bibitem[O{\etalchar{+}}90]{owen1990empirical}
Art Owen et~al.
\newblock Empirical likelihood ratio confidence regions.
\newblock {\em The Annals of Statistics}, 18(1):90--120, 1990.

\bibitem[Owe88]{owen1988empirical}
Art~B Owen.
\newblock Empirical likelihood ratio confidence intervals for a single
  functional.
\newblock {\em Biometrika}, 75(2):237--249, 1988.

\bibitem[Owe01]{owen2001empirical}
Art~B Owen.
\newblock {\em Empirical likelihood}.
\newblock CRC press, 2001.

\bibitem[PCMY15]{papandreou2015weakly-application-partial}
George Papandreou, Liang-Chieh Chen, Kevin~P Murphy, and Alan~L Yuille.
\newblock Weakly-and semi-supervised learning of a deep convolutional network
  for semantic image segmentation.
\newblock In {\em Proceedings of the IEEE international conference on computer
  vision}, pages 1742--1750, 2015.

\bibitem[QCJ{\etalchar{+}}20]{qin2020learning}
Zengyi Qin, Jiansheng Chen, Zhenyu Jiang, Xumin Yu, Chunhua Hu, Yu~Ma, Suhua
  Miao, and Rongsong Zhou.
\newblock Learning fine-grained estimation of physiological states from
  coarse-grained labels by distribution restoration.
\newblock {\em Scientific Reports}, 10(1):1--10, 2020.

\bibitem[RDS{\etalchar{+}}15]{imagenet}
Olga Russakovsky, Jia Deng, Hao Su, Jonathan Krause, Sanjeev Satheesh, Sean Ma,
  Zhiheng Huang, Andrej Karpathy, Aditya Khosla, Michael~S. Bernstein,
  Alexander~C. Berg, and Fei{-}Fei Li.
\newblock Imagenet large scale visual recognition challenge.
\newblock {\em International journal of computer vision}, 115(3):211--252,
  2015.

\bibitem[RGGV15]{RGGG15}
M.~{Ristin}, J.~{Gall}, M.~{Guillaumin}, and L.~{Van Gool}.
\newblock From categories to subcategories: Large-scale image classification
  with partial class label refinement.
\newblock In {\em 2015 IEEE Conference on Computer Vision and Pattern
  Recognition (CVPR)}, pages 231--239, 2015.

\bibitem[SBH13]{scott2013classification}
Clayton Scott, Gilles Blanchard, and Gregory Handy.
\newblock Classification with asymmetric label noise: Consistency and maximal
  denoising.
\newblock In {\em Conference on learning theory}, pages 489--511. PMLR, 2013.

\bibitem[Sch86]{Schneider86}
Helmut Schneider.
\newblock {\em Truncated and censored samples from normal populations}.
\newblock Marcel Dekker, Inc., 1986.

\bibitem[Sha18]{shamir2018distribution}
Ohad Shamir.
\newblock Distribution-specific hardness of learning neural networks.
\newblock {\em The Journal of Machine Learning Research}, 19(1):1135--1163,
  2018.

\bibitem[SSBD14]{shalev2014understanding}
Shai Shalev-Shwartz and Shai Ben-David.
\newblock {\em Understanding {M}achine {L}earning: From {T}heory to
  {A}lgorithms}.
\newblock Cambridge University Press, 2014.

\bibitem[TG75]{thomas1975confidence}
David~R Thomas and Gary~L Grunkemeier.
\newblock Confidence interval estimation of survival probabilities for censored
  data.
\newblock {\em Journal of the American Statistical Association},
  70(352):865--871, 1975.

\bibitem[TGH15]{triguero2015self-apps-partial}
Isaac Triguero, Salvador Garc{\'\i}a, and Francisco Herrera.
\newblock Self-labeled techniques for semi-supervised learning: taxonomy,
  software and empirical study.
\newblock {\em Knowledge and Information systems}, 42(2):245--284, 2015.

\bibitem[TKD{\etalchar{+}}19]{taherkhani2019weakly}
Fariborz Taherkhani, Hadi Kazemi, Ali Dabouei, Jeremy Dawson, and Nasser~M
  Nasrabadi.
\newblock A weakly supervised fine label classifier enhanced by coarse
  supervision.
\newblock In {\em Proceedings of the IEEE/CVF International Conference on
  Computer Vision}, pages 6459--6468, 2019.

\bibitem[Tob58]{tobin1958estimation}
James Tobin.
\newblock Estimation of relationships for limited dependent variables.
\newblock {\em Econometrica: journal of the Econometric Society}, pages 24--36,
  1958.

\bibitem[TSD{\etalchar{+}}20]{touvron2020grafit}
Hugo Touvron, Alexandre Sablayrolles, Matthijs Douze, Matthieu Cord, and
  Herv{\'e} J{\'e}gou.
\newblock Grafit: Learning fine-grained image representations with coarse
  labels.
\newblock {\em arXiv preprint arXiv:2011.12982}, 2020.

\bibitem[VEH20]{van2020survey}
Jesper~E Van~Engelen and Holger~H Hoos.
\newblock A survey on semi-supervised learning.
\newblock {\em Machine Learning}, 109(2):373--440, 2020.

\bibitem[Ver18]{vershynin2018high}
Roman Vershynin.
\newblock {\em High-dimensional probability: An introduction with applications
  in data science}, volume~47.
\newblock Cambridge university press, 2018.

\bibitem[VRW17]{van2017theory-corrupted-theory-rooyen}
Brendan Van~Rooyen and Robert~C Williamson.
\newblock A theory of learning with corrupted labels.
\newblock {\em J. Mach. Learn. Res.}, 18(1):8501--8550, 2017.

\bibitem[VW19]{vempala2019gradient}
Santosh Vempala and John Wilmes.
\newblock Gradient descent for one-hidden-layer neural networks: Polynomial
  convergence and sq lower bounds.
\newblock In {\em Conference on Learning Theory}, pages 3115--3117. PMLR, 2019.

\bibitem[WCH{\etalchar{+}}21]{wen2021leveraged-icml-theory}
Hongwei Wen, Jingyi Cui, Hanyuan Hang, Jiabin Liu, Yisen Wang, and Zhouchen
  Lin.
\newblock Leveraged weighted loss for partial label learning.
\newblock In {\em International Conference on Machine Learning}, pages
  11091--11100. PMLR, 2021.

\bibitem[WDS19]{WDS19}
Shanshan Wu, Alexandros~G. Dimakis, and Sujay Sanghavi.
\newblock Learning distributions generated by one-layer relu networks.
\newblock In Hanna~M. Wallach, Hugo Larochelle, Alina Beygelzimer, Florence
  d'Alch{\'{e}}{-}Buc, Emily~B. Fox, and Roman Garnett, editors, {\em Advances
  in Neural Information Processing Systems 32: Annual Conference on Neural
  Information Processing Systems 2019, NeurIPS 2019, December 8-14, 2019,
  Vancouver, BC, Canada}, pages 8105--8115, 2019.

\bibitem[Wol79]{wolynetz79}
MS~Wolynetz.
\newblock Algorithm as 139: Maximum likelihood estimation in a linear model
  from confined and censored normal data.
\newblock {\em Journal of the Royal Statistical Society. Series C (Applied
  Statistics)}, 28(2):195--206, 1979.

\bibitem[XQGZ21]{xu2021instance}
Ning Xu, Congyu Qiao, Xin Geng, and Min-Ling Zhang.
\newblock Instance-dependent partial label learning.
\newblock {\em Advances in Neural Information Processing Systems}, 34, 2021.

\bibitem[YZ16a]{yu2016maximum}
Fei Yu and Min-Ling Zhang.
\newblock Maximum margin partial label learning.
\newblock In {\em Asian conference on machine learning}, pages 96--111. PMLR,
  2016.

\bibitem[YZ16b]{yu2016maximum-app-partial-max-margin}
Fei Yu and Min-Ling Zhang.
\newblock Maximum margin partial label learning.
\newblock In {\em Asian conference on machine learning}, pages 96--111. PMLR,
  2016.

\bibitem[Zha14]{zhang2014disambiguation}
Min-Ling Zhang.
\newblock Disambiguation-free partial label learning.
\newblock In {\em Proceedings of the 2014 SIAM International Conference on Data
  Mining}, pages 37--45. SIAM, 2014.

\bibitem[ZYT17]{zhang2017disambiguation-app-partial}
Min-Ling Zhang, Fei Yu, and Cai-Zhi Tang.
\newblock Disambiguation-free partial label learning.
\newblock {\em IEEE Transactions on Knowledge and Data Engineering},
  29(10):2155--2167, 2017.

\end{thebibliography}


\appendix

\section{Training Models from Coarse Data}
\label{appendix:training}
Consider a parameterized family of functions $\vec x \rightarrow f(\vec x; \vec w)$, where the parameters $\vec w$ lie in some parameter space $\mathcal{W} \subseteq \reals^p.$ For instance, the family may correspond to a feed-forward neural network with $L$ layers.
Given a finely labeled training sample $(\vec x_1, y_1), \ldots, (\vec x_N, y_N) \in \mathcal{X} \times \mathcal{Y}$, the parameters $\vec w$ are chosen using a gradient method in order to minimize the empirical risk, 
\[
\L_N(\vec w) = \frac{1}{N} \sum_{i = 1}^N \l(f(\vec x_i; \vec w), y_i)\,,
\]
for some loss function $\l : \mathcal{Y} \times \mathcal{Y} \rightarrow \mathbb{R}$ and the goal of this optimization task is to minimize the population risk function $\L(\vec w) = \E_{(\vec x, y) \sim \D(\vec w^{\star})}[\l(f(\vec x; \vec w), y)]$ (where the distribution $\D(\vec w^{\star})$ is unknown). For simplicity, let us focus on differentiable loss functions. Performing the SGD algorithm, we can circumvent the lack of knowledge of the population risk function $\L$. Specifically, instead of computing the gradient of $\L(\vec w)$, the algorithm steps towards a random direction $\vec v$ with the constraint that the expected value of $\vec v$ is equal to the negative of the true gradient, i.e., it is an unbiased estimate of $-\nabla \L(\vec w)$. Such a random vector $\vec v$ can be computed without knowing $\D(\vec w^{\star})$ using the interchangeability between the expectation and the gradient operators. Assume that the algorithm is at iteration $t \geq 1$. Let $(\vec x, y) \sim \D(\vec w^{\star})$ be a fresh sample and define $\vec v_t$ be the gradient of the loss function with respect to $\vec w$, at the point $\vec w_t$, i.e.,
\[
\E[\vec v_t | \vec w_t] = \E_{(\vec x, y) \sim \D(\vec w^{\star})} \left[\nabla \l(f(\vec x; \vec w_t), y)\right] = \nabla \E_{(\vec x, y) \sim \D(\vec w^{\star})}\left[\l(f(\vec x; \vec w_t), y)\right] = \nabla \L(\vec w_t)\,.
\]

Hence, an algorithm 
that has query access to a SQ oracle can implement a noisy version of the above iterative process (with inexact gradients, see e.g.,~\cite{d2008smooth, devolder2014first, FeldmanGV15}) using the query functions $q_i(\vec x, y) = \left ( \nabla \l (f(\vec x; \vec w_t),y) \right)_{i}$ for any $i \in [p]$. Note that the algorithm knows the loss function $\l$, the parameterized functions' family $\{ f(\cdot~; \vec w) : \vec w \in \mathcal{W} \}$ and the current guess $\vec w_t$. Specifically, the algorithm performs $p$ queries (one for each coordinate of the parameter vector) and the oracle returns to the algorithm a noisy gradient vector $\vec r_t$ that satisfies $\| \vec r_t - \nabla \L(\vec w_t) \|_{\infty} \leq \tau$.

In our setting, we do not have access to the SQ oracle with finely labeled examples. Our main result of this section (\Cref{theorem:intro-reduction}) is a mechanism that enables us to obtain access to such an oracle using a few coarsely labeled examples (with high probability). Hence, we can still perform the noisy gradient descent of the previous paragraph with an additional overhead on the sample complexity, due to the reduction.

\section{Multiclass Logistic Regression with Coarsely Labeled Data}
\label{appendix:logistic}
A first application for the above generic reduction from coarse data to statistical queries is the case of coarse multiclass logistic regression. In the standard (finely labeled) multiclass logistic regression problem, there are $k$ fine labels (that correspond to classes), each one associated with a weight vector $\vec w_z \in \reals^n$ with $z \in [k]$. We can consider the weight matrix $\vec W \in \reals^{k \times n}$. Given an example $\vec x \in \reals^n$, the vector $\vec x$ is filtered via the softmax function $\sigma(\vec W,\vec x)$, which is a probability distribution over $\Delta^k$ with $\sigma(\vec W, \vec x; z) = \exp(\vec w_z^T \vec x)/\sum_{y \in [k]}\exp(\vec w_y^T \vec x), z \in [k]$ and the output is the  finely labeled example $(\vec x,z) \in \reals^n \times [k]$. The goal is to estimate the weight matrix $\vec W$, given finely labeled examples. Let us denote by $\D(\vec W)$ the joint distribution over the finely labeled examples for simplicity. When we have access to finely labeled examples $(\vec x,z) \sim \D(\vec W^\star)$, the population log-likelihood objective $\mathcal{L}$ of the multiclass logistic regression problem 
\[
\mathcal{L}(\vec W) = \E_{(\vec x,z) \sim \D({\vec W^{\star}})} \Big [\vec w_z^T \vec x - \log \Big (\sum_{j \in \mathcal{Z}}\exp(\vec w_j^T \vec x) \Big ) \Big]\,,
\]
is concave (see~\cite{friedman2001elements}) with respect to the weight matrix $\vec W \in \mathbb{R}^{k \times n}$ and is solved using gradient methods. On the other hand, if we have sample access only to coarsely labeled examples $(\vec x,S) \sim \D_{\pi}(\vec W^\star)$, the population log-likelihood objective $\mathcal{L}_{\pi}$ of the coarse multiclass logistic regression problem 
\[
\mathcal{L}_{\pi}(\vec W) = \E_{(\vec x,S) \sim \D_{\pi}(\vec W^{\star})} \Big [ \log \Big ( \sum_{z \in S}\exp(\vec w_z^T \vec x) \Big ) - \log \Big (\sum_{j \in \mathcal{Z}}\exp(\vec w_j^T \vec x) \Big ) \Big]\,,
\]
which is no more concave. However, as an application of our main result (\Cref{theorem:intro-reduction}), we can still solve it. In fact, since we can implement statistical queries using the sample access to the coarse data generative process $\D_{\pi}(\vec W^\star)$, we can compute the gradients of the log-likelihood objective that corresponds to the \emph{finely labeled examples}. Hence, the total sample complexity of optimizing this non-convex objective is equal to the sample complexity of solving the convex problem with an additional overhead at each iteration of computing the gradients, that is given by \Cref{theorem:intro-reduction}.

\section{Geometric Information Preservation}
\label{appendix:preservation}

In this section, we aim to provide some intuition behind the notion of information preserving partitions. The following result provides a geometric property for the partition distribution $\pi$. We show that if the partition distribution satisfies this particular geometric property, then it is also information preserving. We underline that the geometric property is quite important for our better understanding and it has the advantage that it is easy to verify. Hence, while the notion of information preserving distributions may be less intuitive, we believe that the geometric preservation property that we state in \Cref{lemma:preservation} can fulfill this lack of intuition. The property informally states that, for any hyperplane, the sets in the partition that are not cut by this hyperplane have non trivial probability mass with respect to the true Gaussian. In the case of mixtures of convex partitions, we would like the same property to hold in expectation. 
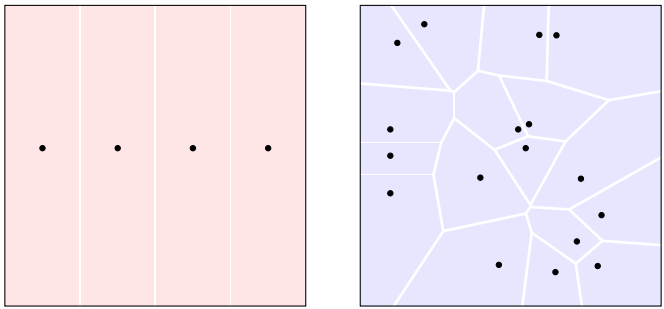
\begin{figure}[h]
    \centering
    \begin{subfigure}{0.3\textwidth}
    \centering
    \begin{tikzpicture}[scale=0.5]
    \def\pts{}
    
    \edef\pts{\pts, (-3,0.2)}
    \edef\pts{\pts, (-1,0.2)}
    \edef\pts{\pts, (1,0.2)}
    \edef\pts{\pts, (3,0.2)}
    
    \xintForpair #1#2 in \pts \do{
      \edef\pta{#1,#2}
      \begin{scope}
        \xintForpair \#3#4 in \pts \do{
          \edef\ptb{#3,#4}
          \ifx\pta\ptb\relax 
            \tikzstyle{myclip}=[];
          \else
            \tikzstyle{myclip}=[clip];
          \fi;
          \path[myclip] (#3,#4) to[half plane] (#1,#2);
        }
        \clip (-\maxxy,-\maxxy) rectangle (\maxxy,\maxxy); 
        \pgfmathsetmacro{\randhue}{rnd}
        \definecolor{randcolor}{hsb}{\randhue,.5,1}
        \fill[red, opacity=0.1] (#1,#2) circle (4*\biglen); 
        \fill[draw=black] (#1,#2) circle (2pt); 
      \end{scope}
    }
    \pgfresetboundingbox
    \draw (-\maxxy,-\maxxy) rectangle (\maxxy,\maxxy);
  \end{tikzpicture}
\end{subfigure}%
 \begin{subfigure}{0.3\textwidth}
    \centering
    \begin{tikzpicture}[scale=0.5]
    \def\pts{}
    \pgfmathsetseed{1908}
    \xintFor* #1 in {\xintSeq {1}{12}} \do{
      \pgfmathsetmacro{\ptx}{.9*\maxxy*rand} 
      \pgfmathsetmacro{\pty}{.9*\maxxy*rand} 
      \edef\pts{\pts, (\ptx,\pty)} 
    }
    \edef\pts{\pts, (0.4,0.2)}
    \edef\pts{\pts, (0.2,0.7)}
    \edef\pts{\pts, (-3.2,0.7)}
    \edef\pts{\pts, (-3.2,0)}
    \edef\pts{\pts, (-3.2,-1.0)}
    
    \xintForpair #1#2 in \pts \do{
      \edef\pta{#1,#2}
      \begin{scope}
        \xintForpair \#3#4 in \pts \do{
          \edef\ptb{#3,#4}
          \ifx\pta\ptb\relax 
            \tikzstyle{myclip}=[];
          \else
            \tikzstyle{myclip}=[clip];
          \fi;
          \path[myclip] (#3,#4) to[half plane] (#1,#2);
        }
        \clip (-\maxxy,-\maxxy) rectangle (\maxxy,\maxxy); 
        \pgfmathsetmacro{\randhue}{rnd}
        \definecolor{randcolor}{hsb}{\randhue,.5,1}
        \fill[blue, opacity=0.1] (#1,#2) circle (4*\biglen); 
        \fill[draw=black] (#1,#2) circle (2pt); 
      \end{scope}
    }
    \pgfresetboundingbox
    \draw (-\maxxy,-\maxxy) rectangle (\maxxy,\maxxy);
  \end{tikzpicture}
\end{subfigure}%
    \caption{
    (a) is a very rough partition that makes learning the mean impossible: 
    Gaussians $\N((0,z))$ centered along the same vertical line $(0,z)$ assign exactly the same probability to all cells of the
partitions and therefore, $\dtv(\N_\pi((0, z_1)), \N_\pi((0,z_2)) ) = 0$: it is impossible to learn the second coordinate of the mean.  
    (b) is a convex partition of $\reals^2$, that makes recovering the Gaussian possible.
    }
    \label{fig:appendix:convex}
\end{figure}

Before stating \Cref{lemma:preservation}, let us return to \Cref{fig:appendix:convex}. Observe that, in the first example with the four halfspaces, the geometric property does not hold, since there exists a line (i.e., a hyperplane) that intersects with all the sets. On the other hand, if we consider the second example with the Voronoi partition and assume that the true mean lies in the middle of the picture, we can see that any hyperplane does not intersect with a sufficient number of sets and, hence, the union of the uncut sets has non trivial probability mass for any hyperplane.

For a hyperplane $\mathcal{H}_{\vec w, c} = \{\vec x \in \reals^{\dim} : \vec w^T \vec x = c \}$ with normal vector $\vec w \in \reals^{\dim}$ and threshold $c \in \reals$, we denote the two associated halfspaces by $\mathcal{H}_{\vec w, c}^+ = \{ \vec x \in \reals^{\dim} : \vec w^T \vec x > c \}$ and $\mathcal{H}_{\vec w, c}^- = \{ \vec x \in \reals^{\dim}: \vec w^T \vec x < c \}$. Before stating the next Lemma, we shortly describe what means for a hyperplane to cut a set with respect to a Gaussian $\N$. The set $S$ is not cut by the hyperplane $\mathcal{H}$, if it totally lies in a halfspace induced by the hyperplane, say $\mathcal{H}^+$, i.e., it holds that $\N(S) = \N(S \cap \mathcal{H}^+)$. 
\begin{lemma}
[Geometric Information Preservation]
\label{lemma:preservation}
Consider the generative process of coarse $\dim$-dimensional Gaussian data $\N_{\pi}(\vmu^{\star}),$ (see \Cref{definition:intro-gaussian-coarse}). Consider an arbitrary hyperplane $\mathcal{H}_{\vec w, c}$ with normal vector $\vec w \in \reals^{\dim}$ and threshold $c \in \reals$. 
For a partition $\S \in \mathrm{supp}(\pi)$ of~$\reals^\dim$, consider the collection that contains all the sets that are not cut by the hyperplane $\mathcal{H}_{\vec w, c}$, i.e.,
\[
U_{\vec w, c, \S} = \bigcup \Big \{ S \in \S : \N^{\star}(S \cap \mathcal{H}_{\vec w, c}^+) = \N^{\star}(S) \lor \N^{\star}(S \cap \mathcal{H}_{\vec w, c}^-) = \N^{\star}(S)  \Big \}\,.
\]
Assume that $\pi$ satisfies 
\begin{equation}
\label{eq:geom-pres}
\E_{\S \sim \pi} \Big [\N(\vmu^{\star}; U_{\vec w, c,\S}) \Big ] \geq \alpha\,,
\end{equation}
for some $\alpha \in (0,1]$. Then, for any Gaussian distribution $\N(\vmu)$, it holds that
\[
\dtv(\N_{\pi}(\vmu), \N_{\pi}(\vmu^{\star})) \geq C_{\alpha} \cdot \dtv(\N(\vmu), \N(\vmu^{\star}))\,,
\]
for some $C_{\alpha}$ that depends only on $\alpha$ and satisfies $C_{\alpha} = \poly(\alpha)$, i.e., the partition distribution is $C_{\alpha}$-information preserving.
\end{lemma}

Hence, the above geometric property is sufficient for information preservation.
If we assume that the total variation distance between the true Gaussian distribution $\N(\vmu^{\star})$ and a possible guess $\N(\vmu)$ is at least $\eps$ and the partition distribution satisfies the geometric property of Equation~\eqref{eq:geom-pres}, 
we get that the coarse generative process preserves a sufficiently large gap, in the sense that $\dtv(\N_{\pi}(\vmu^{\star}), \N_{\pi}(\vmu)) \geq \poly(\alpha)\eps$. The proof of the above lemma, which relies on high-dimensional anti-concentration results on Gaussian distributions, follows.

\begin{proof}[\textit{Proof of}~\Cref{lemma:preservation}]
Let us denote the true distribution by $\N^{\star} = \N(\vmu^{\star}, \vec I)$ for short. Consider an arbitrary hyperplane $\mathcal{H}_{\vec w, c}$ with normal vector $\vec w \in \reals^{\dim}$ and threshold $c \in \reals$. Since the partition distribution (supported on a family of partitions $\mathcal{B}$) satisfies Equation~\eqref{eq:geom-pres}, we have that, for the random variable $\N^{\star}(U_{\vec w, c,\S}),$ that takes values in $[0,1]$, there exists $\alpha$ such that
\[
\E_{\S \sim \pi}\Big [\N^{\star}(U_{\vec w, c,\S})\Big ] = \alpha\,.
\]
We will use the following simple Markov-type inequality for bounded random variables.
\begin{fact}[Lemma B.1 from~\cite{shalev2014understanding}]
\label{lemma:reverse-markov}
Let $Z$ be a random variable that takes values in $[0,1]$. Then, for any $\alpha \in (0,1)$, it holds that
\[
\Pr[Z > \alpha] \geq \frac{\E[Z] - \alpha}{1-\alpha} \geq \E[Z] - \alpha\,.
\]
\end{fact}
\noindent By the \Cref{lemma:reverse-markov},  it holds that
\[
\Pr_{\S \sim \pi}\Big [\N^{\star}(U_{\vec w, c,\S}) \geq \alpha/2\Big ] \geq \alpha/2\,.
\]
Hence, the mass of the ``good'' partitions is at least $\alpha/2$.
Fix such a partition $\S \in \mathcal{B}$ (in the support of the partition distribution) and consider the true $\N^{\star} = \N(\vmu^{\star})$ and the guess $\N = \N(\vmu)$ distributions. For this pair of distributions, consider the set
\[
\mathcal{H} = \Big \{ \vec x \in \reals^{\dim} : \vec x^T (\vmu - \vmu^{\star}) = \big (\|\vmu\|_2^2 - \|\vmu^{\star}\|_2^2\big )/2 \Big \}\,.
\]
Observe that this set is a hyperplane with normal vector $\vmu^{\star} - \vmu$, that contains the midpoint $\frac{1}{2}(\vmu + \vmu^{\star})$ (see \Cref{fig:testing}). 
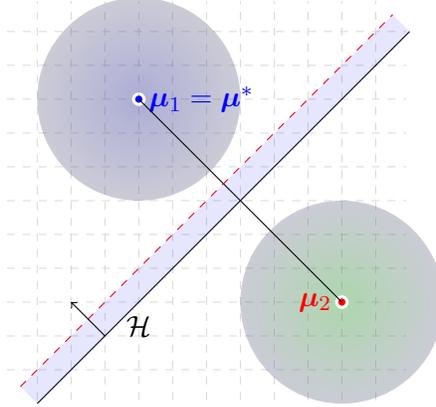
\begin{figure}[ht]

      \centering
    \begin{tikzpicture}[scale=0.45]
    \draw[help lines, color=gray!30, dashed] (-5.9,-5.9) grid (5.9,5.9);
    
    \def\incurve{(-2,3) circle(0.2)}
    \def\outcurve{(-2,3) circle(3)}
    \fill[inner color=blue!80, outer color=blue!20,even odd rule, opacity=0.2] \incurve \outcurve;
    
     \def\incurve{(4,-3) circle(0.2)}
    \def\outcurve{(4,-3) circle(3)}
    \fill[inner color=green!80, outer color=blue!20,even odd rule, opacity=0.2] \incurve \outcurve;
    
    \fill[blue] (-2,3) coordinate (m1) circle (3pt) node[anchor=west] {$\vmu_1 = \vmu^*$};
    \fill[red] (4,-3) coordinate (m2) circle (3pt) node[anchor=east] {$\vmu_2$};
    
    \draw[name path=m12] (m1)--(m2);
    \draw[name path=perp] (-5,-6)--(6,5);
    \draw[dashed, red, name path=perp2] (-5.5,-5.5)--(5.5,5.5);
    
    \tikzfillbetween[of=perp and perp2]{blue, opacity = 0.1};
    
    \node[above]
    at (-2,-4.3) {$\mathcal{H}$};
    
    \draw[->] (-3,-4)--(-4,-3);
    
    
    
    
    
     
    
    \end{tikzpicture}
    \caption{Illustration of the worst-case set in testing the hypotheses $h_1 = \{\vmu_1 = \vmu^\star\}$ and $h_2 = \{\vmu_2 = \vmu^\star \}.$}
    \label{fig:testing}
\end{figure}

Our main focus is to lower bound the total variation distance of the coarse distributions $\N^{\star}_{\pi}$ and $\N_{\pi}$. We claim that this lower bound can be described as a fractional knapsack problem and, hence, it is attained by a worst-case set, that (intuitively) places points as close as possible to the hyperplane $\mathcal{H}$, until its mass with respect to the true Gaussian $\N^{\star}$ is at least $\alpha/2$. Recall that the total variation distance between the two coarse distributions is
\[
\dtv(\N_{\pi}, \N^{\star}_{\pi}) = \sum_{\S \in \mathcal{B}}\pi(\S) \sum_{S \in \S}\Big  |\N(S) - \N^{\star}(S)\Big | \,.
\]
So, the LHS is at least $\Theta(\alpha)$ times the absolute gap of the masses assigned by $\N$ and $\N^\star$ over a worst-case set that lies in a good partition (one with $\N^\star(U_{\vec w, c, \S}) \geq \alpha/2$). This holds since the probability to draw a good partition is at least $\alpha/2$.
The following optimization problem gives a lower bound on the mass gap of a worst-case set in a good partition and, consequently, a lower bound on the total variation distance between $\N^{\star}_{\pi}$ and $\N_{\pi}$.
\[
\min_{S} \Big| \int (\N(\vmu^{\star}; 
\vec x) - \N(\vmu; \vec x)) \vec 1_S(\vec x) d\vec x \Big|\,,
\]
\[
\text{subj. to}~~~ \int \N(\vmu^{\star}; \vec x) \vec 1_S(\vec x) d \vec x \geq \alpha / 2 \,.
\]
We begin with a claim about the shape of the worst case set. Let $t = (\|\vmu\|_2^2 - \|\vmu^{\star}\|_2^2)/2$ be the hyperplane threshold.

\begin{claim}
Let $\mathcal{H}^+ = \{ \vec x : \vec x^T (\vmu - \vmu^{\star}) < t \}$ and $\mathcal{H}^- = \{ \vec x : \vec x^T (\vmu - \vmu^{\star}) > t \}$. The mass of the solution of the fractional knapsack is totally contained in either $\mathcal{H}^+$ or $\mathcal{H}^-$. 
\end{claim}
Since the partition distribution satisfies Equation~\eqref{eq:geom-pres} with respect to the true Gaussian $\N(\vmu^{\star})$ and since the set $\mathcal{H}$ is a hyperplane, the probability mass that is not cut by $\mathcal{H}$ is at least $\alpha$. Hence, there exists a halfspace (either $\mathcal{H}^+$ or $\mathcal{H}^-$) with mass at least $\alpha/2$. Also, observe that the hyperplane $\mathcal{H}$ is the zero locus of the polynomial $q(\vec x) = \|\vec x - \vmu \|_2^2 - \|\vec x - \vmu^{\star} \|_2^2$ and, hence, it is the set of points where the two spherical Gaussians $\N(\vmu)$ and $\N(\vmu^{\star})$ assign equal mass. We have that
\[
\mathcal{H}^+ = \Big \{\vec x : \N(\vmu^{\star}) > \N(\vmu) \Big \}\,.
\]
Hence, we can assume that the worst-case set lies totally in $\mathcal{H}^+$ and, then, the optimization problem can be written as
\[
\min_{S} \int \left(1 - \frac{\N(\vmu; \vec x)}{\N(\vec 0; 
\vec x)}\right) \N(\vec 0; 
\vec x) \vec 1_S(\vec x) d\vec x\,,
\]
\[
\text{subj. to}~~~ \int \N(\vec 0; \vec x) \vec 1_S(\vec x) d \vec x \geq \alpha / 2, ~~ S \in \mathcal{H}^+ \,.
\]
Without loss of generality, we assume that $\N^{\star} = \N(\vec 0, \vec I)$ and $\N = \N(\vmu, \vec I)$. In order to design the worst-case set, since the optimization has the structure of the fractional knapsack problem, we can think of each point $\vec x \in \mathcal{H}^+$ as having \emph{weight} equal to its contribution to the mass gap $(\N(\vec 0; \vec x) - \N(\vmu; \vec x))$ and \emph{value} equal to its density with respect to the true Gaussian $\N(\vec 0; \vec x)$. Hence, in order to design the worst-case set, the points $\vec x \in \mathcal{H}^+$ should be included in the set in order of increasing ratio of weight over value, until reaching a threshold $T$. So, we can define the worst-case set to be
\[
S = \Big  \{ \vec x \in \mathcal{H}^+ : 1 - \frac{\N(\vmu ; \vec x)}{\N(\vec 0; \vec x)} \leq T \Big \} =\Big  \{ \vec x \in \mathcal{H}^+ : 1 - \exp(p(\vec x)) \leq T\Big \} \,,
\]
where $p(\vec x) = -\frac{1}{2}(\vmu - \vec x)^T (\vmu - \vec x) + \frac{1}{2}\vec x^T \vec x = -\frac{1}{2}\vmu^T \vmu + \vmu^T \vec x$ and note that $p(\vec x) \leq 0$ for any $\vec x\in \mathcal{H}^+$. We will use the following anti-concentration result about the Gaussian mass of sets, defined by polynomials.
\begin{lemma}
[Theorem 8 of~\cite{carbery2001distributional}]
\label{lemma:carbery-wright}
Let $q,\gamma \in \reals_+, \vmu \in \reals^{\dim}$ and $\vSigma$ in the positive semidefinite cone $\mathbb{S}^{\dim}_+$. Consider $p : \reals^{\dim} \rightarrow \reals$ a multivariate polynomial of degree at most $\l$ and let
\[
\Q = \Big \{\vec x \in \reals^{\dim} : |p(\vec x)| \leq \gamma \Big \} \,.
\]
Then, there exists an absolute constant $C$ such that
\[
\N(\vmu, \vSigma; \Q) \leq \frac{C q \gamma^{1/\l}}{(\E_{\vec z \sim \N(\vmu, \vSigma)}[|p(\vec z)|^{q/\l}])^{1/q}}\,.
\]
\end{lemma}
We can apply \Cref{lemma:carbery-wright} for the quadratic polynomial $p(\vec x)$ by setting $\gamma = \frac{\alpha^2}{256C^2} \sqrt{\E_{\vec x \sim \N^{\star}}[p^2(\vec x)]}$ with $q=4$, where $C$ is the absolute Carbery-Wright constant. Hence, we get that the Gaussian mass of the set $\Q = \{ \vec x : |p(\vec x)| \leq \gamma \}$ is equal to
\[
\N^{\star}(\Q) \leq \alpha/4\,.
\]
So, for any point $\vec x$ in the remaining $\alpha/4$ mass of the set $S$, it holds that $|p(\vec x)| \geq \gamma$. We first observe that $\gamma$ can lower bounded by the total variation distance of $\N^{\star}$ and $\N$. It suffices to lower bound the expectation $\E_{\vec x \sim \N^{\star}}[p^2(\vec x)]$. We have that
\[
\E_{\vec x \sim \N^{\star}}\Big [p^2(\vec x)\Big ] \geq \Var_{\vec x \sim \N^{\star}}\Big [p(x)\Big ] = \Var_{\vec x \sim \N^{\star}}\Big [ -\frac{1}{2}\vmu^T \vmu + \vmu^T \vec x\Big ] = \| \vmu \|_2^2\,,
\]
and, hence
\[
\gamma \geq \frac{\alpha^2}{256C^2} \cdot \| \vmu \|_2\,.
\]
We will use the following lemma for the total variation distance of two Normal distributions.
\begin{lemma}
[see Corollaries 2.13 and 2.14 of \cite{DKK+16b}]
\label{lemma:tv-gauss}
Let $N_1 = \N(\vmu_1, \vSigma_1), N_2 = \N(\vmu_2, \vSigma_2)$ be two Normal distributions. Then, it holds
\[
\dtv(N_1, N_2) \leq \frac{1}{2} \left \| \vSigma_1^{-1/2}(\vmu_1 - \vmu_2) \right\|_2 + \sqrt{2} \left\|\vec I - \vSigma_1^{-1/2} \vSigma_2 \vSigma_1^{-1/2} \right \|_F\,.
\]
\end{lemma}
\noindent Applying \Cref{lemma:tv-gauss} to the above inequality, we get
\[
\gamma \geq \frac{\alpha^2}{256C^2} \cdot \dtv(\N(\vmu), \N(\vmu^{\star}))\,.
\]
To conclude, we have to lower bound the $L_1$ gap between $\N(\vec 0, \vec I; \vec x)\vec 1_S(\vec x)$ and $\N(\vmu, \vec I; \vec x)\vec 1_S(\vec x)$ and since $S$ lies totally in $\mathcal{H}^+$
\[
\int_S (\N(\vec 0; \vec x) - \N(\vmu; \vec x))d \vec x = \E_{\vec x \sim \N^{\star}}\left [1 - \exp(p(\vec x)) \Big | \vec 1_S(\vec x)\right ] \,.
\]
To proceed, we distinguish two cases: First, assume that $\gamma \leq 1$ and recall that $\Q = \{\vec x : |p(\vec x)| \leq \gamma \}$. Note that for $y \in [-1,0]$, it holds that $1 - \exp(y) \geq |y|/2$ and, hence, we have that:
\[
\int_S (\N(\vec 0; \vec x) - \N(\vmu; \vec x))d \vec x \geq \E_{\vec x \sim \N^{\star}}\left [\frac{|p(\vec x)|}{2}\vec 1_{S \setminus \Q}(\vec x)\right ] \geq \gamma \E_{\vec x \sim \N^{\star}}\left [\vec 1_{S \setminus \Q}(\vec x)\right ] \geq \frac{\alpha \gamma}{4}\,,
\]
and, by the lower bound for $\gamma$, we get
\[
\int_S (\N(\vec 0, \vec I; \vec x) - \N(\vmu, \vec I; \vec x))d \vec x \geq C_{\alpha} \cdot \dtv(\N(\vmu), \N(\vmu^{\star}))\,,
\]
for some $C_{\alpha} = \Omega(\alpha^3)$. Otherwise, let $\gamma > 1$. Note that for $y < -1$, it holds that $1-\exp(y) \geq 1/2.$ Hence, we get that
\[
\int_S (\N(\vec 0; \vec x) - \N(\vmu; \vec x))d \vec x \geq \E_{\vec x \sim \N^{\star}}\left [\frac{1}{2}\vec 1_{S \setminus \Q}(\vec x)\right ] \geq \alpha/8 \,.
\]
In conclusion, we get that
\[
\dtv(\N_{\pi}^{*}, \N_{\pi}) \geq C_{\alpha} \cdot \dtv(\N^{\star}, \N) \,,
\]
where $C_{\alpha} = \poly(\alpha)$ and depends only on $\alpha$.
\end{proof}

\section{Literature Overview on Partial Label Learning}
\label{appendix:partial-literature}

The problem of learning from coarse labels falls in the regime of semi-supervised learning \cite{chapelle-book} and it appears in various literature threads termed as (i) partial label learning \cite{cour2011learning-theory}, (ii) ambiguous label learning \cite{cour2009learning-ambiguous-theory,hullermeier2006learning}, (iii) superset label learning \cite{hullermeier2015superset} and (iv) soft label learning \cite{come2008mixture}. Closely related to these tasks are the problems of learning from complementary labels \cite{ishida2017learning-complementary-labels} and, more generally, learning from noisy and corrupted examples \cite{angluin1988learning,scott2013classification,blanchard2014decontamination,van2017theory-corrupted-theory-rooyen,lukasik2020does-icml-theory}.

We stick with the term partial label learning for now since this is the most widely used. Many real-world learning tasks were solved under the framework of partial label learning such as multimedia content analysis \cite{cour2009learning-ambiguous-theory,cour2011learning-theory} and semantic image segmentation \cite{papandreou2015weakly-application-partial}.

We refer to \cite{jin2002learning,nguyen2008classification-old-intro} and the references therein for 
some seminal papers in the area.
Through the years, various approaches have been proposed to solve this challenging problem by utilizing major machine learning techniques, such as maximum likelihood estimation and Expectation-Maximization \cite{jin2002learning}, convex optimization \cite{cour2011learning-theory}, $k$-nearest neighbors \cite{hullermeier2006learning} and error-correcting output codes \cite{zhang2014disambiguation,zhang2017disambiguation-app-partial}. For an overview of the practical treatment on the problem, we refer the interested reader to \cite{yu2016maximum-app-partial-max-margin,xu2021instance,wen2021leveraged-icml-theory} (and the references therein) and more broadly to \cite{triguero2015self-apps-partial,van2020survey}.

Despite extensive studies on partial label learning from an industrial perspective (applied ML), our theoretical level of understanding is still limited. A fundamental line of research deals with the statistical consistency (see e.g.,
\cite{cour2011learning-theory,cid2014consistency-theory, feng2020provably-consistent-theory,cabannnes2020structured-icml-theory,lv2020progressive,wen2021leveraged-icml-theory})
and the learnability \cite{liu2014learnability-superset-label} of partial label learning algorithms. Moreover, \cite{cauchois2022predictive-theory-duchi} present a methodology between partial supervision and validation.

Closer to our learning from coarse labels approach are the works of \cite{cid2012proper} and \cite{van2017theory-corrupted-theory-rooyen}. In the former, the goal is to estimate the posterior class probabilities from partially labelled data while, in the latter,
the authors study a more general problem of learning from corrupted labels and aim to ``invert'' the corruption.
This technique is inspired by the work of \cite{natarajan2013learning-theory}, where the authors proposed the method of unbiased estimators (which is close to the connection between random classification noise and the SQ framework of \cite{kearns1998efficient}).
This backward correction procedure of \cite{natarajan2013learning-theory,cid2012proper,van2017theory-corrupted-theory-rooyen} recovers the information lost from the corrupted labels (under some structural assumptions) and results in an unbiased estimate of the risk with respect to true distribution. Crucially, these works have to assume that the corruption process (i.e., the coarsening mechanism) is known. This is also commented in \cite{cabannnes2020structured-icml-theory}. Our SQ reduction does not require to know the mechanism; in some sense, the algorithm uses rejection sampling and learning coarse discrete distributions (which is an unsupervised learning problem) in order to invert the coarsening in the sense of \cite{van2017theory-corrupted-theory-rooyen} and obtain statistical queries with respect to the distribution over the finely-labeled examples.

\end{document}